\documentclass[twoside]{article}

\usepackage[accepted]{aistats2021}
\usepackage[colorlinks=true, allcolors=blue]{hyperref}
\usepackage{natbib}
\renewcommand{\cite}{\citep}

\usepackage{url}            
\usepackage{booktabs}       
\usepackage{amsfonts}       
\usepackage{nicefrac}       
\usepackage{microtype}      
\usepackage[dvipsnames]{xcolor}
\usepackage{amsthm}
\usepackage{amsmath}
\usepackage{amssymb}
\usepackage{mathtools} 
 \usepackage{mathrsfs}
 \usepackage{hhline}
 \usepackage{algorithm}
\usepackage{algorithmic}
\usepackage{subfig}
\usepackage{float}
\usepackage{stfloats}
\usepackage{footmisc}
 \usepackage[shortlabels]{enumitem}

\usepackage{custom_commands}
\usepackage{accents}
\begin{document}

%

%

\twocolumn[

\aistatstitle{Global Convergence and Generalization Bound of\\ Gradient-Based Meta-Learning with Deep Neural Nets}

\aistatsauthor{ Haoxiang Wang \And Ruoyu Sun \And  Bo Li }

\aistatsaddress{ University of Illinois, Urbana-Champaign } ]

\begin{abstract}

Gradient-based meta-learning (GBML) with deep neural nets (DNNs) has become a popular approach for few-shot learning. However, due to the non-convexity of DNNs and the bi-level optimization in GBML, the theoretical properties of GBML with DNNs remain largely unknown. In this paper, we first aim to answer the following question: \emph{Does GBML with DNNs have global convergence guarantees?} We provide a positive answer to this question by proving that GBML with over-parameterized DNNs is guaranteed to converge to \emph{global optima} at a linear rate. The second question we aim to address is: \emph{How does GBML achieve fast adaption to new tasks with prior experience on past tasks?} To answer it, we theoretically show that GBML is equivalent to a \emph{functional gradient descent} operation that explicitly propagates experience from the past tasks to new ones, and then we prove a generalization error bound of GBML with over-parameterized DNNs.

\end{abstract}

\section{Introduction}\label{sec:intro}

Meta-learning, or learning-to-learn (LTL) \cite{learningtolearn}, has received much attention due to its applicability in few-shot image classification \cite{few-shot-survey,hospedales2020metalearning}, meta reinforcement learning \cite{vanschoren2018meta,Finn:EECS-2018-105,hospedales2020metalearning}, and other domains such as natural language processing \cite{yu-etal-2018-diverse,bansal2019learning} and computational biology \cite{luo2019mitigating}. The primary motivation for meta-learning is to fast learn a new task from a small amount of data, with prior experience on similar but different tasks.
Gradient-based meta-learning (GBML) is a popular
meta-learning approach, due to its simplicity and good performance in many meta-learning tasks \cite{Finn:EECS-2018-105}. Also, GBML represents a family of meta-learning methods that originate from the model-agnostic meta-learning (MAML) algorithm \cite{maml}. MAML formulates a bi-level optimization problem, where the inner-level objective represents the adaption to a given task, and the outer-level objective is the meta-training loss.
Most GBML methods can be viewed as variants of MAML \citep{reptile,adaptive-GBML,prob-maml,finn2019online,imaml},
and they are almost always applied together with deep neural networks (DNNs) in practice. Even though GBML with DNNs is empirically successful, this approach still lacks a thorough theoretical understanding.

\textbf{\underline{Motivations}}

To theoretically understand why GBML works well in practice, we shall comprehend the \textit{optimization} properties of GBML with DNNs. 
Several recent works theoretically analyze GBML in the case of \textit{convex} objectives \cite{finn2019online,provable-gbml,adaptive-GBML,hu2020biased,xu2020meta}.
However, DNNs are always \textit{non-convex}, so these works 
do not directly apply to GBML with DNNs.
On the other hand, \cite{maml_nonconvex,ji2020multistep,imaml,zhou2019metalearning} consider the non-convex setting, but they only provide convergence to \textit{stationary points}.
Since \textit{stationary points} can have high training/test error, the convergence to them is not very satisfactory. Hence, a crucial question remains unknown for GBML optimization is: \emph{Does GBML with DNNs have global convergence?}\footnote{See Appendix \ref{supp:related-works:meta-learning} for a discussion on the concurrent work \citet{wang2020global}, which is released \textit{after} the first submission of our paper to another conference.} This question motivates us to analyze the optimization properties of GBML with DNNs.

The original intuition behind the design of GBML is that for non-linear DNNs, the meta-learned parameter initializations can encode experience and error from past tasks to achieve fast adaption to new tasks, given past and new tasks are similar \cite{maml}. However, there is no rigorous theory to confirm this intuition. Hence, an important question that remains theoretically unclear is, for non-linear DNNs, \emph{how does GBML achieve fast adaption to a new task with prior experience on past tasks?} In other words, it is mysterious why DNNs trained under GBML enjoy \textit{low generalization error} to unseen tasks.

\textbf{\underline{Technical Challenges.}}
In this paper, our primary goal is to provide theoretical guarantees on the \textit{optimization} and \textit{generalization} properties of GBML with DNNs. 
There exist two main challenges: (1) the non-convexity of DNNs, (2) the bi-level formulation of GBML. In fact, both challenges are entangled together, which makes the theoretical analysis more challenging. To tackle them, we make use of the over-parameterization property of DNNs to ameliorate the non-convexity issue,
and develop a novel analysis to handle the bi-level GBML that is more complex than supervised learning.

{\bf \underline{Main Contributions.}}
\vspace{-0.5em}
\begin{itemize}[leftmargin=*,align=left]
    \item \textbf{Global Convergence of GBML with DNNs}: 
    We prove that with over-parameterized DNNs (i.e., DNNs with a large number of neurons in each layer), GBML is guaranteed to \textit{converge to global optima with zero training loss} at a linear rate. The key to our proof is to develop bounds on the gradient of the GBML objective, and then analyze the optimization trajectory of DNN parameters trained under GBML. 
    Furthermore, we show that for infinitely wide DNNs, GBML becomes equivalent to a kernel regression with a new class of kernels, which we name as Meta Neural Kernels (MNKs).

  \item \textbf{Generalization Bound for GBML with DNNs}: 
Based on the global convergence analysis, we theoretically demonstrate that with over-parameterized DNNs, GBML is equivalent to a \textit{functional gradient descent} that explicitly propagates the prior knowledge about past tasks to new tasks. Motivated by the functional gradient descent formulation, we initiate a theoretical analysis on the generalization ability of GBML. Finally, we prove a \textit{generalization error bound for GBML with over-parameterized neural nets}, by leveraging the Meta Neural Kernel we derived. Besides, we perform an empirical validation for the generalization bound on both synthetic and real-world datasets.
To the best of our knowledge, it is the first generalization bound for GBML with non-linear DNNs, and we believe it provides an in-depth theoretical understanding of GBML.
\end{itemize}
Apart from meta-learning, our analysis might be of independent interest to researchers in other areas such as hyper-parameter optimization, which contains a popular class of algorithms that shares similar formulation with GBML \cite{franceschi2018bilevel}.

{\bf \underline{Related Works.}} Except for papers discussed in this section, more works are related to this paper in the literature of supervised learning and meta-learning, and we discuss these works detailedly in Appendix \ref{supp:related-works}.

\section{Preliminaries} \label{sec:preliminaries}

In this section, we start by introducing the typical setup for few-shot learning. Then we review MAML, the seed of most GBML methods. Notations defined in this section will be adopted in the entire paper.

\subsection{Few-Shot Learning}
\label{sec:few-shot-learning}

Consider a few-shot learning problem with a set of \textit{training tasks} that contains $N$ \textit{supervised-learning tasks} $\{\task_i\}_{i=1}^{N}$. Each task is represented as
$$\task_i= (\xyxyi)\in \mathbb R^{n\times d} \times \mathbb R^{n  k} \times \mathbb R^{m\times d} \times \mathbb R^{m  k},$$ where $(\sX_i,\sY_i)$ represents
$n$ \textit{query} samples (i.e. test samples of $\task_i$) with corresponding labels, while $(\sX_i',\sY_i')$ represents $m$ \textit{support} samples (i.e. training samples of $\task_i$) with labels. For convenience, we denote $$\X=(\sX_i)_{i=1}^N,\Y=(\sY_i)_{i=1}^N,\X'=(\sX_i')_{i=1}^N,\Y'=(\sY_i')_{i=1}^N.$$
In few-shot learning, $\{\task_i\}_{i=1}^N$ are training tasks for meta-learners to train on (i.e., for meta-training). 
In the inference stage, an arbitrary test task $\task = (\xyxy)$ is picked, and the labelled support samples $(X',Y')$ are given to the trained meta-learner as input, then the meta-learner is asked predict the labels of the query samples $X$ from $\task$.

\textbf{Remarks.} This few-shot learning problem above can also be called a $n$-shot $k$-way learning problem. See Fig. \ref{fig:domain-gen} for an illustration of this problem setting.

\subsection{Gradient-Based Meta-Learning}\label{sec:GBML}

\begin{algorithm}[t]
\caption{MAML for Few-Shot Learning}
\label{alg:mamlsup}
\begin{algorithmic}[1]
{\footnotesize
\REQUIRE $\{\task_i\}_{i=1}^N$: Training Tasks
\REQUIRE $\eta$, $\lambda$: Learning rate hyperparameters
\STATE Randomly initialize $\theta$
\WHILE{not done}
  \FORALL{$\task_i$} 
    \STATE Evaluate the loss of $f_\theta$ on support samples of $\task_i$: $ \ell(f_\theta(X_i'),Y_i')$
    \STATE Compute adapted parameters $\theta_i'$ with gradient descent: $\theta_i'=\theta-\lambda \nabla_\theta \ell(f_\theta(X_i'),Y_i')$
    \STATE Evaluate the loss of $f_{\theta_i'}$ on query samples of $\task_i$: $\ell(f_{\theta_i'}(X_i), Y_i)$
 \ENDFOR
 \STATE Update parameters with gradient descent:\\
    $\theta \leftarrow \theta - \eta \nabla_\theta \sum_{i = 1}^N  \ell ( f_{\theta_i'}(X_i),Y_i)$ 
\ENDWHILE
}
\end{algorithmic}
\end{algorithm}
GBML is equipped with parametric models, which are almost always neural networks in practice. Consider a parametric model $f$ with parameters $\theta$ such that $f_\theta: \mathbb{R}^d \mapsto \mathbb{R}^k$, and its output on arbitrary sample $x\in \mathbb{R}^d$ is $f_\theta(x)$. Suppose an arbitrary task is given as $\task = (\xyxy)$. In GBML, it is helpful to define the \textit{meta-output}, $F: (X',Y') \mapsto f_{\theta'}$, a mapping depending on support samples and labels such that $F_\theta(\cdot, X',Y') = f_{\theta'}(\cdot)$, where $\theta'$ is the \textit{adapted parameters} depending on $\theta$ and $(X',Y')$. Specifically, we define $F$ as the vectorized output of the model $f$ with adapted parameters,
\begin{align}
\label{eq:meta-output}
F_\theta(\xxy) = f_{\theta'}\left(\sX\right) = \operatorname{vec} \pp{ \left [f_{\theta'} (x)\right]_{x \in X}} \in \mathbb R^{n k}
\end{align}
where the adapted parameters $\theta'$ is obtained as follows: 
use $\theta$ as the initial parameter and update it by $\tau$ steps of gradient descent on support samples and labels $(\sX',\sY')$, with learning rate $\lambda$ and loss function $\ell$. Mathematically, $\forall j=0,...,\tau-1$, we have
\begin{align}\label{eq:meta-adaption-descrete}
    &\theta=\theta_0, ~~\theta'=\theta_\tau, \theta_{j+1} = \theta_{j} - \lambda \nabla_{\theta_{j}} \ell(f_{\theta_j}(\sX'),\sY')
\end{align}
With the square loss function $\ell (\hat y, y)=\frac{1}{2}\|\hat y - y\|_2^2$, the training objective of MAML\footnote{\label{footnote:maml-variants}Although we only demonstrate the MAML objective in (\ref{eq:MAML-obj}), slight modifications to (\ref{eq:MAML-obj}) can convert it to many other GBML objectives, including 1st-order MAML \cite{maml,reptile}, Adaptive GBML \cite{adaptive-GBML}, WrapGrad \cite{WrapGrad} and Meta-Curvature \cite{meta-curvature}, etc.}
is
\begin{align} \label{eq:MAML-obj}
\mathcal L(\theta) &= \sum_{i=1}^N \ell (F_\theta(\xxyi), Y_i)
\nonumber\\
&=\frac{1}{2}\sum_{i=1}^N\|F_\theta(\xxyi)-Y_i\|_2^2\nonumber\\ &=\frac{1}{2}\|F_\theta(\XXY)-\Y\|_2^2 
\end{align}
where $F_\theta(\XXY)\equiv \left(F_\theta(\xxyi)\right)_{i=1}^N = \operatorname{vec}\pp{\left[ F_\theta(\xxyi)\right]_{i\in[N]}}$ is the 
concatenation of meta-outputs on \textit{all} training tasks. 

\textbf{Remarks.} We provide the algorithm of MAML for few-shot learning in Algorithm \ref{alg:mamlsup}. For simplicity, Algorithm \ref{alg:mamlsup} shows MAML with \textit{one-step} meta-adaptation, which is equivalent to the case that $\tau = 1$ in \eqref{eq:meta-adaption-descrete}. 
In this paper, we focus on MAML, but our results could also be extended to variants of MAML\footref{footnote:maml-variants}.

\section{Global Convergence of Gradient-Based Meta-Learning with Deep Neural Nets}
\label{sec:global-convergence}

In this section, we will show for sufficiently over-parameterized neural networks, GBML is guaranteed to convergence to global optima under gradient descent at a linear rate. This convergence analysis also gives rise to an analytical expression of GBML output. Besides, in the infinite width limit (i.e., neural nets are extremely over-parameterized), we prove that GBML is equivalent to a kernel regression with a new class of kernels, which we name as Meta Neural Kernels.
\paragraph{Notation and Setup}Consider a neural network $f_\theta$ with $L$ hidden layers,
where $\theta \in \mathbb R^D$ 
is a vector containing all the parameters of the network.
For $i\in[L]$, we use $l_i$ to denote the width of the $i$-th hidden layer. In this paper, we consider all hidden layers have the same width $l$ for simplicity\footnote{This same-width assumption is not a necessary condition. One can also define $l=\min{\{l_i\}_{i=1}^L}$ instead and all theoretical results in this paper still hold true.}, i.e., $l_1=l_2=\dots=l_L=l$.
For notational convenience, we denote the Jacobian of the meta-output on training data as $J(\theta)=\nabla_\theta F_{\theta}(\XXY)$, and define a kernel function as $\metantk_{\theta}(\cdot, \ast)\coloneqq \frac{1}{l}\nabla_\theta F_\theta(\cdot) \nabla_\theta F_\theta(\ast)^\top$. The notation $\theta_t$ represents the parameters at the training time $t$ (i.e., number of gradient descent steps). For convenience, we denote $F_t(\cdot) \equiv F_{\theta_t}(\cdot)$, $f_t(\cdot) \equiv f_{\theta_t}(\cdot)$ and $\metantk_t(\cdot,\ast) \equiv \metantk_{\theta_t}(\cdot,\ast)$. Besides, we define $\eta$ as the learning rate for gradient descent on the GBML objective, (\ref{eq:MAML-obj}); for any diagonalizable matrix $M$, we use $\lev(M)$ and $\Lev(M)$ to denote the least and largest eigenvalues of $M$. These notations are adopted in the entire paper.
\subsection{Global Convergence Theorem}
To prove the global convergence of GBML with DNNs, we need to first prove the Jacobian of the meta-output, $J$, changes locally in a small region under perturbations on network parameters, $\theta$. Because of the non-convexity of DNNs and the bi-level formulation of GBML, it is non-trivial to obtain such a result. However, we manage to prove this by developing a novel analysis to bound the change of Jacobian under parameter perturbations, shown below as a lemma, with detailed proof in Appendix \ref{supp:global-convergence}.
\begin{lemma}[Local Lipschitzness of Jacobian]
\label{lemma:local-Liphschitzness} 
Suppose\footnote{This assumption is realistic in practice. For example, the official implementation of MAML \cite{maml} for few-shot classification benchmarks adopts (i) $\tau=1, \lambda = 0.4$ and (ii) $\tau=5,\lambda=0.1$, which both satisfy our assumption.} $\tau = \cO(\frac{1}{\lambda})$, then there exists $K>0$ and $l^*>0$ such that: $\forall ~C>0$ and $l > l^*$, the following inequalities hold true with high probability over random initialization,
\begin{align}\label{eq:jacobian-lip}
 \forall \theta, \, \Bar \theta \in B(\theta_0, C l^{-\frac 1 2}),  \begin{cases}  
    \frac 1 {\sqrt l}\|J(\theta) - J(\Bar \theta)\|_{F} \leq K\|\theta - \Bar \theta\|_2
    \\
    \\
    \frac 1 {\sqrt l} \|J(\theta)\|_{F}  \leq K 
    \end{cases}
\end{align}
where $B$ is a neighborhood defined as
\begin{align}
\label{def:B}
    B(\theta_0, R) := \{\theta: \|\theta-\theta_0\|_2 < R\}.   
\end{align}
\end{lemma}
Suppose the neural net is sufficiently over-parameterized, i.e., the width of hidden layers, $l$, is large enough. Then, we can prove GBML with this neural net is guaranteed to converge to global optima with zero training loss at a linear rate, under several mild assumptions. The detailed setup and proof can be found in Appendix \ref{supp:global-convergence}. We provide a simplified theorem with a proof sketch below.
\begin{theorem}[Global Convergence]\label{thm:global-convergence}
Define $\metaNTK = \lim_{l\rightarrow \infty} \frac{1}{l}J(\theta_0)J(\theta_0)^T$. For any $\delta_0 > 0$, $\eta_0 < \frac{2}{\Lev(\metaNTK) + \lev(\metaNTK)}$, and $\tau = \cO(\frac{1}{\lambda})$
there exist $\lss>0$, $\Lambda \in\mathbb N$, $K>1$, and $\lambda_0>0$, such that: for width $l\geq \Lambda$, running gradient descent with learning rates $\eta = \frac {\eta_0}{l}$ and $\lambda < \frac{\lambda_0}{l}$ over random initialization, the following upper bound on the training loss holds true with probability at least $(1 - \delta_0)$:
\begin{align}
    \loss(\theta_t) &= \frac{1}{2}\|F_{\theta_t}(\XXY) - \Y \|_2^2\nonumber\\ 
    &\leq  \left(1 - \frac {\eta_0 \lev(\metaNTK)}{3}\right)^{2t} \frac{\lss^2}{2}\, .  
    \label{eq:convergence-loss}
\end{align}
\end{theorem}
\begin{proofsketch}
First, we consider the Jacobian of the meta-output, $J$, and prove a lemma showing $J$ has bounded norm. Then we prove another lemma showing $\metaNTK$ is a deterministic matrix over random initialization of $\theta_0$. By these lemmas and Lemma \ref{lemma:local-Liphschitzness}, we analyze the optimization trajectory of the neural net parameters, and prove that the parameters move locally during optimization, and the training loss exponentially decays as the number of optimization steps increases, indicating the training loss converges to zero at a linear rate, shown as (\ref{eq:convergence-loss}).
\end{proofsketch}

With this global convergence theorem, we can derive an analytical expression for GBML output at any training time, shown below as a corollary, with proof in Appendix \ref{supp:global-convergence}.

\begin{corollary}[Analytic Expression of Meta-Output]\label{corr:GBML-output}
In the setting of
Theorem \ref{thm:global-convergence}, the training dynamics of the GBML can be described by a differential equation
$$\frac{d F_t(\XXY)}{d t}=- \eta  \, \metantk_0   (F_t(\XXY) - \Y)$$
where we denote $F_t \equiv F_{\theta_t}$ and $\metantk_0 \equiv \metantk_{\theta_0}((\XXY),(\XXY))$ for convenience.

Solving this differential equation, we obtain the meta-output of GBML on training tasks at any training time as 
\begin{align}
    F_t(\XXY)=(I - e^{- \eta\metantk_0 t})\Y + e^{-\eta \metantk_0 t}F_{0}(\XXY) \,. \label{eq:ODE-sol-F}
\end{align}
Similarly, on arbitrary test task $\task=(\xyxy)$, the meta-output of GBML is
\begin{align}
&F_t(\xxy) =F_{0}(\xxy)  \label{eq:F_t:main_text}\\
&\quad \quad \qquad+  \metantk_0(\xxy) \T^{\eta}_{\metantk_0}(t)\left(\Y-F_0(\XXY)\right)\nonumber
\end{align}
where $\metantk_0(\cdot)\equiv \metantk_{\theta_0}(\cdot,(\XXY))$ and $\T^{\eta}_{\metantk_0}(t)=\metantk_0^{-1}\left(I- e^{-\eta\metantk_0 t}\right)$ are shorthand notations.

\end{corollary}

\textbf{Remarks.} This corollary implies for a sufficiently over-parameterized neural network, the training of GBML is \textit{determined} by the \textit{parameter initialization}, $\theta_0$. Given access to $\theta_0$, we can compute the functions $\metantk_0$ and $F_0$, and then the trained GBML output can be obtained by simple calculations, without the need for running gradient descent on $\theta_0$. This nice property enables us to perform deeper analysis on GBML with DNNs in the following sections.

\subsection{Gradient-Based Meta-Learning as Kernel Regression}
\label{sec:MetaNTK}

The following theorem shows that as the width of neural nets approaches infinity, GBML becomes equivalent to a kernel regression with a new class of kernels, which we name as Meta Neural Kernels (MNK). We also provide an analytical expression for the kernels. The proof of this theorem is in Appendix. \ref{supp:MetaNTK}.
\begin{theorem}[GBML as Kernel Regression]\label{thm:MNK}
Suppose learning rates $\eta$ and $\lambda$ are infinitesimal. As the network width $l$ approaches infinity, with high probability over random initialization of the neural net, the GBML output, (\ref{eq:F_t:main_text}), converges to a special kernel regression,
\begin{align}\label{eq:F_t-MetaNTK}
&F_t(\xxy)= G_\NTK^{\tau}(\xxy) \\
&\qquad +\metaNTK((\xx),(\XX)) \T^{\eta}_{\metaNTK}(t) \left(\Y-G_{\NTK}^{\tau}(\XXY)\right)\nonumber
\end{align}
where $G$ is a function defined below, $\NTK$ is the neural tangent kernel (NTK) function from \cite{ntk} that can be analytically calculated without constructing any neural net, and $\metaNTK$ is a new kernel, which we name as Meta Neural Kernel (MNK). The expression for $G$ is
\begin{align}\label{eq:G}
    G_\NTK^\tau(\xxy) =  \NTK(X,X')\widetilde{T}^{\lambda}_\NTK (X',\tau)   Y'.
\end{align}
where $\widetilde{T}^{\lambda}_\NTK (\cdot,\tau) \coloneqq \NTK(\cdot,\cdot)^{-1}(I-e^{-\lambda \NTK(\cdot,\cdot) \tau}) $. Besides, $G_\NTK^{\tau}(\XXY) = (G_\NTK^{\tau}(\xxyi))_{i=1}^N$.

The MNK is defined as $\metaNTK \equiv \metaNTK((\XX),(\XX)) \in \mathbb{R}^{knN \times knN}$, which is a block matrix that consists of $N \times N$ blocks of size $kn\times kn$. For $i,j \in [N]$, the $(i,j)$-th block of $\metaNTK$ is
\begin{align} \label{eq:MetaNTK_ij=kernel}
    [\metaNTK]_{ij}=\SingleTaskmetaNTK((\xxi),(\xxj)) \in \mathbb R^{kn \times kn} , 
\end{align}
where $\SingleTaskmetaNTK: (\mathbb{R}^{n \times k} \times \mathbb{R}^{m\times k}) \times  (\mathbb{R}^{n \times k} \times \mathbb{R}^{m\times k}) \rightarrow \mathbb{R}^{nk \times nk}$ is a kernel function defined as
\begin{align}\label{eq:MetaNTK_ij}
    &\qquad \SingleTaskmetaNTK((\cdot,\ast), (\bullet, \star)) \\
    &= \NTK(\cdot,\bullet) + \NTK(\cdot,\ast)\widetilde{\T}_{\NTK}^\lambda(\ast,\tau)\NTK(\ast,\star)\widetilde{\T}_{\NTK}^\lambda(\star,\tau)^\top \NTK(\star,\bullet) \nonumber\\
 &~ -\NTK(\cdot,\ast)\widetilde{\T}_{\NTK}^\lambda(\ast,\tau) \NTK(\ast,\bullet) - \NTK(\cdot,\star) \widetilde{\T}_{\NTK}^\lambda(\star,\tau)^\top \NTK(\star,\bullet) .\nonumber
    \end{align}
The $\metaNTK((\xx),(\XX)) \in \mathbb{R}^{kn \times knN}$ in (\ref{eq:F_t-MetaNTK}) is also a block matrix, which consists of $1 \times N$ blocks of size $k n \times k n$, with the $(1,j)$-th block as follows for $j \in [N]$, $$[\metaNTK((\xx),(\XX))]_{1,j}= \SingleTaskmetaNTK((X,X'),(X_j,X_j')).$$
\end{theorem} 

\textbf{Remarks.} The kernel $\metaNTK$ is in fact what $\metantk_{0}$ converges to as the neural net width approaches infinity. However, $\metantk_{0} \equiv \metantk_0((\XXY),(\XXY))$ depends on $\Y$ and $\Y'$, while $\metaNTK \equiv \metaNTK((\XX),(\XX))$ does not, since the terms in $\metaNTK$ that depend on $\Y$ or $\Y'$ all vanish as the width approaches infinity. Besides, (\ref{eq:F_t-MetaNTK}) is a sum of two kernel regression terms, but it can be viewed as a single special kernel regression. Notably, the new kernel $\metaNTK$ can be seen as a \textit{composite} kernel built upon the \textit{base} kernel function $\NTK$.

\section{Generalization of Gradient-Based Meta-Learning with Neural Nets}

In this section, we still consider over-parameterized neural nets as models for training, and we first demonstrate that the effect of GBML can be viewed as a \textit{functional gradient descent} operation. Specifically, the outputs of meta-learners (i.e., models with meta-training) are equivalent to functions obtained by a \textit{functional gradient descent} operation on the outputs of base-learners (i.e., models without meta-training). Inspired by this observation, we focus on the \textit{functional gradient} term, and prove a generalization bound on GBML with over-parameterized DNNs, based on results of Theorem \ref{thm:MNK}.
\label{sec:generalization}
\subsection{A \textit{Functional Gradient Descent} View of Gradient-Based Meta-Learning}
\label{sec:FGD}
\begin{figure*}[ht!]
\centering
    \subfloat[Training Tasks]{\includegraphics[width=0.4\linewidth]{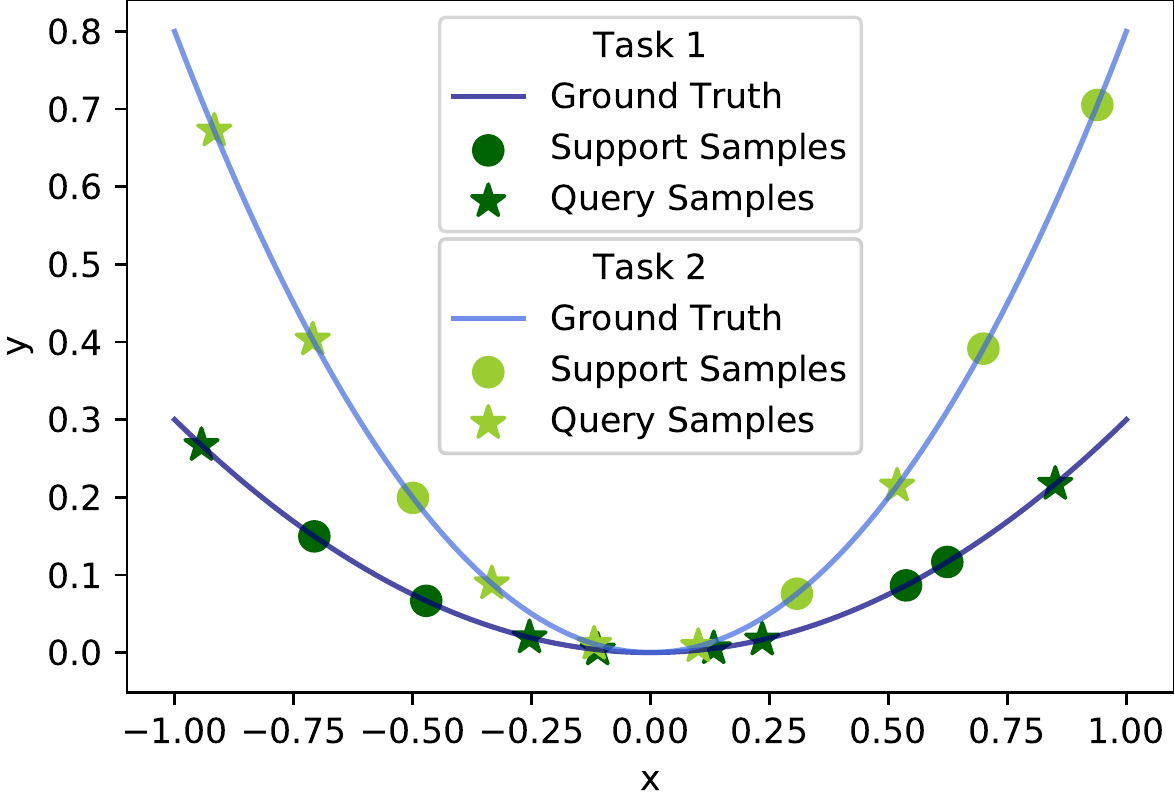}}\label{fig:quadratic:train}%
    \qquad
    \subfloat[A Test Task]{\includegraphics[width=0.4\linewidth]{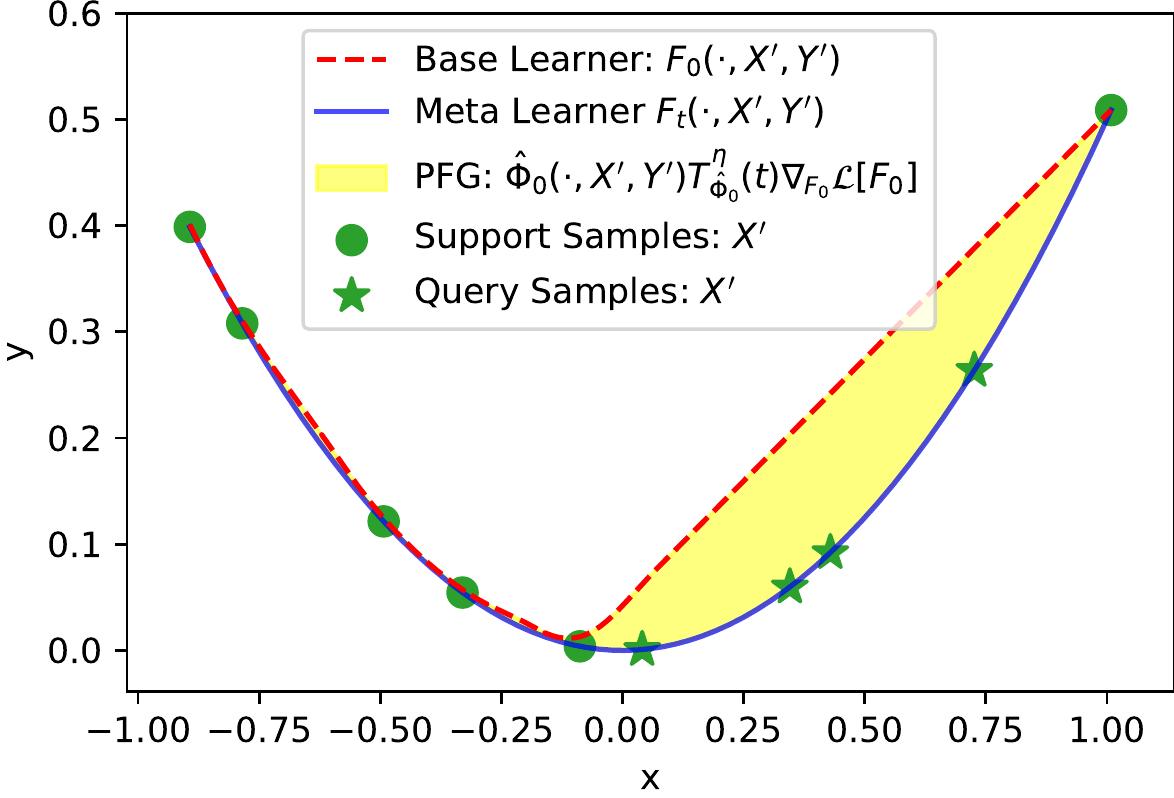}}%
    \caption{Illustration of the example in Sec. \ref{sec:quadratic-example}.
    (a) shows two training tasks with their support/query samples and ground-truth objective functions. (b) first shows a test task with support and query samples. Besides, it demonstrates several terms defined in (\ref{eq:MetaNTK:functional-GD}): the predictions of the meta learner and the base learner, and the projected functional gradient (PFG) term, which is equal to the difference between the first two terms. 
    }
    \label{fig:quadratic-functional-GD}
\end{figure*}
The empirical success of GBML methods is mainly due to their ability to learn good DNN initializations for adaption on new tasks \cite{Finn:EECS-2018-105}. However, it is not theoretically clear why these learned initializations are effective. 
Below, we provide some theoretical insight on this problem by demonstrating equivalence between GBML with DNNs and a functional gradient descent operation.

Note that the GBML output (\ref{eq:F_t:main_text}) can be rewritten as
\begin{align}\label{eq:MetaNTK:functional-GD}
     \underbrace{F_t(\xxy)}_{\textit{Meta Learner}}&=\underbrace{F_0(\xxy)}_{\textit{Base Learner}} \\ &-\underbrace{\overbrace{\metantk_0(\xxy)\T^{\eta}_{\metantk_0}(t)}^{\textit{Projection}} \overbrace{\nabla_{F_0}\Lf [F_0]}^{\textit{Functional Gradient}}}_{\textit{Projected Functional Gradient}}\nonumber
\end{align}
where $\Lf$ is a loss \textit{functional} (i.e., function with functions as input) such that for any function $h$,
$$\Lf[h]= \frac{1}{2}\|h(\XXY)-\Y)\|_2^2,$$
with the corresponding \textit{functional gradient} on the function $h$ as
$$\nabla_h \Lf[h]= \nabla_{h(\XXY)}\Lf[h]= h(\XXY)-\Y.$$
Obviously, (\ref{eq:MetaNTK:functional-GD}) can be seen as a \textit{projected functional gradient descent} operation on $F_0$, with learning rate equal to $1$ and $\metantk_0(\xx) \T^{\eta}_{\metantk}(t)$ as the \textit{projection}. 

The function $F_0$ can be viewed as the output of a purely \textit{supervised learning} model that has no thing to do with meta-learning, since \eqref{eq:meta-output} shows that $F_0(\xxy)=F_{\theta_0}(\xxy)=f_{\theta_0'}(X)$, where $\theta_0'$ is the adapted parameters on random initialized parameters $\theta_0$ trained under $\tau$ steps of gradient descent on $(X',Y')$. In other words, $F_0$ can be viewed as a \textit{base learner} (i.e., a supervised learning model), and the goal of GBML is to train a \textit{meta-learner} (i.e., a meta-learning model), $F_t$, to improve over the base learner on test tasks, by utilizing the prior knowledge on training tasks $\{\task_i\}_{i=1}^N$.

From \eqref{eq:MetaNTK:functional-GD}, we can observe that for over-parameterized DNNs, the effect of GBML is solely the \textit{projected functional gradient} term. This term can be viewed as an \textit{error correction} term to the base learner $F_0$, which propagates prior knowledge on training tasks $(\XYXY)$ to the base learner on the test task $\task = (\xyxy)$ to reduce its test error.

\subsubsection{An 1-d Example}\label{sec:quadratic-example}

To illustrate the equivalence between GBML and functional gradient descent derived in (\ref{eq:MetaNTK:functional-GD}) more intuitively, we present a simple but insightful example of few-shot learning, \textit{1-$d$ few-shot regression with quadratic objective functions}, in which all samples and labels are scalars.

Specifically, for arbitrary training/test task $\task = (\xyxy)$, we assume all samples in $X$ and $X'$ are drawn uniformly from $[0,1]$, and the relation between all samples and their labels is determined by a unique scalar variable, $\alpha \sim \text{Unif}(0,1)$, such that $Y = \alpha X^2 ~\textit{and} ~ Y' = \alpha {X'}^2$.
Each training task in $\{\task_i\}_{i=1}^N$ has its own $\alpha$, so does each test task. It is then natural to expect that a trained meta-learner can comprehend the quadratic nature of objective functions from training tasks and then utilize this prior knowledge to achieve fast adaption to any test task with only a few samples.

Fig. \ref{fig:quadratic-functional-GD} illustrates the training and test tasks. A trained meta-learner trained over training tasks, $F_t$, should predict well on the query samples of this test task, even with only a few unevenly distributed support samples, since it comprehends the quadratic nature of objective functions from the meta-training. However, a base learner, $F_0$, which has the same DNN structure and initial parameters as the meta-leaner, cannot accurately predict on query samples, since it does not know the quadratic nature of objective functions. The difference between the base-learner and the meta-learner is exactly the projected funtional gradient (PFG) term, indicating the PFG term is the reason for the improvement of meta-learner over the base-learner.

\subsection{Generalization Bound of Gradient-Based Meta-Learning}
\label{sec:generalization:gen-bound}

\textbf{Task Distrbution.} A fundamental assumption of meta-learning (i.e., learning to learn) is that all training and test tasks share some \textit{across-task knowledge} or \textit{meta-knowledge} \cite{learningtolearn,hospedales2020metalearning}. In this way, meta-learning algorithms aim to learn the across-task knowledge from training tasks, and utilize it to achieve fast adaptation to test tasks. For instance, in the example of Sec. \ref{sec:quadratic-example}, the across-task knowledge of training/test tasks is the quadratic nature of their objective functions, and a meta-learner which understands this knowledge can easily fit test tasks accurately with a few samples. In the literature of meta-learning and domain generalization \cite{baxter2000model,muandet2013domain,maurer2016benefit,albuquerque2020generalizing,saunshi2020sample,hospedales2020metalearning}, the \textit{across-task knowledge} is usually mathematically described by the definition of \textit{task distribution}: all training and test tasks are assumed to be drawn i.i.d. from a \textit{task distribution}, $\mathscr{P}$. Specifically, for the generation of any task $\task=(\xyxy)$, firstly, a data distribution $\mathbb{P}_\task$ is drawn from the task distribution, i.e., $\mathbb{P}_\task \sim \mathscr{P}$; then, the support and query samples with labels are drawn i.i.d. from this data distribution, i.e., $(x,y)\sim \mathbb{P}_\task \text{ for any } (x,y)\in(X,Y)$ and $(x',y')\sim \mathbb{P}_\task \text{ for any } (x',y')\in(X',Y')$. A schematic diagram is shown in Fig. \ref{fig:domain-gen} to illustrate the notion of \textit{task distribution}.

\textbf{Reformulation.} From \eqref{eq:MetaNTK:functional-GD}, we can see the term $F_0(\xxy)$ is fixed during meta-training, since it is independent of the training time $t$. Furthermore, as discussed in Sec. \ref{sec:FGD}, only the Projected Functional Gradient (PFG) term in \eqref{eq:MetaNTK:functional-GD} is related to the effect of GBML. Thus, we reformulate the test loss of GBML to isolate the PFG term, which could further help us study the generalization of GBML. Specifically, for any test task $\task=(\xyxy)$, we define
\begin{align}\label{eq:gen:def-terms}
\left\{
\begin{aligned}
\wt Y &= Y - F_0(\xxy)\\
\wt \Y &= \Y - F_0(\XXY)\\
\widetilde F_t(\cdot)&=F_t(\cdot) - F_0(\cdot)
\end{aligned}
\right.
\end{align}
Based on \eqref{eq:gen:def-terms} and \eqref{eq:MetaNTK:functional-GD}, we know
\begin{align}
\widetilde F_t(\xxy) &= \metantk_0(\xxy)\T^\eta_{\metantk_0}(t) \cdot \widetilde \Y
\end{align}

Then, the test loss on $\task$ at training time $t$ becomes
\begin{align}
    \loss_\task (t) 
    &= \frac{1}{2}\|F_t(\xxy)-\sY\|_2^2 \label{eq:MetaNTK:GBML-test-loss} \\
    &=\frac 1 2 \| F_0(\xxy) + \widetilde F_t(\xxy)-Y \|_2^2 \nonumber\\
    &=\frac 1 2 \|\widetilde F_t(\xxy) - \widetilde \sY\|_2^2 \label{eq:MetaNTK:supervised-learning} \\
    &= \frac 1 2 \|\metantk_0(\xxy)\T^\eta_{\metantk_0}(t) \cdot \widetilde \Y - \widetilde \sY  \|_2^2 \label{eq:gen:reform:test-loss}
\end{align}
It can be seen that \eqref{eq:gen:reform:test-loss} is in a similar form to the test loss of a kernel regression. This motivates us to study the generalization properties of GBML from a kernel view.

\textbf{Challenges.} We cannot directly apply a kernel generalization bound to GBML, since the equivalence between GBML and kernels only holds true in the \textit{infinite} width limit of DNNs (see Theorem \ref{thm:MNK}), while practical DNNs are \textit{finitely} wide for sure. As one studies the generalization properties of \textit{finitely} wide DNNs, even in the simple supervised learning setting, some non-trivial challenges emerge \cite{AllenZhu2019Generalization,fine-grained,cao2019generalization}. For instance, one need to deal with the initialization scheme, algorithm details, and optimization trajectories when studying DNN generalization. Moreover, the bi-level formulation of GBML and the few-shot learning setting makes our generalization analysis more challenging than the cases of supervised-learning.

\textbf{Theoretical Results.} We present a generalization bound in Theorem \ref{thm:gen-bound} below, which is related to the Meta Neural Kernel derived in Theorem \ref{thm:MNK}. The full proof of Theorem \ref{thm:gen-bound} along with several helper lemmas is attached in Appendix \ref{supp:generalization}, and we provide a brief sketch of the proof in this paragraph. Firstly, we consider an (finitely wide) over-parameterized DNN meta-trained by stochastic gradient descent (SGD) w.r.t. the GBML objective \eqref{eq:MAML-obj}. Secondly, we define a random feature function class by the gradient of the meta-output \eqref{eq:MetaNTK:functional-GD}, and then prove a cumulative loss bound on the meta-trained DNN, which leads to a generalization bound by an online-to-batch conversion \cite{cesa2004generalization}. Furthermore, we relate this bound with the Meta Neural Kernel (MNK) derived in Theorem \ref{thm:MNK}, and finally provide an MNK-based generalization bound for GBML with over-parameterized DNNs, which is shown in Theorem \ref{thm:gen-bound}.

\begin{figure}[t!]
    \centering
    \includegraphics[width=0.5\textwidth]{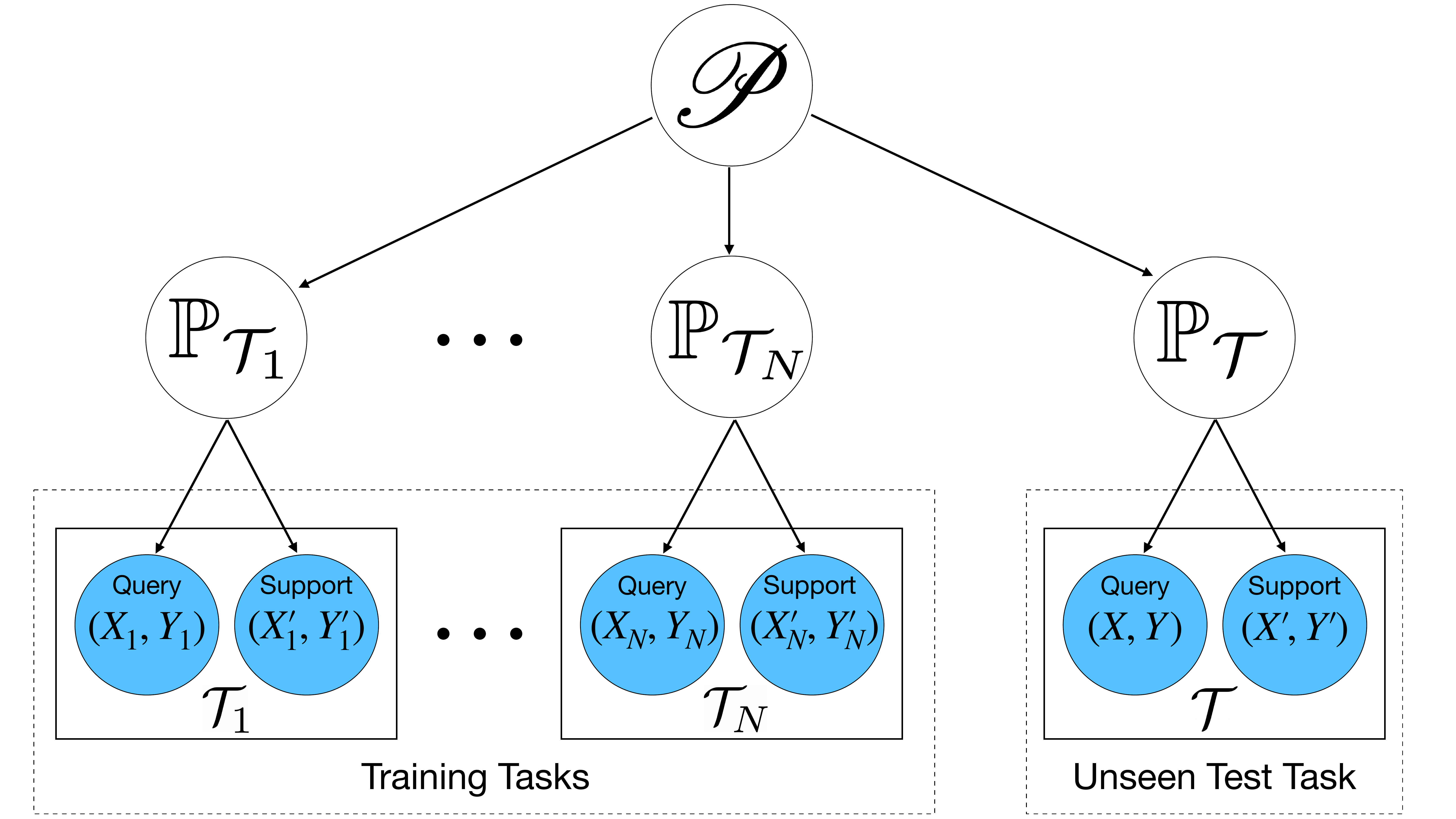}
    \caption{A schematic diagram of a few-shot learning problem with the \textit{task distribution} defined in Sec. \ref{sec:generalization:gen-bound}.
    }
    \label{fig:domain-gen}
\end{figure}

\begin{theorem}[Generalization Error Bound (\textit{Informal})]\label{thm:gen-bound}

Suppose training and test tasks are drawn i.i.d. from a task distribution $\mathscr{P}$. Let $\loss_{\mathscr{P}}$ be the expected loss of GBML tasks drawn i.i.d. from $\mathscr{P}$. Then, we show that for wide enough neural networks with $L$ hidden layers, $\loss_{\mathscr{P}}$ is bounded as the following with high probability,
\begin{align}\label{eq:MetaNTK:gen-bound}
    \loss_{\mathscr{P}} \leq \widetilde{\mathcal O}\left((L+1)\cdot \sqrt{\frac{\widetilde{\Y}_{\scriptscriptstyle{G}}^\top \metaNTK^{-1} \widetilde{\Y}_{\scriptscriptstyle G}}{Nn} }~\right)~.
\end{align}
where $\metaNTK$ is computed following \eqref{eq:MetaNTK_ij=kernel}, and $\wt \Y_{\scriptscriptstyle G}$ can be analytically expressed as
\begin{align}\label{eq:def:Y_G}
    \wt{\Y}_{\scriptscriptstyle G} = \Y - G_\NTK^\tau(\XXY)
\end{align}
with the function $G$ defined in \eqref{eq:G}. 
\end{theorem}

\textbf{Remarks on Data-Dependence.} Notice the generalization bound \eqref{eq:MetaNTK:gen-bound} is \textit{data-dependent}, since both $\wt{\Y}_{\scriptscriptstyle G}$ and $\metaNTK$ are analytically computed based on the training data $\{\task_i\}_{i=1}^N$. This indicates that the generalization bound relies on the quality of training data and properties of the task distribution. For instance, in the example of Sec. \ref{sec:quadratic-example}, if we replace the quadratic objective functions with uncorrelated random piecewise functions, the generalization bound \eqref{eq:MetaNTK:gen-bound} should predict a higher value of generalization error, since the across-task knowledge (i.e., quadratic nature of objective functions) in training data is annihilate by this replacement. See Sec. \ref{sec:exp:gen-bound} for empirical studies on the data-dependence of this generalization bound.

\section{Empirical Validation}\label{sec:exp:gen-bound}

\begin{figure*}[ht!]
    \subfloat[Synthetic Dataset]{\includegraphics[scale=0.5]{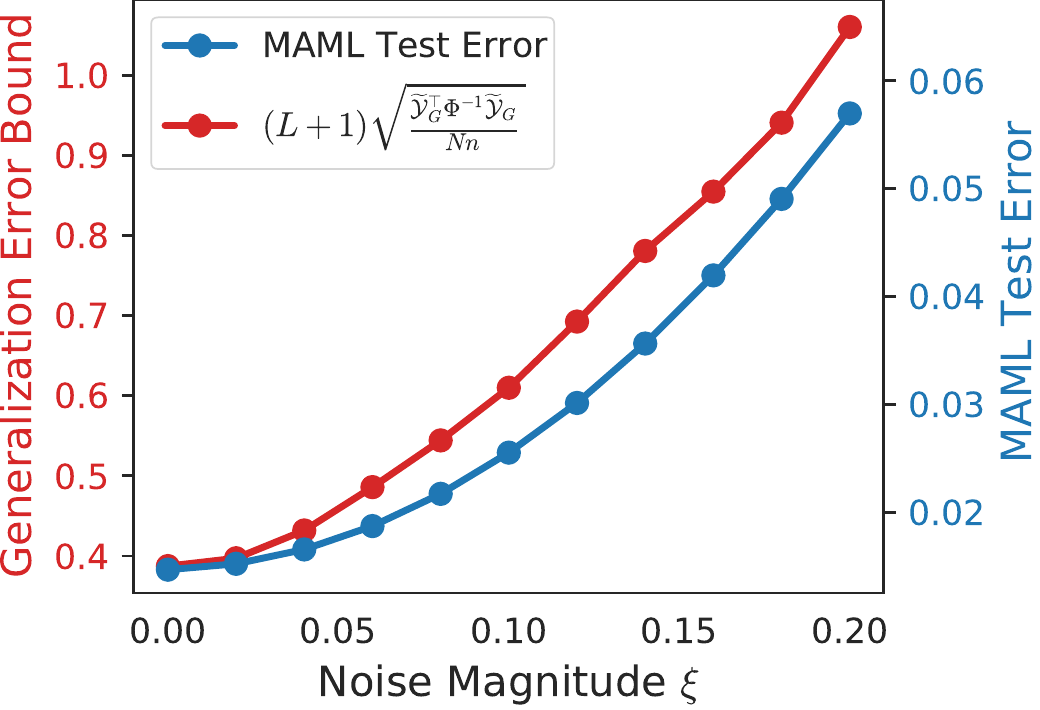}}
    \qquad
    \subfloat[Synthetic Dataset {\tiny (Normalized  $\wt Y_{\scriptscriptstyle G}$)}]{\includegraphics[scale=0.51]{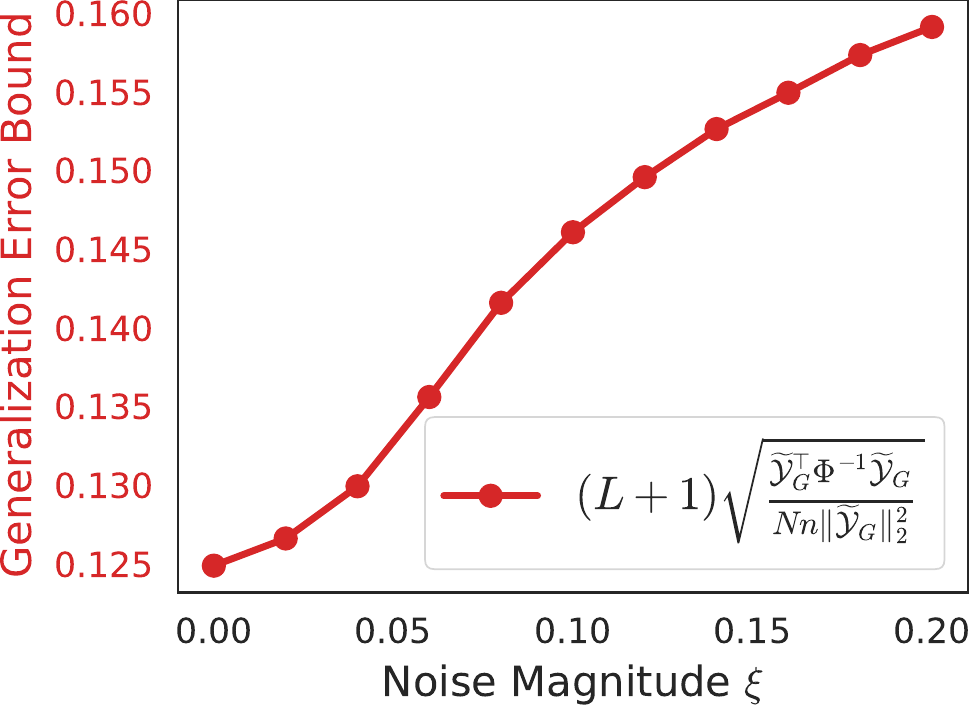}}%
    \qquad
    \subfloat[Omniglot Dataset]{\includegraphics[scale=0.52]{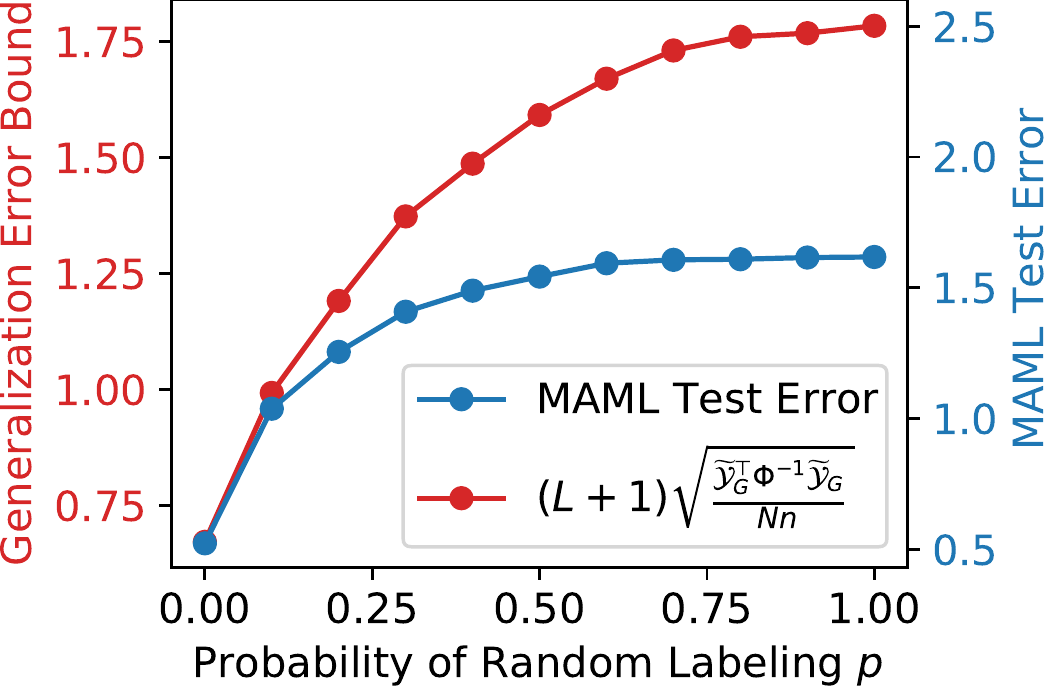}}%
    
    \caption{(a) Generalization error bound vs. MAML test error on the synthetic dataset. (b) Generalization error bound with normalized $\wt Y_{}$
    on the synthetic dataset. (c) Generalization error bound vs. MAML test error on the Omniglot dataset. Note the scale of the generalized error bound is not the same as the MAML test error, since there exists a $\wt \cO(\cdot)$ for the bound in \eqref{eq:MetaNTK:gen-bound}.}\label{fig:random-label} 
\end{figure*}

In this section, we empirically validate the generalization bound of Theorem \ref{thm:gen-bound} on a synthetic dataset and a real-world dataset. Our code is written in Python. We also adopt a popular PyTorch implementation\footnote{\url{https://github.com/facebookresearch/higher/}} of MAML in our experiments. Due to space limit, we put more details about our implementation and results in Appendix \ref{supp:exp:gen-bound}.

\subsection{Synthetic Dataset}\label{sec:exp:gen-bound:synthetic}
We consider the problem setting of \textit{1-d few-shot regression with quadratic objective functions} described in Sec. \ref{sec:quadratic-example}, and generate $N=40$ training tasks along with $40$ test tasks following that setting. As discussed in Sec. \ref{sec:exp:gen-bound}, the generalization bound \eqref{eq:MetaNTK:gen-bound} of Theorem \ref{thm:gen-bound} is \textit{data-dependent}, indicating that it gives different values for different task distributions. Here, we manually add \textit{label noises} to the training and test tasks to gradually change the task distribution, while we compare the value of the generalization error bound \eqref{eq:MetaNTK:gen-bound} against MAML test error in this process. The goal of this experiment is to validate the generalization bound by confirming that it can predict behaviours of practical GBML algorithms correctly. Specifically, for any training/test task $\task = (\xyxy)$, we add i.i.d. Gaussian noises drawn from $\mathcal{N}(0,\xi^2)$ to every label $y\in Y$ and $y'\in Y'$, where $\xi>0$ is the \textit{noise magnitude} parameter we vary in the experiment. Fig. \ref{fig:random-label}(a) shows that the values of our generalization bound are aligned with the MAML test error as the noise magnitude $\xi$ increases. However, one may be concerned that larger label noises result in greater values of $\|\wt \Y_{\scriptscriptstyle G}\|_2$, which is the only reason to the increase of $(L+1) \sqrt{\frac{\widetilde{\Y}_{\scriptscriptstyle{G}}^\top \metaNTK^{-1} \widetilde{\Y}_{\scriptscriptstyle G}}{Nn} }$. Fig. \ref{fig:random-label}(b) resolves this concern by showing that normalizing  $\widetilde{\Y}_{\scriptscriptstyle{G}}$ does not change the increasing trend of the generalization bound. Hence, we can conclude that our generalization bound is validated on this dataset.

\subsection{Omniglot Dataset}\label{sec:exp:gen-bound:omniglot}
The Omnitlot dataset is a popular benchmark for few-shot learning \cite{omniglot}. Following the convention of meta-learning \cite{maml}, we adopt the standard \textit{1-shot 5-way image classification} problem setting (i.e., $N=1$, $k=5$ in our notation). Similar to Sec. \ref{sec:exp:gen-bound:synthetic}, we want to compare the generalization error bound \eqref{eq:MetaNTK:gen-bound} against MAML test error as more label noises are added. Since class labels are discrete for the Omniglot dataset, we perform \textit{random labeling} to add label noises. Specifically, we define a parameter $p\in[0,1]$, which represents the \textit{probability of random labeling}. Then, for any training or test task,  \textit{with probability} $p$, each query/support label is replaced with an independent random class label. From Fig. \ref{fig:random-label}(b), we can see that as $p$ increases, the values of the generalization error bound are aligned with MAML test error values. Therefore, our generalization bound is also validated on the Omniglot dataset.

\textbf{Remarks.} Please see Appendix \ref{supp:exp:gen-bound} for more details about the experiments in Sec. \ref{sec:exp:gen-bound:synthetic} and \ref{sec:exp:gen-bound:omniglot}, including specification of both datasets, computation of the generalization bound, and hyper-parameter choices, etc.

\section{Conclusion}

This paper studies the \textit{optimization} and \textit{generalization} properties of \textit{gradient-based meta-learning} (GBML) equipped with deep neural networks (DNNs). First, we prove the \textit{global convergence} of GBML with over-parameterized DNNs. Based on the convergence analysis, we prove that in the infinite width limit of DNNs, GBML converges to a kernel regression with a new class of kernels, which we name as \textit{Meta Neural Kernels}(MNKs). Then, we show an equivalence between GBML and a \textit{functional gradient descent} operation, which provides a theoretical verification of the intuition behind the design of GBML. Furthermore, this equivalence provides us a novel perspective to study the generalization of GBML, and we finally prove an MNK-based \textit{generalization error bound} for GBML with over-parameterized DNNs. To the best of our knowledge, it is the first generalization bound for GBML with non-linear DNNs. Empirically, we validate the effectiveness of this generalization error bound on both synthetic and realistic datasets. 

\section*{Acknowledgements}
Haoxiang Wang would like to thank Simon Du, Niao He, Han Liu, Yunan Luo, Ruosong Wang and Han Zhao for insightful discussions.
\bibliography{main}

\begin{thebibliography}{59}
\providecommand{\natexlab}[1]{#1}
\providecommand{\url}[1]{\texttt{#1}}
\expandafter\ifx\csname urlstyle\endcsname\relax
  \providecommand{\doi}[1]{doi: #1}\else
  \providecommand{\doi}{doi: \begingroup \urlstyle{rm}\Url}\fi

\bibitem[Albuquerque et~al.(2020)Albuquerque, Monteiro, Darvishi, Falk, and
  Mitliagkas]{albuquerque2020generalizing}
Isabela Albuquerque, João Monteiro, Mohammad Darvishi, Tiago~H. Falk, and
  Ioannis Mitliagkas.
\newblock Generalizing to unseen domains via distribution matching, 2020.

\bibitem[Allen-Zhu et~al.(2019{\natexlab{a}})Allen-Zhu, Li, and
  Liang]{AllenZhu2019Generalization}
Zeyuan Allen-Zhu, Yuanzhi Li, and Yingyu Liang.
\newblock Learning and generalization in overparameterized neural networks,
  going beyond two layers.
\newblock In \emph{Advances in neural information processing systems}, pages
  6155--6166, 2019{\natexlab{a}}.

\bibitem[Allen-Zhu et~al.(2019{\natexlab{b}})Allen-Zhu, Li, and
  Song]{AllenZhu2018ACT}
Zeyuan Allen-Zhu, Yuanzhi Li, and Zhao Song.
\newblock A convergence theory for deep learning via over-parameterization.
\newblock \emph{International Conference on Machine Learning},
  2019{\natexlab{b}}.

\bibitem[Arora et~al.(2019{\natexlab{a}})Arora, Du, Hu, Li, and
  Wang]{fine-grained}
Sanjeev Arora, Simon Du, Wei Hu, Zhiyuan Li, and Ruosong Wang.
\newblock Fine-grained analysis of optimization and generalization for
  overparameterized two-layer neural networks.
\newblock In \emph{International Conference on Machine Learning}, pages
  322--332, 2019{\natexlab{a}}.

\bibitem[Arora et~al.(2019{\natexlab{b}})Arora, Du, Hu, Li, Salakhutdinov, and
  Wang]{CNTK}
Sanjeev Arora, Simon~S. Du, Wei Hu, Zhiyuan Li, Ruslan Salakhutdinov, and
  Ruosong Wang.
\newblock On exact computation with an infinitely wide neural net.
\newblock \emph{NeurIPS}, 2019{\natexlab{b}}.

\bibitem[Arora et~al.(2020)Arora, Du, Li, Salakhutdinov, Wang, and
  Yu]{harnessNTK}
Sanjeev Arora, Simon~S. Du, Zhiyuan Li, Ruslan Salakhutdinov, Ruosong Wang, and
  Dingli Yu.
\newblock Harnessing the power of infinitely wide deep nets on small-data
  tasks.
\newblock In \emph{International Conference on Learning Representations}, 2020.
\newblock URL \url{https://openreview.net/forum?id=rkl8sJBYvH}.

\bibitem[Balcan et~al.(2019)Balcan, Khodak, and Talwalkar]{provable-gbml}
Maria-Florina Balcan, Mikhail Khodak, and Ameet Talwalkar.
\newblock Provable guarantees for gradient-based meta-learning.
\newblock In \emph{International Conference on Machine Learning}, pages
  424--433, 2019.

\bibitem[Bansal et~al.(2019)Bansal, Jha, and McCallum]{bansal2019learning}
Trapit Bansal, Rishikesh Jha, and Andrew McCallum.
\newblock Learning to few-shot learn across diverse natural language
  classification tasks, 2019.

\bibitem[Baxter(2000)]{baxter2000model}
Jonathan Baxter.
\newblock A model of inductive bias learning.
\newblock \emph{Journal of artificial intelligence research}, 12:\penalty0
  149--198, 2000.

\bibitem[Behl et~al.(2019)Behl, Baydin, and Torr]{alphaMAML}
Harkirat~Singh Behl, Atılım~Güneş Baydin, and Philip H.~S. Torr.
\newblock Alpha maml: Adaptive model-agnostic meta-learning, 2019.

\bibitem[Cao and Gu(2019)]{cao2019generalization}
Yuan Cao and Quanquan Gu.
\newblock Generalization bounds of stochastic gradient descent for wide and
  deep neural networks.
\newblock In \emph{Advances in Neural Information Processing Systems}, pages
  10835--10845, 2019.

\bibitem[Cao and Gu(2020)]{cao2020generalization}
Yuan Cao and Quanquan Gu.
\newblock Generalization error bounds of gradient descent for learning
  over-parameterized deep relu networks.
\newblock \emph{AAAI}, 2020.

\bibitem[Cesa-Bianchi et~al.(2004)Cesa-Bianchi, Conconi, and
  Gentile]{cesa2004generalization}
Nicolo Cesa-Bianchi, Alex Conconi, and Claudio Gentile.
\newblock On the generalization ability of on-line learning algorithms.
\newblock \emph{IEEE Transactions on Information Theory}, 50\penalty0
  (9):\penalty0 2050--2057, 2004.

\bibitem[Du et~al.(2019)Du, Lee, Li, Wang, and Zhai]{du2018gradient}
Simon~S. Du, Jason~D. Lee, Haochuan Li, Liwei Wang, and Xiyu Zhai.
\newblock Gradient descent finds global minima of deep neural networks.
\newblock \emph{International Conference on Machine Learning}, 2019.

\bibitem[Fallah et~al.(2020)Fallah, Mokhtari, and Ozdaglar]{maml_nonconvex}
Alireza Fallah, Aryan Mokhtari, and Asuman Ozdaglar.
\newblock On the convergence theory of gradient-based model-agnostic
  meta-learning algorithms.
\newblock In \emph{International Conference on Artificial Intelligence and
  Statistics}, pages 1082--1092, 2020.

\bibitem[Finn(2018)]{Finn:EECS-2018-105}
Chelsea Finn.
\newblock \emph{Learning to Learn with Gradients}.
\newblock PhD thesis, EECS Department, University of California, Berkeley, Aug
  2018.
\newblock URL
  \url{http://www2.eecs.berkeley.edu/Pubs/TechRpts/2018/EECS-2018-105.html}.

\bibitem[Finn et~al.(2017)Finn, Abbeel, and Levine]{maml}
Chelsea Finn, Pieter Abbeel, and Sergey Levine.
\newblock Model-agnostic meta-learning for fast adaptation of deep networks.
\newblock In \emph{Proceedings of the 34th International Conference on Machine
  Learning-Volume 70}, pages 1126--1135. JMLR. org, 2017.

\bibitem[Finn et~al.(2018)Finn, Xu, and Levine]{prob-maml}
Chelsea Finn, Kelvin Xu, and Sergey Levine.
\newblock Probabilistic model-agnostic meta-learning.
\newblock In \emph{Advances in Neural Information Processing Systems}, pages
  9516--9527, 2018.

\bibitem[Finn et~al.(2019)Finn, Rajeswaran, Kakade, and Levine]{finn2019online}
Chelsea Finn, Aravind Rajeswaran, Sham Kakade, and Sergey Levine.
\newblock Online meta-learning.
\newblock In \emph{International Conference on Machine Learning}, pages
  1920--1930, 2019.

\bibitem[Flennerhag et~al.(2020)Flennerhag, Rusu, Pascanu, Visin, Yin, and
  Hadsell]{WrapGrad}
Sebastian Flennerhag, Andrei~A. Rusu, Razvan Pascanu, Francesco Visin, Hujun
  Yin, and Raia Hadsell.
\newblock Meta-learning with warped gradient descent.
\newblock In \emph{International Conference on Learning Representations}, 2020.
\newblock URL \url{https://openreview.net/forum?id=rkeiQlBFPB}.

\bibitem[Franceschi et~al.(2018)Franceschi, Frasconi, Salzo, Grazzi, and
  Pontil]{franceschi2018bilevel}
Luca Franceschi, Paolo Frasconi, Saverio Salzo, Riccardo Grazzi, and
  Massimiliano Pontil.
\newblock Bilevel programming for hyperparameter optimization and
  meta-learning.
\newblock In \emph{International Conference on Machine Learning}, pages
  1568--1577, 2018.

\bibitem[Grefenstette et~al.(2019)Grefenstette, Amos, Yarats, Htut, Molchanov,
  Meier, Kiela, Cho, and Chintala]{higher}
Edward Grefenstette, Brandon Amos, Denis Yarats, Phu~Mon Htut, Artem Molchanov,
  Franziska Meier, Douwe Kiela, Kyunghyun Cho, and Soumith Chintala.
\newblock Generalized inner loop meta-learning.
\newblock \emph{arXiv preprint arXiv:1910.01727}, 2019.

\bibitem[Hospedales et~al.(2020)Hospedales, Antoniou, Micaelli, and
  Storkey]{hospedales2020metalearning}
Timothy Hospedales, Antreas Antoniou, Paul Micaelli, and Amos Storkey.
\newblock Meta-learning in neural networks: A survey, 2020.

\bibitem[Hu et~al.(2020)Hu, Zhang, Chen, and He]{hu2020biased}
Yifan Hu, Siqi Zhang, Xin Chen, and Niao He.
\newblock Biased stochastic gradient descent for conditional stochastic
  optimization.
\newblock \emph{arXiv preprint arXiv:2002.10790}, 2020.

\bibitem[Ioffe and Szegedy(2015)]{batchnorm}
Sergey Ioffe and Christian Szegedy.
\newblock Batch normalization: Accelerating deep network training by reducing
  internal covariate shift.
\newblock In \emph{International Conference on Machine Learning}, pages
  448--456, 2015.

\bibitem[Jacot et~al.(2018)Jacot, Gabriel, and Hongler]{ntk}
Arthur Jacot, Franck Gabriel, and Cl{\'e}ment Hongler.
\newblock Neural tangent kernel: Convergence and generalization in neural
  networks.
\newblock In \emph{Advances in neural information processing systems}, pages
  8571--8580, 2018.

\bibitem[Jacot et~al.(2020)Jacot, Gabriel, and Hongler]{hessian-ntk}
Arthur Jacot, Franck Gabriel, and Clement Hongler.
\newblock The asymptotic spectrum of the hessian of dnn throughout training.
\newblock In \emph{International Conference on Learning Representations}, 2020.
\newblock URL \url{https://openreview.net/forum?id=SkgscaNYPS}.

\bibitem[Ji et~al.(2020)Ji, Yang, and Liang]{ji2020multistep}
Kaiyi Ji, Junjie Yang, and Yingbin Liang.
\newblock Multi-step model-agnostic meta-learning: Convergence and improved
  algorithms.
\newblock \emph{arXiv preprint arXiv:2002.07836}, 2020.

\bibitem[Ji and Telgarsky(2019)]{ji2018gradient}
Ziwei Ji and Matus Telgarsky.
\newblock Gradient descent aligns the layers of deep linear networks.
\newblock In \emph{International Conference on Learning Representations}, 2019.
\newblock URL \url{https://openreview.net/forum?id=HJflg30qKX}.

\bibitem[K{\aa}gstr{\"o}m(1977)]{1977bounds}
Bo~K{\aa}gstr{\"o}m.
\newblock Bounds and perturbation bounds for the matrix exponential.
\newblock \emph{BIT Numerical Mathematics}, 17\penalty0 (1):\penalty0 39--57,
  1977.

\bibitem[Kawaguchi and Kaelbling(2019)]{kawaguchi2019elimination}
Kenji Kawaguchi and Leslie~Pack Kaelbling.
\newblock Elimination of all bad local minima in deep learning.
\newblock \emph{arXiv preprint arXiv:1901.00279}, 2019.

\bibitem[Khodak et~al.(2019)Khodak, Balcan, and Talwalkar]{adaptive-GBML}
Mikhail Khodak, Maria-Florina~F Balcan, and Ameet~S Talwalkar.
\newblock Adaptive gradient-based meta-learning methods.
\newblock In \emph{Advances in Neural Information Processing Systems}, pages
  5915--5926, 2019.

\bibitem[Kingma and Ba(2015)]{adam}
Diederik~P. Kingma and Jimmy Ba.
\newblock Adam: A method for stochastic optimization.
\newblock In \emph{ICLR}, 2015.

\bibitem[Lake et~al.(2015)Lake, Salakhutdinov, and Tenenbaum]{omniglot}
Brenden~M Lake, Ruslan Salakhutdinov, and Joshua~B Tenenbaum.
\newblock Human-level concept learning through probabilistic program induction.
\newblock \emph{Science}, 350\penalty0 (6266):\penalty0 1332--1338, 2015.

\bibitem[Lee et~al.(2019)Lee, Xiao, Schoenholz, Bahri, Sohl-Dickstein, and
  Pennington]{lee2019wide}
Jaehoon Lee, Lechao Xiao, Samuel~S Schoenholz, Yasaman Bahri, Jascha
  Sohl-Dickstein, and Jeffrey Pennington.
\newblock Wide neural networks of any depth evolve as linear models under
  gradient descent.
\newblock \emph{NeurIPS}, 2019.

\bibitem[Li et~al.(2018)Li, Ding, and Sun]{li2018over}
Dawei Li, Tian Ding, and Ruoyu Sun.
\newblock On the benefit of width for neural networks: Disappearance of bad
  basins.
\newblock \emph{arXiv preprint arXiv:1812.11039}, 2018.

\bibitem[Li and Liang(2018)]{li2018learning}
Yuanzhi Li and Yingyu Liang.
\newblock Learning overparameterized neural networks via stochastic gradient
  descent on structured data.
\newblock In \emph{Advances in Neural Information Processing Systems}, pages
  8157--8166, 2018.

\bibitem[Li et~al.(2019)Li, Wang, Yu, Du, Hu, Salakhutdinov, and Arora]{ECNTK}
Zhiyuan Li, Ruosong Wang, Dingli Yu, Simon~S Du, Wei Hu, Ruslan Salakhutdinov,
  and Sanjeev Arora.
\newblock Enhanced convolutional neural tangent kernels.
\newblock \emph{arXiv preprint arXiv:1911.00809}, 2019.

\bibitem[Liang et~al.(2018{\natexlab{a}})Liang, Sun, Lee, and
  Srikant]{liang2018adding}
Shiyu Liang, Ruoyu Sun, Jason~D Lee, and R~Srikant.
\newblock Adding one neuron can eliminate all bad local minima.
\newblock In \emph{Advances in Neural Information Processing Systems}, pages
  4355--4365, 2018{\natexlab{a}}.

\bibitem[Liang et~al.(2018{\natexlab{b}})Liang, Sun, Li, and
  Srikant]{liang2018understanding}
Shiyu Liang, Ruoyu Sun, Yixuan Li, and Rayadurgam Srikant.
\newblock Understanding the loss surface of neural networks for binary
  classification.
\newblock In \emph{International Conference on Machine Learning}, pages
  2835--2843, 2018{\natexlab{b}}.

\bibitem[Liang et~al.(2019)Liang, Sun, and Srikant]{liang2019revisiting}
Shiyu Liang, Ruoyu Sun, and R~Srikant.
\newblock Revisiting landscape analysis in deep neural networks: Eliminating
  decreasing paths to infinity.
\newblock \emph{arXiv preprint arXiv:1912.13472}, 2019.

\bibitem[Luo et~al.(2019)Luo, Ma, Ideker, Zhao, Su, and Liu]{luo2019mitigating}
Yunan Luo, Jianzhu Ma, Xiaoming~Trey Ideker, Jian Zhao, Peng1~YufengB Su, and
  Yang Liu.
\newblock Mitigating data scarcity in protein binding prediction using
  meta-learning.
\newblock In \emph{Research in Computational Molecular Biology: 23rd Annual
  International Conference, RECOMB 2019, Washington, DC, USA, May 5-8, 2019,
  Proceedings}, volume 11467, page 305. Springer, 2019.

\bibitem[Maurer et~al.(2016)Maurer, Pontil, and
  Romera-Paredes]{maurer2016benefit}
Andreas Maurer, Massimiliano Pontil, and Bernardino Romera-Paredes.
\newblock The benefit of multitask representation learning.
\newblock \emph{The Journal of Machine Learning Research}, 17\penalty0
  (1):\penalty0 2853--2884, 2016.

\bibitem[Meng and Zheng(2010)]{inverse-perturbation}
Lingsheng Meng and Bing Zheng.
\newblock The optimal perturbation bounds of the moore--penrose inverse under
  the frobenius norm.
\newblock \emph{Linear algebra and its applications}, 432\penalty0
  (4):\penalty0 956--963, 2010.

\bibitem[Muandet et~al.(2013)Muandet, Balduzzi, and
  Sch{\"o}lkopf]{muandet2013domain}
Krikamol Muandet, David Balduzzi, and Bernhard Sch{\"o}lkopf.
\newblock Domain generalization via invariant feature representation.
\newblock In \emph{International Conference on Machine Learning}, pages 10--18,
  2013.

\bibitem[Nguyen et~al.(2019)Nguyen, Mukkamala, and Hein]{nguyen2018loss}
Quynh Nguyen, Mahesh~Chandra Mukkamala, and Matthias Hein.
\newblock On the loss landscape of a class of deep neural networks with no bad
  local valleys.
\newblock In \emph{International Conference on Learning Representations}, 2019.

\bibitem[Nichol et~al.(2018)Nichol, Achiam, and Schulman]{reptile}
Alex Nichol, Joshua Achiam, and John Schulman.
\newblock On first-order meta-learning algorithms.
\newblock \emph{arXiv preprint arXiv:1803.02999}, 2018.

\bibitem[Park and Oliva(2019)]{meta-curvature}
Eunbyung Park and Junier~B Oliva.
\newblock Meta-curvature.
\newblock In \emph{Advances in Neural Information Processing Systems}, pages
  3309--3319. Curran Associates, Inc., 2019.
\newblock URL \url{http://papers.nips.cc/paper/8593-meta-curvature.pdf}.

\bibitem[Pedregosa et~al.(2011)Pedregosa, Varoquaux, Gramfort, Michel, Thirion,
  Grisel, Blondel, Prettenhofer, Weiss, Dubourg, Vanderplas, Passos,
  Cournapeau, Brucher, Perrot, and Duchesnay]{sklearn}
F.~Pedregosa, G.~Varoquaux, A.~Gramfort, V.~Michel, B.~Thirion, O.~Grisel,
  M.~Blondel, P.~Prettenhofer, R.~Weiss, V.~Dubourg, J.~Vanderplas, A.~Passos,
  D.~Cournapeau, M.~Brucher, M.~Perrot, and E.~Duchesnay.
\newblock Scikit-learn: Machine learning in {P}ython.
\newblock \emph{Journal of Machine Learning Research}, 12:\penalty0 2825--2830,
  2011.

\bibitem[Rajeswaran et~al.(2019)Rajeswaran, Finn, Kakade, and Levine]{imaml}
Aravind Rajeswaran, Chelsea Finn, Sham~M Kakade, and Sergey Levine.
\newblock Meta-learning with implicit gradients.
\newblock In \emph{Advances in Neural Information Processing Systems}, pages
  113--124, 2019.

\bibitem[Saunshi et~al.(2020)Saunshi, Zhang, Khodak, and
  Arora]{saunshi2020sample}
Nikunj Saunshi, Yi~Zhang, Mikhail Khodak, and Sanjeev Arora.
\newblock A sample complexity separation between non-convex and convex
  meta-learning, 2020.

\bibitem[Sun(2019)]{sun2019optimization}
Ruoyu Sun.
\newblock Optimization for deep learning: theory and algorithms.
\newblock \emph{arXiv preprint arXiv:1912.08957}, 2019.

\bibitem[Thrun and Pratt(1998)]{learningtolearn}
Sebastian Thrun and Lorien Pratt.
\newblock Learning to learn: Introduction and overview.
\newblock In \emph{Learning to learn}, pages 3--17. Springer, 1998.

\bibitem[Vanschoren(2018)]{vanschoren2018meta}
Joaquin Vanschoren.
\newblock Meta-learning: A survey.
\newblock \emph{arXiv preprint arXiv:1810.03548}, 2018.

\bibitem[Wang et~al.(2020)Wang, Cai, Yang, and Wang]{wang2020global}
Lingxiao Wang, Qi~Cai, Zhuoran Yang, and Zhaoran Wang.
\newblock On the global optimality of model-agnostic meta-learning.
\newblock In \emph{International Conference on Machine Learning}, pages
  101--110, 2020.

\bibitem[Wang et~al.(2019)Wang, Yao, Kwok, and Ni]{few-shot-survey}
Yaqing Wang, Quanming Yao, James~T Kwok, and Lionel~M Ni.
\newblock Generalizing from a few examples: A survey on few-shot learning.
\newblock \emph{ACM Computing Surveys (CSUR)}, 2019.

\bibitem[Xu et~al.(2020)Xu, Chen, and Karbasi]{xu2020meta}
Ruitu Xu, Lin Chen, and Amin Karbasi.
\newblock Meta learning in the continuous time limit.
\newblock \emph{arXiv preprint arXiv:2006.10921}, 2020.

\bibitem[Yu et~al.(2018)Yu, Guo, Yi, Chang, Potdar, Cheng, Tesauro, Wang, and
  Zhou]{yu-etal-2018-diverse}
Mo~Yu, Xiaoxiao Guo, Jinfeng Yi, Shiyu Chang, Saloni Potdar, Yu~Cheng, Gerald
  Tesauro, Haoyu Wang, and Bowen Zhou.
\newblock Diverse few-shot text classification with multiple metrics.
\newblock In \emph{Proceedings of the 2018 Conference of the North {A}merican
  Chapter of the Association for Computational Linguistics: Human Language
  Technologies, Volume 1 (Long Papers)}, pages 1206--1215, New Orleans,
  Louisiana, June 2018. Association for Computational Linguistics.
\newblock \doi{10.18653/v1/N18-1109}.
\newblock URL \url{https://www.aclweb.org/anthology/N18-1109}.

\bibitem[Zhou et~al.(2019)Zhou, Yuan, Xu, and Yan]{zhou2019metalearning}
Pan Zhou, Xiaotong Yuan, Huan Xu, and Shuicheng Yan.
\newblock Efficient meta learning via minibatch proximal update.
\newblock \emph{Neural Information Processing Systems}, 2019.

\end{thebibliography}

\onecolumn
\newpage
\appendix

\section{Neural Network Setup}
\label{supp:NTK-setup}

In this paper, we consider a fully-connected feed-forward network with $L$ hidden
layers. Each hidden layer has width $l_{i}$, for $i = 1, ..., L$. The readout layer (i.e. output layer) has width $l_{L+1} = k$. At each layer $i$, for arbitrary input $x\in \mathbb R^{d}$, we denote the pre-activation and post-activation functions by $h^i(x), z^i(x)\in\mathbb R^{l_i}$. The relations between layers in this network are
\begin{align}
\label{eq:recurrence}
\begin{cases}
    h^{i+1}&=z^i W^{i+1} + b^{i+1}
    \\
    z^{i+1}&=\activation \left(h^{i+1}\right) 
    \end{cases}
    \,\, \textrm{and} 
    \,\,
    \begin{cases}
  W^i_{\mu,\nu}& =  \omega_{\mu \nu}^i \sim \mathcal{N} (0, \frac {\sigma_\omega} {\sqrt{l_i}} )
    \\
    b_\nu^i &= \beta_\nu^i \sim \mathcal{N} (0,\sigma_b  )
\end{cases}
,
\end{align}
where $W^{i+1}\in \mathbb R^{l_i\times l_{i+1}}$ and $b^{i+1}\in\mathbb R^{l_{i+1}}$ are the weight and bias of the layer, $\omega_{\mu \nu}^l$ and $ b_\nu^l $ are trainable variables drawn i.i.d. from zero-mean Gaussian distributions at initialization (i.e.,
$\frac{\sws}{l_i}$ and $\sbs$ are variances for weight and bias, and $\activation$ is a point-wise activation function.

\section{Related Works}\label{supp:related-works}

\subsection{Related Works on Supervised Learning}
\label{supp:related-works:supervised-learning}

Recently, there is a line of works studying the optimization of neural networks in the setting of supervised learning (e.g. \cite{AllenZhu2018ACT,du2018gradient,ntk,CNTK,lee2019wide,ji2018gradient,fine-grained,
liang2018adding,liang2018understanding, kawaguchi2019elimination,
li2018over,nguyen2018loss,li2018learning,liang2019revisiting}), 
and it has been proved that the optimization is guaranteed to converge to global optima;
see \citet{sun2019optimization} for an overview. 
Below, we briefly discuss one such convergence result.
 
 Consider a supervised learning task:
 given the training samples $\sX=(\sx_i)_{i=1}^n$ and targets $\sY=(\sy_i)_{i=1}^n$,
 we learn the neural network $f$ by minimizing
\begin{align}\label{eq:supervised}
     L(\theta) =  \sum_{i=1}^{m}\ell(f_\theta(\sx_i),\sy_i)=\frac{1}{2}\|f_\theta(\sX)-\sY\|^2_2,
\end{align}
where $\ell(\hat {\sy}, \sy)=\frac{1}{2}\|\hat{\sy} - \sy\|^2_2$ is the loss function, and $f_\theta(\sX)\equiv(f_\theta(\sx_i))_{i=1}^n\in \mathbb R^{kn}$ is the concatenated network outputs on all samples. 

The update rule of gradient descent on $L$ is standard:
\begin{align}
    \theta_{t+1} = \theta_t - \lambda \nabla_{\theta_t} L(\theta_t)\nonumber
\end{align}
where $t$ represents the number of training steps.

With sufficiently small learning rate $\lambda$ and sufficiently large network width, \cite{ntk,lee2019wide,CNTK} show that the gradient descent is guaranteed to obtain zero training loss given long enough training time, i.e., $\lim_{t \rightarrow \infty}L(\theta_t) = 0$. Currently, the global convergence of supervised learning with deep neural networks is mostly restricted to $\ell_2$ loss. Other loss functions are generally harder to analyze. For example, with the cross-entropy loss function, the model parameters do not converge to a point without regularization. This is the reason why we also consider $\ell_2$ loss.

\subsection{Related Works on Meta-Learning Theory}
\label{supp:related-works:meta-learning}

Except for the papers listed in Sec. \ref{sec:intro}, there are two recent meta-learning theory works that are more relevant to the topic of this paper, which we discuss in detail below.

In the concurrent work\footnote{The release of \citet{wang2020global} is \textit{before} our first submission of this paper to another conference.} of \citet{wang2020global}, the authors also study the optimization of MAML in a different approach from our paper and proves its global convergence for a special class of neural nets. Specifically, they only consider a class of \textit{two-layer} neural nets with second layers frozen (i.e., \textit{not trained} during the GBML optimization). Besides, their neural nets do not have bias terms, and the second layers are assigned with \textit{binary values} $\{+1,-1\}$ only, which is not a practical nor common assumption. In contrast, in this paper, we consider neural nets of \textit{any depth} that have bias terms, and all layers are initialized with values drawn i.i.d. from Gaussian distributions, which is a commonly used initialization scheme. Also, no layer is assumed to be frozen in this paper. Overall, we think the neural nets studied in \citet{wang2020global} are too simple, special, and unrealistic. In our opinion, our work provides a global convergence guarantee to GBML with a much broader class of neural nets, which are more general and common in practice.

\citet{saunshi2020sample} studies the generalization ability of Reptile \cite{reptile}, a first-order variant of MAML, in a toy problem setting with 2-layer \textit{linear} neural nets (i.e., neural nets without non-linear activation). Specifically, they construct a meta-learning problem with only \textit{two unique training/test tasks}, and they assume labels are generated by a \textit{linear function}, which is a very strong data assumption. Notably, \textit{linear} neural nets are just toy models for certain theoretical analyses, and they are never used in practice, since their representation power is the same as linear models. Overall, \citet{saunshi2020sample} does not consider common \textit{non-linear} neural nets and general problem settings, thus its results are not much meaningful to general GBML.

\section{Proof of Global Convergence for Gradient-Based Meta-Learning with Deep Neural Networks}
\label{supp:global-convergence}
In this section, we will prove Theorem \ref{thm:global-convergence}, which states that with sufficiently over-parameterized neural networks, gradient-based meta-learning trained under gradient descent is guaranteed to converge to global optima at linear convergence rate. 

We consider the standard parameterization scheme of neural networks shown in (\ref{eq:recurrence}).

Theorem \ref{thm:global-convergence} depends on several assumptions and lemmas. The assumptions are listed below. After that, we present the lemmas and the global convergence theorem, with proofs in Appendix \ref{supp:global-convergence:lemma-proof},\ref{supp:lemma-proof:bounded_init_loss},\ref{supp:global-convergence:lemma-proof:kernel-convergence} and \ref{supp:global-convergence:theorem-proof}. For Corollary \ref{corr:GBML-output}, we append its proof to Appendix \ref{supp:GBML-output}.

\begin{assumption}[Bounded Input Norm]\label{assum:input-norm<=1}
$\forall \sX\in \X$, for any sample $\sx \in \sX$, $\|\sx\|_2 \leq 1$. Similarly, $\forall \sX' \in \X'$, for any sample $\sx' \in \sX'$, $\|\sx'\|\leq 1$. (This is equivalent to a input normalization operation, which is common in data preprocessing.)
\end{assumption}

\begin{assumption}[Non-Degeneracy]\label{assum:non-degeneracy}
The meta-training set $(\X,\Y)$ and the meta-test set $(\X',\Y')$ are both contained in some compact set. Also, $\X$ and $\X'$ are both non-degenerate, i.e. $\forall \sX, \widetilde{\sX} \in \X$, $\sX\neq \widetilde{\sX}$, and $\forall \sX', \widetilde{\sX}' \in \X'$, $\sX'\neq \widetilde{\sX}'$.
\end{assumption}

\begin{assumption}[Same Width for Hidden Layers]\label{assum:same-width} 
All hidden layers share the same width, $l$, i.e., $l_1=l_2=\dots=l_L = l$.

\end{assumption}

\begin{assumption}[Activation Function]\label{assum:activation}
The activation function used in neural networks, $\phi$, has the following properties: 
    \begin{align}
    \label{eq:activation-assumption}
    |\phi(0)| < \infty,\quad  \|\phi'\|_\infty < \infty, \quad \sup_{x\neq \tilde x} |\phi'(x) - \phi'(\tilde x)|/|x-\tilde x| < \infty.
    \end{align}
\end{assumption}

\begin{assumption}[Full-Rank]\label{assum:full-rank} The kernel $\metaNTK$ defined in Lemma \ref{lemma:kernel-convergence} is full-rank.
\end{assumption}

These assumptions are common, and one can find similar counterparts of them in the literature for supervised learning \cite{lee2019wide,CNTK}.

As defined in the main text, $\theta$ is used to represent the neural net parameters. For convenience, we define some short-hand notations:
\begin{align}
    f_t(\cdot) &= f_{\theta_t}(\cdot)\\
    F_t(\cdot) &= F_{\theta_t}(\cdot)\\
    f(\theta) &= f_{\theta}(\X) = ((f_{\theta}(\sX_i))_{i=1}^{N} \\
    F(\theta) &= F_{\theta}(\XXY) =((F_{\theta}(\xxyi))_{i=1}^{N}\\
    g(\theta) &=  F_{\theta}(\XXY) - \Y \\
    J(\theta) &= \nabla_{\theta} F(\theta) = \nabla_\theta F_\theta(\XXY)
\end{align}
and
\begin{align}
    \mc L (\theta_t)&=\ell(F(\theta_t),\Y)=\frac 1 2 \|g(\theta_t)\|_2^2 \label{eq:supp:train_loss}\\
    \metantk_t &= \frac{1}{l} \nabla F_{\theta_t}(\XXY) \nabla F_{\theta_t}(\XXY) = \frac{1}{l} J(\theta) J(\theta)^\top \label{eq:supp:emp_MNK}
\end{align}
where we use the $\ell_2$ loss function $\ell (\hat{y},y)=\frac 1 2 \|\hat y - y\|^2_2$ in the definition of training loss $\loss(\theta_t)$ in (\ref{eq:supp:train_loss}), and the $\metantk_t$ in (\ref{eq:supp:emp_MNK}) is based on the definition\footnote{There is a typo in the definition of $\metantk_{\theta}(\cdot,\star)$ in Sec. \ref{sec:global-convergence}: a missing factor $\frac{1}{l}$. The correct definition should be $\metantk_{\theta}(\cdot,\star)= \frac{1}{l}\nabla_\theta F_\theta(\cdot)\nabla_\theta F_\theta (\star)^\top$. Similarly, the definition of $\metaNTK$ in Theorem \ref{thm:global-convergence} also missis this factor: the correct version is $\metaNTK = \frac{1}{l}\lim_{l \rightarrow \infty} J(\theta_0) J(\theta_0)^\top$} of $\metantk_\theta(\cdot,\star)$ in Sec. \ref{sec:global-convergence}.

Below, Lemma \ref{lemma:local-Liphschitzness} proves the Jacobian $J$ is locally Lipschitz, Lemma \ref{lemma:bounded_init_loss} proves the training loss at initialization is bounded, and Lemma \ref{lemma:kernel-convergence} proves $\metantk_0$ converges in probability to a deterministic kernel matrix with bounded positive eigenvalues. Finally, with these lemmas, we can prove the global convergence of GBML in Theorem \ref{thm:global-convergence:restated}.

\begin{lemma}[{\bf Local Lipschitzness of Jacobian} (Lemma \ref{lemma:local-Liphschitzness} restated)]
\label{lemma:supp:local-Liphschitzness} 
Suppose\footnote{This assumption is realistic in practice. For example, the official implementation of MAML \cite{maml} for few-shot classification benchmarks adopts (i) $\tau=1, \lambda = 0.4$ and (ii) $\tau=5,\lambda=0.1$, which both satisfy our assumption.} $\tau = \cO(\frac{1}{\lambda})$, then there exists $K>0$ and $l^*>0$ such that: $\forall ~C>0$ and $l > l^*$, the following inequalities hold true with high probability over random initialization,
\begin{align}\label{eq:jacobian-lip:supp}
 \forall \theta, \, \Bar \theta \in B(\theta_0, C l^{-\frac 1 2}),  \begin{cases}  
    \frac 1 {\sqrt l}\|J(\theta) - J(\Bar \theta)\|_{F} &\leq K\|\theta - \Bar \theta\|_2
    \\
    \\
    \frac 1 {\sqrt l} \|J(\theta)\|_{F} & \leq K 
    \end{cases}
\end{align}
where $B$ is a neighborhood defined as
\begin{align}
\label{def:B:supp}
    B(\theta_0, R) := \{\theta: \|\theta-\theta_0\|_2 < R\}.   
\end{align}
\end{lemma}
 \begin{proof}
 See Appendix \ref{supp:global-convergence:lemma-proof}.
 \end{proof}

 \begin{lemma}[\textbf{Bounded Initial Loss}] \label{lemma:bounded_init_loss}
For arbitrarily small $\delta_0 > 0$, there are constants $R_0>0$ and $l^*>0$ such that as long as the width $l > l^*$, with probability at least $(1-\delta_0)$ over random initialization,
\begin{align}
    \|g(\theta_0)\|_2 &= \|F_{\theta_0}(\XXY) - \Y\|_2\leq R_0,
\end{align}
which is also equivalent to
\begin{align*}
    \loss(\theta_0) = \frac 1 2 \|g(\theta_0)\|^2_2 \leq \frac 1 2 R_0^2.
\end{align*}
\end{lemma}
\begin{proof}
See Appendix \ref{supp:lemma-proof:bounded_init_loss}.
\end{proof}
\begin{lemma}[\textbf{Kernel Convergence}]\label{lemma:kernel-convergence}
Suppose the learning rates $\eta$ and $\lambda$ suffiently small. As the network width $l$ approaches infinity, $\metantk_0 = J(\theta_0) J(\theta_0)^\top$ converges in probability to a deterministic kernel matrix $\metaNTK$ (i.e., $\metaNTK=\lim_{l \rightarrow \infty} \metantk_0$), which is independent of $\theta_0$ and can be analytically calculated. Furthermore, the eigenvalues of $\metaNTK$ is bounded as, $0 < \lev (\metaNTK) \leq \Lev (\metaNTK) < \infty$.
\end{lemma}
\begin{proof}
See Appendix \ref{supp:global-convergence:lemma-proof:kernel-convergence}.
\end{proof}

Note the update rule of gradient descent on $\theta_t$ with learning rate $\eta$ can be expressed as
\begin{align}\label{eq:gd&jacobian}
    \theta_{t+1} = \theta_t - \eta J(\theta_t)^{\top}g(\theta_t).
\end{align}
The following theorem proves the global convergence of GBML under the update rule of gradient descent.
\begin{theorem}[\textbf{Global Convergence} \textit{(Theorem \ref{thm:global-convergence} restated)}]\label{thm:global-convergence:restated}
Suppose $\tau = \cO(\frac{1}{\lambda})$. Denote $\lev=\lev(\metaNTK)$ and $\Lev=\Lev(\metaNTK)$. For any $\delta_0 > 0$ and $\eta_0 < \frac{2}{\Lev + \lev}$,
there exist $\lss>0$, $\Lambda \in\mathbb N$, $K>1$, and $\lambda_0>0$, such that: for width $l\geq \Lambda$, running gradient descent with learning rates $\eta = \frac {\eta_0}{l}$ and $\lambda < \frac{\lambda_0}{l}$ over random initialization, the following inequalities hold true with probability at least $(1 - \delta_0)$:
\begin{align}
     \sum_{j=1}^{t}\|\theta_j - \theta_{j-1}\|_2  &\leq  \frac {3K \lss}{\lev}  l^{-\frac 1 2} \label{eq:convergence-parameters:supp}\\
    \sup_{t} \|  \metantk_0 - \metantk_t\|_F  &\leq \frac {6K^3\lss}{\mins}  l^{-\frac 1 2}
    \label{eq:convergence-metantk:supp}
\end{align}
and
\begin{align}
g(\theta_t) = \|F(\theta_t) - \Y \|_2 \leq  \left(1 - \frac {\eta_0 \mins}{3}\right)^t \lss\,,
\end{align}
which leads to
\begin{align}
\loss (\theta_t) = \frac{1}{2}\|F(\theta_t) - \Y \|_2^2 \leq  \left(1 - \frac {\eta_0 \mins}{3}\right)^{2t} \frac{\lss^2}{2}\,,
    \label{eq:convergence-loss:supp}
\end{align}
indicating the training loss converges to zero at a linear rate.
\end{theorem}
\begin{proof}
See Appendix \ref{supp:global-convergence:theorem-proof}.
\end{proof}

In the results of Theorem \ref{thm:global-convergence:restated} above, (\ref{eq:convergence-parameters:supp}) considers the optimization trajectory of network parameters, and show the parameters move locally during training. (\ref{eq:convergence-metantk:supp}) indicates the kernel matrix $\metantk_t$ changes slowly. Finally, (\ref{eq:convergence-loss:supp}) demonstrates that the training loss of GBML decays exponentially to zero as the training time evolves, indicating convergence to global optima at a linear rate.

\subsection{Helper Lemmas}
\label{supp:global-convergence:lemma-proof}

\begin{lemma}[]\label{lemma:helper-zero-F-norm}
As the width $l\rightarrow \infty$, for any vector $\mathbf{a} \in \bR^{m\times 1}$ that $\|\mathbf{a}\|_F \leq C$ with some constant $C>0$, we have 
\begin{align}\label{eq:NTK-grad-vec-Frob}
\|\nabla_\theta \ntk_\theta(x,X') \cdot \mathbf{a}\|_F \rightarrow 0
\end{align}
where $\theta$ is randomly intialized parameters.
\end{lemma}
\begin{proof}
Notice that
\begin{align}
\ntk_\theta(x,X') = \frac{1}{l}\underbrace{\nabla_\theta f_\theta(\overbrace{x}^{\in\bR^{ d}})}_{\in \bR^{ 1\times D}}\cdot \underbrace{\nabla_\theta f_\theta(\overbrace{X'}^{\in\bR^{m\times d}})^\top}_{\in \bR^{D \times m}} \in \bR^{ 1\times m}
\end{align}
with gradient as
\begin{align}\label{eq:helper-zero-F-norm:NTK-grad}
\nabla_\theta \ntk_\theta(x,X') &= \frac{1}{l}\underbrace{\nabla^2_\theta f_\theta(x)}_{\in \bR^{ 1 \times D \times D}} \cdot\underbrace{\nabla_\theta f_\theta(X')^\top}_{\in \bR^{D\times m}} + \frac{1}{l}\underbrace{ \nabla_\theta f_\theta(x)}_{\in \bR^{1\times D}} \cdot \underbrace{\nabla^2_\theta f_\theta(X')^\top}_{\in \bR^{D\times m \times D}} \in \bR^{1 \times m \times D}
\end{align}
where we apply a \textit{dot product} in the \textit{first two dimensions} of 3-tensors and matrices to obtain matrices.

Then, it is obvious that our goal is to bound the Frobenius Norm of
\begin{align}\label{eq:helper-zero-F-norm:NTK-grad:tensor-contraction}
    \nabla_\theta \ntk_\theta(x,X') \cdot \mathbf{a} = \left(\frac{1}{l}\nabla^2_\theta f_\theta(x) \cdot\nabla_\theta f_\theta(X')^\top\right) \cdot \mathbf{a} + \left(\frac{1}{l}\nabla_\theta f_\theta(x) \cdot\nabla^2_\theta f_\theta(X')^\top\right) \cdot \mathbf{a}
\end{align}

Below, we prove that as the width $l\rightarrow \infty$, the first and second terms of \eqref{eq:helper-zero-F-norm:NTK-grad:tensor-contraction} both have vanishing Frobenius norms, which finally leads to the proof of \eqref{eq:NTK-grad-vec-Frob}.
\begin{itemize}
    \item \textit{First Term of (\ref{eq:helper-zero-F-norm:NTK-grad:tensor-contraction})}. Obviously, reshaping $\nabla^2_\theta f_\theta(x) \in \bR^{1\times D\times D}$ as a $\bR^{D \times D}$ matrix does not change the Frobenius norm $\|\frac{1}{l}\nabla^2_\theta f_\theta(x) \cdot\underbrace{\nabla_\theta f_\theta(X')^\top}_{\in \bR^{D\times m}}\|_F$ (in other words, $\|\frac{1}{l}\underbrace{\nabla^2_\theta f_\theta(x)}_{\in \bR^{ {1 \times D \times D}}} \cdot\underbrace{\nabla_\theta f_\theta(X')^\top}_{\in \bR^{D\times m}}\|_F = \|\frac{1}{l}\underbrace{\nabla^2_\theta f_\theta(x)}_{\in \bR^{D \times D}} \cdot\underbrace{\nabla_\theta f_\theta(X')^\top}_{\in \bR^{D\times m}}\|_F$). 
    
    By combining the following three facts,
    \begin{enumerate}
        \item $\|\frac{1}{\sqrt{l}}\underbrace{\nabla^2_\theta f_\theta(x)}_{\in \bR^{D \times D}} \|_{op} \rightarrow 0 $ indicated by \citet{hessian-ntk},
        \item the matrix algebraic fact $\|HB\|_F \leq \|H\|_{op} \|B\|_F$,
        \item the bound $\|\frac{1}{\sqrt{l}} \nabla_\theta f_\theta(\cdot)\|_F < constant$ from \cite{lee2019wide},
    \end{enumerate}
    one can easily show that the first term of (\ref{eq:helper-zero-F-norm:NTK-grad}) has vanishing Frobenius norm, i.e., 
    \begin{align}\label{eq:NTK-grad:first-term-Frob}
    \|\frac{1}{l}\nabla^2_\theta f_\theta(x) \cdot\nabla_\theta f_\theta(X')^\top\|_F \rightarrow 0
    \end{align}
    Then, obviously,
    \begin{align}\label{eq:NTK-grad:first-term-vec-Frob}
        \|\left(\frac{1}{l}\nabla^2_\theta f_\theta(x) \cdot\nabla_\theta f_\theta(X')^\top\right) \cdot \mathbf{a}\|_F \leq \|\frac{1}{l}\nabla^2_\theta f_\theta(x) \cdot\nabla_\theta f_\theta(X')^\top\|_F \|\mathbf{a}\|_F \rightarrow 0
    \end{align}
    \item \textit{Second Term of (\ref{eq:helper-zero-F-norm:NTK-grad:tensor-contraction}).} From \citet{hessian-ntk}, we know that
    \begin{align}\label{eq:NTK-grad:Hessian-vec-op-norm}
    \|\underbrace{\frac{1}{\sqrt{l}} \nabla^2_\theta f_\theta(X')^\top}_{\in \bR^{D \times m \times D}}\cdot \underbrace{\mathbf{a}}_{\in \bR^{m \times 1}} \|_{op} \rightarrow 0
    \end{align}
    Then, similar to the derivation of (\ref{eq:NTK-grad:first-term-Frob}), we have
    \begin{align}\label{eq:NTK-grad:second-term-vec-Frob}
        \|\left(\frac{1}{l}\nabla_\theta f_\theta(x) \cdot\nabla^2_\theta f_\theta(X')^\top\right) \cdot \mathbf{a}\|_F \leq \overbrace{\|\frac{1}{\sqrt{l}}\nabla_\theta f_\theta(x)\|_F}^{\leq constant} \cdot\overbrace{\|\nabla^2_\theta f_\theta(X')^\top\cdot\mathbf{a}\|_{op}}^{\rightarrow 0} \rightarrow 0
    \end{align}
    
    \item Finally, combining (\ref{eq:NTK-grad:first-term-vec-Frob}) and (\ref{eq:NTK-grad:second-term-vec-Frob}), we obtain (\ref{eq:NTK-grad-vec-Frob}) by
    \begin{align}
        \|\nabla_\theta \ntk_\theta(x,X') \cdot \mathbf{a}\|_F&\leq \|\left(\frac{1}{l}\nabla^2_\theta f_\theta(x) \cdot\nabla_\theta f_\theta(X')^\top\right) \cdot \mathbf{a}\|_F \nonumber\\ 
        &~+\|\left(\frac{1}{l}\nabla_\theta f_\theta(x) \cdot\nabla^2_\theta f_\theta(X')^\top\right) \cdot \mathbf{a}\|_F\nonumber\\
        &\rightarrow 0
    \end{align}
\end{itemize}
\end{proof}

\begin{lemma}
\label{lemma:helper-local-Liphschitzness}
Given any task $\task = (\xyxy)$ and randomly initialized parameters $\theta$, as the width $l \rightarrow \infty$, for any $x \in X$, where $x\in \bR^{d}$ and $X\in \bR^{n\times d}$, we have
\begin{align}\label{eq:lemma:helper-local-Liphschitzness}
    \|\nabla_\theta \left(\ntk_{\theta}(x,X')\ntk_{\theta}^{-1} (I - e^{-\lambda \ntk_{\theta} \tau})  \right) (f_\theta(\sX') - Y')\|_F \rightarrow 0~,
\end{align}
and furthermore,
\begin{align}\label{eq:lemma:helper-local-Liphschitzness:stack-of-x}
     \|\nabla_\theta \left(\ntk_{\theta}(X,X')\ntk_{\theta}^{-1} (I - e^{-\lambda \ntk_{\theta} \tau})  \right)  (f_\theta(\sX') - Y')\|_F \rightarrow 0~.
\end{align}
\end{lemma}

\begin{proof}[Proof of Lemma \ref{lemma:helper-local-Liphschitzness}] 


~\\
\textbf{Overview.} In this proof, we consider the expression
\begin{align}
    &\quad \nabla_\theta \left(\ntk_{\theta}(x,X')\ntk_{\theta}^{-1} (I - e^{-\lambda \ntk_{\theta} \tau})  \right)  (f_\theta(\sX') - Y') \label{eq:lemma:helper-local-Liphschitzness:grad-all}\\
    &= ~~~\nabla_\theta \left(\ntk_{\theta}(x,X')\right)\ntk_{\theta}^{-1} (I - e^{-\lambda \ntk_{\theta} \tau})  (f_\theta(\sX') - Y') \label{eq:lemma:helper-local-Liphschitzness:grad-1}\\
    &\quad +\ntk_{\theta}(x,X')\left(\nabla_\theta \ntk_{\theta}^{-1}\right) (I - e^{-\lambda \ntk_{\theta} \tau})  (f_\theta(\sX') - Y')\label{eq:lemma:helper-local-Liphschitzness:grad-2}\\
    &\quad +\ntk_{\theta}(x,X') \ntk_{\theta}^{-1} \left(\nabla_\theta (I - e^{-\lambda \ntk_{\theta} \tau}) \right)  (f_\theta(\sX') - Y')\label{eq:lemma:helper-local-Liphschitzness:grad-3},
\end{align}
and we prove the terms of (\ref{eq:lemma:helper-local-Liphschitzness:grad-1}), (\ref{eq:lemma:helper-local-Liphschitzness:grad-2}) and (\ref{eq:lemma:helper-local-Liphschitzness:grad-3}) all have vanishing Frobenius norm. Thus, (\ref{eq:lemma:helper-local-Liphschitzness:grad-all}) also has vanishing Frobenius norm in the infinite width limit, which is exactly the statement of (\ref{eq:lemma:helper-local-Liphschitzness}). This indicates that \eqref{eq:lemma:helper-local-Liphschitzness:stack-of-x} also has a vanishing Frobenius norm, since $\ntk_\theta(X,X')$ can be seen as a stack of $n$ copies of $\ntk_\theta(x,X')$, where $n$ is a finite constant.

\paragraph{Step I.}Each factor of \eqref{eq:lemma:helper-local-Liphschitzness:grad-all} has bounded Frobenius norm.
\begin{itemize}
    \item $\|\ntk_{\theta}(x,X')\|_F$. It has been shown that $\|\frac{1}{\sqrt l} \nabla_\theta f(\cdot)\|_F \leq constant$ in \citet{lee2019wide}, thus we have $\|\ntk_{\theta}(x,X')\|_F = \| \frac{1}{l} \nabla_\theta f(x) \nabla_\theta f(X')^\top\|_F\leq \|\frac{1}{\sqrt l} \nabla_\theta f(x)\|_F\|\frac{1}{\sqrt l} \nabla_\theta f(X')\|_F\leq constant$.
    \item $\|\ntk_{\theta}^{-1}\|_F$. It has been shown that $\ntk_\theta$ is positive definite with positive least eigenvalue \cite{ntk,CNTK}, thus $\|\ntk_\theta^{-1}\|_F \leq constant$.
    \item $\|I - e^{-\lambda \ntk_{\theta} \tau}\|_F$. \citet{cao2019generalization} shows that largest eigenvalues of $\ntk_\theta$ are of $O(L)$, and we know $\ntk_\theta$ is positive definite \cite{ntk,CNTK}, thus it is obvious the eigenvalues of $I - e^{-\lambda \ntk_{\theta} \tau}$ fall in the set $\{z~|~0<z<1\}$. Therefore, certainly we have $\|I - e^{-\lambda \ntk_{\theta} \tau}\|_F \leq constant$.
    \item $\| f_\theta(\sX') - Y'\|_F$. \citet{lee2019wide} shows that $\| f_\theta(\sX') - Y'\|_2 \leq constant$, which indicates that $\| f_\theta(\sX') - Y'\|_F \leq constant$.
\end{itemize}
In conclusion, we have shown
\begin{align}\label{eq:lemma:helper-local-Liphschitzness:factors-bounded-frob}
\|\ntk_{\theta}(x,X')\|_F,\|\ntk_{\theta}^{-1}\|_F,\|I - e^{-\lambda \ntk_{\theta} \tau}\|_F, \|f_\theta(\sX')-Y'\|_F \leq constant
\end{align}

\paragraph{Step II.} Bound (\ref{eq:lemma:helper-local-Liphschitzness:grad-1}).

 Without loss of generality, let us consider the neural net output dimension $k=1$ in this proof, i.e., $f_\theta: \mathbb{R}^d \mapsto \mathbb{R}$. (Note: with $k>1$, the only difference is that $\nabla_{\theta} f(X') \in \bR^{mk\times D}$, which has no impact on the proof). Then, we have
\begin{align}
\ntk_\theta(x,X') = \frac{1}{l}\underbrace{\nabla_\theta f_\theta(\overbrace{x}^{\in\bR^{ d}})}_{\in \bR^{ 1\times D}}\cdot \underbrace{\nabla_\theta f_\theta(\overbrace{X'}^{\in\bR^{m\times d}})^\top}_{\in \bR^{D \times m}} \in \bR^{ 1\times m}
\end{align}
with gradient as
\begin{align}\label{eq:NTK-grad}
\nabla_\theta \ntk_\theta(x,X') &= \frac{1}{l}\underbrace{\nabla^2_\theta f_\theta(x)}_{\in \bR^{ 1 \times D \times D}} \cdot\underbrace{\nabla_\theta f_\theta(X')^\top}_{\in \bR^{D\times m}} + \frac{1}{l}\underbrace{ \nabla_\theta f_\theta(x)}_{\in \bR^{1\times D}} \cdot \underbrace{\nabla^2_\theta f_\theta(X')^\top}_{\in \bR^{D\times m \times D}} \in \bR^{1 \times m \times D}
\end{align}
where we apply a \textit{dot product} in the \textit{first two dimensions} of 3-tensors and matrices to obtain matrices.

Based on (\ref{eq:lemma:helper-local-Liphschitzness:factors-bounded-frob}), we know that $$\|\overbrace{\ntk_{\theta}^{-1} (I - e^{-\lambda \ntk_{\theta} \tau})  (f_\theta(\sX') - Y')}^{\in \bR^{m \times 1}}\|_F\leq \|\ntk_{\theta}^{-1} \|_F \|I - e^{-\lambda \ntk_{\theta} \tau}\|_F  \|f_\theta(\sX') - Y'\|_F \leq constant~.$$ Then, applying (\ref{eq:NTK-grad-vec-Frob}), we have
\begin{align}\label{eq:lemma:helper-local-Liphschitzness:term-1}
\|\nabla_\theta \left(\ntk_{\theta}(x,X')\right)\cdot \ntk_{\theta}^{-1} (I - e^{-\lambda \ntk_{\theta} \tau})  (f_\theta(\sX') - Y') \|_F \rightarrow 0
\end{align}

\paragraph{Step III.} Bound (\ref{eq:lemma:helper-local-Liphschitzness:grad-2}) and (\ref{eq:lemma:helper-local-Liphschitzness:grad-3})

\begin{itemize}
    \item Bound (\ref{eq:lemma:helper-local-Liphschitzness:grad-2}): $\ntk_{\theta}(x,X')\left(\nabla_\theta \ntk_{\theta}^{-1}\right) (I - e^{-\lambda \ntk_{\theta} \tau}) (f_\theta(\sX') - Y')$.
    
    Clearly, $\underbrace{\nabla_\theta \ntk_\theta^{-1}}_{m \times m \times D} = -\underbrace{\ntk_\theta^{-1}}_{\in \bR^{m \times m}} \cdot \underbrace{(\nabla_\theta \ntk_\theta)}_{\in \bR^{m \times m \times D}} \cdot \underbrace{\ntk_\theta^{-1}}_{\bR^{m \times m}}$, where we apply a dot product in the first two dimensions of the 3-tensor and matrices. 
    
    Note that $\nabla_\theta \ntk_\theta = \sqrt{1}{l} \nabla_\theta^2 f_\theta (X') \cdot \nabla_\theta f_\theta (X')^\top+ \sqrt{1}{l}\nabla_\theta f_\theta (X') \cdot \nabla_\theta^2 f_\theta (X')^\top$. Obviously, by (\ref{eq:NTK-grad:Hessian-vec-op-norm}) and (\ref{eq:lemma:helper-local-Liphschitzness:factors-bounded-frob}), we can easily prove that 
    \begin{align}\label{eq:lemma:helper-local-Liphschitzness:term-2}
        \|\ntk_{\theta}(x,X')\left(\nabla_\theta \ntk_{\theta}^{-1}\right) (I - e^{-\lambda \ntk_{\theta} \tau}) (f_\theta(\sX') - Y')\|_F \rightarrow 0
    \end{align}
    \item Bound (\ref{eq:lemma:helper-local-Liphschitzness:grad-3}): $\ntk_{\theta}(x,X') \ntk_{\theta}^{-1} \left(\nabla_\theta (I - e^{-\lambda \ntk_{\theta} \tau}) \right) (f_\theta(\sX') - Y')$
    
    Since $\underbrace{\nabla_\theta (I - e^{-\lambda \ntk_\theta \tau })}_{\in \bR^{m \times m \times D}} = \lambda \tau \cdot \underbrace{e^{-\lambda \ntk_\theta \tau}}_{\in \bR^{m \times m}}\cdot \underbrace{\nabla_\theta \ntk_\theta}_{\in \bR^{m \times m \times D}}$, we can easily obtain the following result by (\ref{eq:NTK-grad:Hessian-vec-op-norm}) and (\ref{eq:lemma:helper-local-Liphschitzness:factors-bounded-frob}),
    \begin{align}\label{eq:lemma:helper-local-Liphschitzness:term-3}
         \|\ntk_{\theta}(x,X') \ntk_{\theta}^{-1} \left(\nabla_\theta (I - e^{-\lambda \ntk_{\theta} \tau}) \right) (f_\theta(\sX') - Y')\|_F \rightarrow 0
    \end{align}

\end{itemize}
\paragraph{Step IV.}Final result: prove (\ref{eq:lemma:helper-local-Liphschitzness:grad-all}) and \eqref{eq:lemma:helper-local-Liphschitzness:stack-of-x}.

Combining (\ref{eq:lemma:helper-local-Liphschitzness:term-1}), (\ref{eq:lemma:helper-local-Liphschitzness:term-2}) and (\ref{eq:lemma:helper-local-Liphschitzness:term-3}), we can prove \eqref{eq:lemma:helper-local-Liphschitzness:grad-all}
\begin{align}
\|\nabla_\theta \left(\ntk_{\theta}(x,X')\ntk_{\theta}^{-1} (I - e^{-\lambda \ntk_{\theta} \tau})  \right) (f_\theta(\sX') - Y')|_F &\leq \|\nabla_\theta \left(\ntk_{\theta}(x,X')\right)\ntk_{\theta}^{-1} (I - e^{-\lambda \ntk_{\theta} \tau}) (f_\theta(\sX') - Y') \|_F \nonumber\\
&+  \|\ntk_{\theta}(x,X')\left(\nabla_\theta \ntk_{\theta}^{-1}\right) (I - e^{-\lambda \ntk_{\theta} \tau})(f_\theta(\sX') - Y')\|_F \nonumber\\
&+\|\ntk_{\theta}(x,X') \ntk_{\theta}^{-1} \left(\nabla_\theta (I - e^{-\lambda \ntk_{\theta} \tau}) \right) (f_\theta(\sX') - Y')\|_F\nonumber \\
&\rightarrow 0 
\end{align}

Then, since $\ntk_\theta(X,X')$ can be seen as a stack of $n$ copies of $\ntk_\theta(x,X')$, where $n$ is a finite constant, we can easily prove \eqref{eq:lemma:helper-local-Liphschitzness:stack-of-x} by
\begin{align}
    \|\nabla_\theta \left(\ntk_{\theta}(X,X')\ntk_{\theta}^{-1} (I - e^{-\lambda \ntk_{\theta} \tau})  \right) (f_\theta(\sX') - Y')\|_F &\leq \sum_{i\in [n]} \|\nabla_\theta \left(\ntk_{\theta}(x_i,X')\ntk_{\theta}^{-1} (I - e^{-\lambda \ntk_{\theta} \tau})  \right) (f_\theta(\sX') - Y')\|_F \nonumber\\
    &\rightarrow 0
\end{align}
where we denote $X = (x_i)_{i=1}^n$.

\end{proof}
\subsection{Proof of Lemma \ref{lemma:supp:local-Liphschitzness}}
\begin{proof}[Proof of Lemma \ref{lemma:supp:local-Liphschitzness}]
Consider an arbitrary task $\task = (\xyxy)$. Given sufficiently large width $l$, for any parameters in the neighborhood of the initialization, i.e., $\theta \in \B(\theta_0, C l^{-1/2})$, based on \cite{lee2019wide}, we know the meta-output can be decomposed into a terms of $f_\theta$,
\begin{align}\label{eq:lemma:jacobian:0}
    F_\theta(\xxy) = f_\theta(X) - \ntk_{\theta}(X,X')\ntk_{\theta}^{-1} (I - e^{-\lambda \ntk_{\theta} \tau})(f_\theta(X')-Y'),
\end{align}
where $\ntk_\theta(X,X') = \frac{1}{l}\nabla_\theta f_\theta(X) \nabla_\theta f_\theta(X')^\top$, and $\ntk_{\theta}\equiv \ntk_\theta(X',X')$ for convenience.

Then, we consider $\nabla_\theta F_\theta(\xxy)$, the gradient of $F_\theta(\xxy)$ in (\ref{eq:lemma:jacobian:0}),
\begin{align}\label{eq:lemma:jacobian:full-F-grad}
    \nabla_{\theta} F_\theta(\xxy)
    &= \nabla_{\theta}f_\theta(\sX) - \ntk_{\theta}(X,X')\ntk_{\theta}^{-1} (I - e^{-\lambda \ntk_{\theta} \tau}) \nabla_{\theta}f_\theta(\sX') \nonumber\\
    &\quad - \nabla_\theta \left(\ntk_{\theta}(X,X')\ntk_{\theta}^{-1} (I - e^{-\lambda \ntk_{\theta} \tau})  \right) (f_\theta(\sX') - Y')
\end{align}

By Lemma \ref{lemma:helper-local-Liphschitzness}, we know the last term of \eqref{eq:lemma:jacobian:full-F-grad} has a vanishing Frobenius norm as the width increases to infinity. Thus, for sufficiently large width $l$ (i.e., $l > l^*$), we can drop the last term of \eqref{eq:lemma:jacobian:full-F-grad}, resulting in
\begin{align}\label{eq:lemma:jacobian:F-grad}
    \nabla_{\theta} F_\theta(\xxy)
    = \nabla_{\theta}f_\theta(\sX) - \ntk_{\theta}(X,X')\ntk_{\theta}^{-1} (I - e^{-\lambda \ntk_{\theta} \tau}) \nabla_{\theta}f_\theta(\sX').
\end{align}

Now, let us consider the SVD decomposition on $\frac{1}{\sqrt l}\nabla_\theta f_\theta(X') \in \mathbb R^{km \times D}$, where $X' \in \mathbb{R}^{k \times m}$ and $\theta \in \mathbb{R}^D$. such that $\frac{1}{\sqrt l}\nabla_\theta f_\theta(X') = U \Sigma V^\top$, where $U\in \mathbb R^{km \times km},V\in \mathbb R^{D \times km}$ are orthonormal matrices while $\Sigma \in \mathbb R^{km \times km}$ is a diagonal matrix. Note that we take $km \leq D$ here since the width is sufficiently wide.

Then, since $\ntk_\theta = \frac{1}{l}\nabla_\theta f_\theta(X') \nabla_\theta f_\theta(X')^\top= U \Sigma V^\top V \Sigma U^\top = U \Sigma^2 U^\top$, we have $\ntk^{-1}_\theta = U \Sigma^{-2} U^\top$. Also, by Taylor expansion, we have 
\begin{align}\label{eq:lemma:jacobian:taylor}
I - e^{-\lambda\ntk_\theta \tau} = I - \sum_{i=0}^\infty \frac{(-\lambda \tau)^i}{i!} (\ntk_\theta)^i = U\left( I - \sum_{i=0}^\infty \frac{(-\lambda \tau)^i}{i!} (\Sigma)^i \right)U^\top = U\left(I - e^{-\lambda \Sigma \tau}\right) U^\top.
\end{align}

With these results of SVD, (\ref{eq:lemma:jacobian:F-grad}) becomes
\begin{align}
    \nabla_{\theta} F((\xxy),\theta) &= \nabla_{\theta}f_\theta(\sX) - \frac{1}{l}\nabla_{\theta} f_\theta(X) \nabla_{\theta} f_\theta(X')^\top \ntk_{\theta}^{-1} (I - e^{-\lambda \ntk_{\theta} \tau}) \nabla_{\theta}f_\theta(\sX')\nonumber\\
    &=\nabla_{\theta}f_\theta(\sX) - \frac{1}{l}\nabla_{\theta} f_\theta(X) (\sqrt{l}V \Sigma U^\top) (U \Sigma^{-2} U^\top) [ U\left(I - e^{-\lambda \Sigma \tau}\right) U^\top] (\sqrt{l}U \Sigma V^\top) \nonumber\\
    &=\nabla_{\theta}f_\theta(\sX) - \nabla_\theta f_\theta(X) V \Sigma^{-1} \left(I - e^{-\lambda \Sigma \tau}\right) \Sigma V^\top\nonumber\\
    &=\nabla_{\theta}f_\theta(\sX) - \nabla_\theta f_\theta(X) V \left(I - e^{-\lambda \Sigma \tau}\right)  V^\top \nonumber\\
    &=\nabla_{\theta}f_\theta(\sX) - \nabla_\theta f_\theta(X) (I - e^{-\lambda \conjntk_{\theta} \tau})\nonumber \\
    &=\nabla_\theta f_\theta(X) e^{-\lambda \conjntk_{\theta} \tau} \label{eq:lemma:jacobian:F-grad:-2}
\end{align}
where $\conjntk_\theta\equiv \conjntk_\theta(X',X')=\frac{1}{l} \nabla_\theta f_\theta(X')^\top \nabla_\theta f_\theta(X') \in \mathbb R^{D \times D}$, and the step (\ref{eq:lemma:jacobian:F-grad:-2}) can be easily obtained by a Taylor expansion similar to (\ref{eq:lemma:jacobian:taylor}).

Note that $\conjntk_\theta$ is a product of $\nabla_\theta f_\theta(X')^\top$ and its transpose, hence it is positive semi-definite, and so does $e^{-\lambda \conjntk \tau}$. By eigen-decomposition on $\conjntk$, we can easily see that the eigenvalues of $e^{-\lambda \conjntk \tau}$ are all in the range $[0,1)$ for arbitrary $\tau>0$. Therefore, it is easy to get that for arbitrary $\tau > 0$,
\begin{align}\label{eq:lemma:jacobian:F-grad:norm:0}
    \|\nabla_{\theta} F((\xxy),\theta)\|_F = \|\nabla_\theta f_\theta(X) e^{-\lambda \conjntk_{\theta} \tau}\|_F \leq \|\nabla_\theta f_\theta(X)\|_F 
\end{align}
By Lemma 1 of \cite{lee2019wide}, we know that there exists a $K_0>0$ such that for any $X$ and $\theta$,
\begin{align}\label{eq:lemma:jacobian:1}
    \|\frac{1}{\sqrt l} \nabla f_\theta(X) \|_F \leq K_0.
\end{align}
Combining (\ref{eq:lemma:jacobian:F-grad:norm:0}) and (\ref{eq:lemma:jacobian:1}), we have
\begin{align}
    \|\frac{1}{\sqrt l}\nabla_{\theta} F((\xxy),\theta)\|_F \leq \|\frac{1}{\sqrt l}\nabla_{\theta} f_\theta(X)\|_F \leq K_0,
\end{align}
which is equivalent to 
\begin{align}
    \frac{1}{\sqrt l}\|J(\theta)\|_F \leq K_0
\end{align}

Now, let us study the other term of interest, $\|J(\theta)-J(\Bar \theta)\|_F=\|\frac{1}{\sqrt l}\nabla_{\theta} F((\xxy),\theta)-\frac{1}{\sqrt l}\nabla_{\theta} F((\xxy),\Bar \theta)\|_F$, where $\theta,\Bar \theta \in \B(\theta_0, C l^{-1/2})$. 

To bound $\|J(\theta)-J(\Bar \theta)\|_F$, let us consider
\begin{align}
    &\quad ~\|\nabla_{\theta} F((\xxy),\theta)-\nabla_{\theta} F((\xxy),\Bar \theta)\|_{op} \label{eq:lemma:jacobian:grad_diff:expression}\\
    &= \|\nabla_\theta f_\theta(X) e^{-\lambda \conjntk_{\theta} \tau} - \nabla_{\Bar \theta} f_{\Bar \theta}(X) e^{-\lambda \conjntk_{\Bar \theta} \tau}\|_{op} \nonumber\\
    &=\frac{1}{2}\|\left(\nabla_\theta f_\theta(X) - \nabla_{\Bar \theta} f_{\Bar \theta}(X)\right)\left( e^{-\lambda \conjntk_{\theta} \tau} + e^{-\lambda \conjntk_{\Bar \theta} \tau}\right)\\
    &\quad +\left(\nabla_\theta f_\theta(X) + \nabla_{\Bar \theta} f_{\Bar \theta}(X)\right)\left( e^{-\lambda \conjntk_{\theta} \tau} - e^{-\lambda \conjntk_{\Bar \theta} \tau}\right)\|_{op} \nonumber\\
    &\leq \frac{1}{2} \|\nabla_\theta f_\theta(X) - \nabla_{\Bar \theta} f_{\Bar \theta}(X)\|_{op}\| e^{-\lambda \conjntk_{\theta} \tau} + e^{-\lambda \conjntk_{\Bar \theta} \tau}\|_{op}\\
    &\quad +\frac{1}{2}\|\nabla_\theta f_\theta(X) + \nabla_{\Bar \theta} f_{\Bar \theta}(X)\|_{op}\| e^{-\lambda \conjntk_{\theta} \tau} - e^{-\lambda \conjntk_{\Bar \theta} \tau}\|_{op}\nonumber \\
    &\leq \frac{1}{2} \|\nabla_\theta f_\theta(X) - \nabla_{\Bar \theta} f_{\Bar \theta}(X)\|_{op} \left(\| e^{-\lambda \conjntk_{\theta} \tau}\|_{op} + \|e^{-\lambda \conjntk_{\Bar \theta} \tau}\|_{op}\right) \label{eq:lemma:jacobian:grad_diff:0}\\
    &\quad +\frac{1}{2}\left(\|\nabla_\theta f_\theta(X)\|_{op} + \|\nabla_{\Bar \theta} f_{\Bar \theta}(X)\|_{op}\right)\| e^{-\lambda \conjntk_{\theta} \tau} - e^{-\lambda \conjntk_{\Bar \theta} \tau}\|_{op}  \label{eq:lemma:jacobian:grad_diff:1}
\end{align}
It is obvious that $\| e^{-\lambda \conjntk_{\theta} \tau}\|_{op},\|e^{-\lambda \conjntk_{\Bar \theta} \tau}\|_{op} \leq 1$. Also, by the relation between the operator norm and the Frobenius norm, we have
\begin{align}\label{eq:lemma:jacobian:1.1}
    \|\nabla_\theta f_\theta(X) - \nabla_{\Bar \theta} f_{\Bar \theta}(X)\|_{op} \leq \|\nabla_\theta f_\theta(X) - \nabla_{\Bar \theta} f_{\Bar \theta}(X)\|_{F} 
\end{align}
Besides, Lemma 1 of \cite{lee2019wide} indicates that there exists a $K_1 > 0$ such that for any $X$ and $\theta,\Bar \theta \in \B(\theta_0, C l^{-1/2})$,
\begin{align}\label{eq:lee:jocobian:1}
    \|\frac{1}{\sqrt l} \nabla_\theta f_\theta(X) - \frac{1}{\sqrt l}\nabla_\theta f_{\Bar \theta}(X)\|_F \leq K_1 \|\theta - \Bar \theta\|_2
\end{align}
Therefore, (\ref{eq:lemma:jacobian:1.1}) gives
\begin{align}
    \|\nabla_\theta f_\theta(X) - \nabla_{\Bar \theta} f_{\Bar \theta}(X)\|_{op}  \leq K_1 \sqrt{l}\|\theta - \Bar \theta\|_2
\end{align}
and then (\ref{eq:lemma:jacobian:grad_diff:0}) is bounded as
\begin{align}\label{eq:lemma:jacobian:grad_diff:0:bound}
     \frac{1}{2} \|\nabla_\theta f_\theta(X) - \nabla_{\Bar \theta} f_{\Bar \theta}(X)\|_{op} \left(\| e^{-\lambda \conjntk_{\theta} \tau}\|_{op} + \|e^{-\lambda \conjntk_{\Bar \theta} \tau}\|_{op}\right) \leq K_1 \sqrt{l} \|\theta - \Bar \theta\|_2.
\end{align}
As for (\ref{eq:lemma:jacobian:grad_diff:1}), notice that $\|\cdot\|_{op} \leq \|\cdot\|_F$ and (\ref{eq:lemma:jacobian:1}) give us
\begin{align}\label{eq:lemma:jacobian:grad_diff:1:1}
    \frac{1}{2}\left(\|\nabla_\theta f_\theta(X)\|_{op} + \|\nabla_{\Bar \theta} f_{\Bar \theta}(X)\|_{op}\right) \leq \sqrt{l} K_0.
\end{align}

Then, to bound $\| e^{-\lambda \conjntk_{\theta} \tau} - e^{-\lambda \conjntk_{\Bar \theta} \tau}\|_{op}$ in (\ref{eq:lemma:jacobian:grad_diff:1}), let us bound the following first
\begin{align}
    \|\conjntk_\theta -\conjntk_{\Bar \theta}\|_{F} &= \|\frac{1}{l} \nabla_\theta f_\theta(X')^\top \nabla_\theta f_\theta(X') - \frac{1}{l} \nabla_{\Bar \theta} f_{\Bar \theta}(X')^\top \nabla_{\Bar \theta} f_{\Bar \theta}(X')\|_{F} \nonumber\\
    &= \frac{1}{l} \|\frac 1 2( \nabla_\theta f_\theta(X')^\top + \nabla_{\Bar \theta} f_{\Bar \theta}(X')^\top)(\nabla_\theta f_\theta(X') -\nabla_{\Bar \theta} f_{\Bar \theta}(X')) \nonumber\\
    &+ \frac 1 2( \nabla_\theta f_\theta(X')^\top -  \nabla_{\Bar \theta} f_{\Bar \theta}(X')^\top)(\nabla_\theta f_\theta(X')+\nabla_{\Bar \theta} f_{\Bar \theta}(X'))\|_F\nonumber\\
    &\leq \frac 1 l \|\nabla_\theta f_\theta(X') + \nabla_{\Bar \theta} f_{\Bar \theta}(X')\|_F\|\nabla_\theta f_\theta(X') -  \nabla_{\Bar \theta} f_{\Bar \theta}(X')\|_F\nonumber\\
    &\leq \frac 1 l \left(\|\nabla_\theta f_\theta(X')\|_F +\| \nabla_{\Bar \theta} f_{\Bar \theta}(X')_F\|\right) \|\nabla_\theta f_\theta(X') -  \nabla_{\Bar \theta} f_{\Bar \theta}(X')\|_F \nonumber\\
    &\leq 2 K_0 K_1 \|\theta - \Bar \theta\|_2
\end{align}
Then, with the results above and a perturbation bound on matrix exponentials from \cite{1977bounds}, we have
\begin{align}
    \| e^{-\lambda \conjntk_{\theta} \tau} - e^{-\lambda \conjntk_{\Bar \theta} \tau}\|_{op} 
    &\leq \lambda \tau \|\conjntk_\theta -\conjntk_{\Bar \theta}\|_{op} \cdot \exp(-\lambda \tau \cdot \min\{\lev(\conjntk_\theta), \lev(\conjntk_{\Bar \theta})\}) \nonumber\\
    &\leq \lambda \tau \|\conjntk_\theta -\conjntk_{\Bar \theta}\|_{op} \nonumber\\
    &\leq \lambda \tau \|\conjntk_\theta -\conjntk_{\Bar \theta}\|_{F} \nonumber\\
    &\leq 2 K_0 K_1 \lambda \tau \|\theta - \Bar \theta\|_2 \label{eq:lemma:jacobian:grad_diff:1:2}
\end{align}

Hence, by (\ref{eq:lemma:jacobian:grad_diff:1:1}) and (\ref{eq:lemma:jacobian:grad_diff:1:2}), we can bound (\ref{eq:lemma:jacobian:grad_diff:1}) as
\begin{align}\label{eq:lemma:jacobian:grad_diff:1:bound}
\frac{1}{2}\left(\|\nabla_\theta f_\theta(X)\|_{op} + \|\nabla_{\Bar \theta} f_{\Bar \theta}(X)\|_{op}\right)\| e^{-\lambda \conjntk_{\theta} \tau} - e^{-\lambda \conjntk_{\Bar \theta} \tau}\|_{op} \leq 
2\sqrt{l} K_0^2 K_1 \lambda \tau \|\theta - \Bar \theta\|_2 
\end{align}

Finally, with (\ref{eq:lemma:jacobian:grad_diff:0:bound}) and (\ref{eq:lemma:jacobian:grad_diff:1:bound}), we can bound (\ref{eq:lemma:jacobian:grad_diff:expression}) as
\begin{align*}
    \|\nabla_{\theta} F((\xxy),\theta)-\nabla_{\theta} F((\xxy),\Bar \theta)\|_{op} \leq (K_1 + 2 K_0^2 K_1\lambda \tau ) \sqrt{l} \|\theta - \Bar \theta\|_2 
\end{align*}

Finally, combining these bounds on (\ref{eq:lemma:jacobian:grad_diff:0}) and (\ref{eq:lemma:jacobian:grad_diff:1}), we know that
\begin{align}
    \|J(\theta)-J(\Bar \theta)\|_F&=\|\frac{1}{\sqrt l}\nabla_{\theta} F((\xxy),\theta)-\frac{1}{\sqrt l}\nabla_{\theta} F((\xxy),\Bar \theta)\|_F\nonumber\\
    &\leq \frac{\sqrt{kn}}{\sqrt l} \|\nabla_{\theta} F((\xxy),\theta)-\nabla_{\theta} F((\xxy),\Bar \theta)\|_{op}\nonumber\\
    &\leq \sqrt{kn} (K_1+ 2K_0^2 K_1\lambda \tau) \|\theta - \Bar \theta\|_2\label{eq:lemma:jacobian:grad_diff:final}
\end{align}

Define $K_2 = \sqrt{kn} (K_1+ 2K_0^2 K_1 \lambda \tau) $, we have 
\begin{align}
    \|J(\theta)-J(\Bar \theta)\|_F \leq K_2 \|\theta - \Bar \theta\|_2
\end{align}

Note that since $\tau = \cO(\frac{1}{\lambda})$, we have $\lambda \tau = \cO(1)$, indicating the factor $\lambda \tau$ is neglectable compared with other factors in $K_2$. Hence, the various choices of $\tau$ under $\tau = \cO(\frac{1}{\lambda})$ do not affect this proof.

Taking $K=\max \{K_0,K_2\}$ completes the proof.
\end{proof}

\subsection{Proof of Lemma \ref{lemma:bounded_init_loss}}\label{supp:lemma-proof:bounded_init_loss}
\begin{proof}[Proof of Lemma \ref{lemma:bounded_init_loss}]
It is known that $f_{\theta_0}(\cdot)$ converges in distribution to a mean zero Gaussian with the covariance $\mc K$ determined by the parameter initialization \cite{lee2019wide}. As a result, for arbitrary $\delta_1 \in (0,1)$ there exist constants $l_1 > 0$ and $R_1 > 0$, such that: $\forall~ l \geq l_1$, over random initialization, the following inequality holds true with probability at least $(1-\delta_1)$,
\begin{align}\label{eq:convergence:init-error-bounded}
    \|f_{\theta_0}(X)-Y\|_2, \|f_{\theta_0}(X')-Y'\|_2 \leq R_1
\end{align}

From (\ref{eq:meta-output}), we know that $\forall (\xyxy) =  \task \in \D$,
$$F_{\theta_0}(\xxy) = f_{\theta_0'}(X)$$
where $\theta_0'$ is the parameters after $\tau$-step update on $\theta_0$ over the meta-test task $(\sX',\sY')$:
\begin{align}
    &\theta_\tau=\theta', ~~\theta_0=\theta, \nonumber\\
    &\theta_{i+1} = \theta_{i} - \lambda \nabla_{\theta_{i}} \ell(f_{\theta_{i}}(\sX'),\sY') ~~ \forall i=0,...,\tau-1,
\end{align}
Suppose the learning rate $\lambda$ is sufficiently small, then based on Sec. (\ref{sec:MetaNTK}), we have
\begin{align}
F_{\theta_0}(\xxy) = f_{\theta_0}(X) + \ntk_0(X,X') \ntk_0^{-1} (I-e^{-\lambda \ntk_0 \tau})(f_{\theta_0}(X')-Y').
\end{align}
where $\ntk_0(\cdot,\star) = \nabla_{\theta_0} f_{\theta_0}(\cdot) \nabla_{\theta_0} f_{\theta_0}(\star)^\top$ and we use a shorthand $\ntk_0 \equiv \ntk_0(X',X') $.

\cite{ntk} proves that for sufficiently large width, $\ntk_0$ is positive definite and converges to $\NTK$, the Neural Tangent Kernel, a full-rank kernel matrix with bounded positive eigenvalues. Let $\lev(\NTK)$ and $\Lev(\NTK)$ denote the least and largest eigenvalue of $\NTK$, respectively. Then, it is obvious that for a sufficiently over-parameterized neural network, the operator norm of $\ntk(X,X') \ntk^{-1} (I-e^{-\lambda \ntk \tau})$ can be bounded based on $\lev(\NTK)$ and $\Lev(\NTK)$. Besides, \cite{CNTK,lee2019wide} demonstrate that the neural net output at initialization, $f_{\theta_0}(\cdot)$, is a zero-mean Gaussian with small-scale covaraince. Combining these results and (\ref{eq:convergence:init-error-bounded}), we know there exists $R(R_1,N,\lev(\NTK),\Lev(\NTK))$ such that 
\begin{align}
\|F_{\theta_0}(\xxy)-Y\|_2\leq R(R_1,N,\lev(\NTK),\Lev(\NTK))
\end{align}
By taking an supremum over $R(R_1,N,\lev,\Lev)$ for each training task in $\{\task_i=(\xyxyi)\}_{i\in[N]}$, we can get $R_2$ such that $\forall i \in [N]$
\begin{align}
\|F_{\theta_0}(\xxyi)-Y_i\|_2\leq R_2
\end{align}
and for $R_0 = \sqrt{N} R_2$, define $\delta_0$ as some appropriate scaling of $\delta_1$, then the following holds true with probability $(1-\delta_0)$ over random initialization,
\begin{align}
\|g(\theta_0)\|_2 &=\sqrt{\sum_{\xyxy \in \D}\|F((\xxy),\theta_0)-y\|_2^2} \leq R_0
\end{align}
\end{proof}
\subsection{Proof of Lemma \ref{lemma:kernel-convergence}}\label{supp:global-convergence:kernel-convergence}
\begin{proof}[Proof of Lemma \ref{lemma:kernel-convergence}]
The learning rate for meta-adaption, $\lambda$, is sufficiently small, so (\ref{eq:meta-adaption-descrete}) becomes \textit{continuous-time} gradient descent. Based on \cite{lee2019wide}, for any task $\task=(\xyxy)$,
\begin{align}\label{eq:lemma-proof:kernel-convergence:F_0}
 F_0(\xxy) =f_0(\sX) + \ntk_0(\xx)\widetilde{\T}_{\ntk_0}^\lambda(\sX',\tau)\left(\sY' - f_0(\sX')\right),
\end{align}
where $\ntk_0(\cdot,\star)= \frac{1}{l}\nabla_{\theta_0} f_0(\cdot) \nabla_{\theta_0} f_0(\star)^\top$, and $\widetilde{T}^{\lambda}_{\ntk_0} (\cdot,\tau) \coloneqq \ntk_0(\cdot,\cdot)^{-1}(I-e^{-\lambda \ntk_0(\cdot,\cdot) \tau})$.

Then, we consider $\nabla_{\theta_0} F_{0}(\xxy)$, the gradient of $F_{0}(\xxy)$ in (\ref{eq:lemma-proof:kernel-convergence:F_0}). By Lemma \ref{lemma:helper-local-Liphschitzness}, we know that for sufficiently wide networks, the gradient of $F_{0}(\xxy)$ becomes 
\begin{align}\label{eq:MetaNTK:F_t-grad}
    \nabla_{\theta_0} F_0(\xxy)= \nabla_{\theta_0}f_0(\sX) - \ntk_0(\xx)\T_{\ntk_0}^\lambda(\sX',\tau) \nabla_{\theta_0}f_0(\sX')
\end{align}
Since $\metantk_0 \equiv \metantk_0((\XXY),(\XXY)) = \frac 1 l \nabla_{\theta_0} F_0(\XXY) \nabla_{\theta_0} F_0(\XXY)^\top$ and $F_0(\XXY) = (F_0(\xxyi))_{i=1}^N \in \mathbb{R}^{knN}$, we know $\metantk_0$ is a block matrix with $N\times N$ blocks of size $kn \times kn$. For $i,j \in [N]$, the $(i,j)$-th block can be denoted as $[\metantk_0]_{ij}$ such that
\begin{align}
    [\metantk_0]_{ij} &= \frac{1}{l} \nabla_{\theta_0} F_0(\xxyi) \nabla_{\theta_0} F_0(\xxyj)^\top \nonumber\\
    &= \quad \frac{1}{l}\nabla_{\theta_0}f_0(X_i) \nabla_{\theta_0}f_0(X_j)^\top \nonumber\\
    &\quad + \frac{1}{l} \ntk_0(\xxi)\widetilde{\T}_{\ntk_0}^\lambda(X_i',\tau) \nabla_{\theta_0}f_0(X_i')  \nabla_{\theta_0}f_0(X_j')^\top \widetilde{\T}_{\ntk_0}^\lambda(X_j',\tau)^\top \ntk_0(X_j',X_j)\nonumber\\
    &\quad - \frac{1}{l} \nabla_{\theta_0}f_0(X_i)  \nabla_{\theta_0}f_0(X_j')^\top \widetilde{\T}_{\ntk_0}^\lambda(X_j',\tau)^\top \ntk_0(X_j',X_j) \nonumber\\
    &\quad - \frac{1}{l} \ntk_0(\xxi)\widetilde{\T}_{\ntk_0}^\lambda(X_i',\tau) \nabla_{\theta_0}f_0(X_i')\nabla_{\theta_0}f_0(X_j)^\top \nonumber\\
    &= \quad \ntk_0 (X_i,X_j) \nonumber\\
    &\quad + \ntk_0(\xxi)\widetilde{\T}_{\ntk_0}^\lambda(X_i',\tau) \ntk_0(X_i',X_j')  \widetilde{\T}_{\ntk_0}^\lambda(X_j',\tau)^\top \ntk_0(X_j',X_j)\nonumber\\
    &\quad -  \ntk_0(X_i,X_j') \widetilde{\T}_{\ntk_0}^\lambda(X_j',\tau)^\top \ntk_0(X_j',X_j) \nonumber\\
    &\quad -\ntk_0(\xxi)\widetilde{\T}_{\ntk_0}^\lambda(X_i',\tau) \ntk_0(X_i',X_j) 
\end{align}
where we used the equivalences $\ntk_0(\cdot, \star) = \ntk_0(\star,\cdot)^\top$ and $\frac{1}{l} \nabla_{\theta_0} f_0(\cdot) \nabla_{\theta_0} f_0(\star) = \ntk_0(\cdot, \star)$.

By Algebraic Limit Theorem for Functional Limits, we have
\begin{align}
     &\quad \lim_{l\rightarrow\infty}[\metantk_0]_{ij} \nonumber\\
     &= \lim_{l\rightarrow\infty} \ntk_0 (X_i,X_j) \nonumber\\
    &\quad + \lim_{l\rightarrow\infty}\ntk_0(\xxi)\T_{\lim_{l\rightarrow\infty}\ntk_0}^\lambda(X_i',\tau) \lim_{l\rightarrow\infty} \ntk_0(X_i',X_j')  \T_{\lim_{l\rightarrow\infty} \ntk_0}^\lambda(X_j',\tau)^\top \lim_{l\rightarrow\infty} \ntk_0(X_j',X_j)\nonumber\\
    &\quad -  \lim_{l\rightarrow\infty} \ntk_0(X_i,X_j') \T_{\lim_{l\rightarrow\infty}\ntk_0}^\lambda(X_j',\tau)^\top \lim_{l\rightarrow\infty}\ntk_0(X_j',X_j) \nonumber\\
    &\quad -\lim_{l\rightarrow\infty}\ntk_0(\xxi)\T_{\lim_{l\rightarrow\infty}\ntk_0}^\lambda(X_i',\tau) \ntk_0(X_i',X_j) \nonumber\\
        &= \quad \NTK (X_i,X_j) \nonumber\\
    &\quad + \NTK(\xxi)\widetilde{\T}_{\NTK}^\lambda(X_i',\tau) \NTK(X_i',X_j')  \widetilde{\T}_{\NTK}^\lambda(X_j',\tau)^\top \NTK(X_j',X_j)\nonumber\\
    &\quad -  \NTK(X_i,X_j') \widetilde{\T}_{\NTK}^\lambda(X_j',\tau)^\top \NTK(X_j',X_j) \nonumber\\
    &\quad -\NTK(\xxi)\widetilde{\T}_{\NTK}^\lambda(X_i',\tau) \NTK(X_i',X_j) \label{eq:lemma-proof:kernel-convergence:metantk_ij}
\end{align}
where $\NTK(\cdot,\star) = \lim_{l\rightarrow\infty}\ntk_0 (\cdot,\star)$ is a deterministic kernel function, the Neural Tangent Kernel function (NTK) from the literature on supervised learning \cite{ntk,lee2019wide,CNTK}. Specifically, $\ntk_0 (\cdot,\star)$ converges to $\NTK(\cdot,\star)$ in probability as the width $l$ approaches infinity.

Hence, for any $i,j \in [N]$, as the width $l$ approaches infinity, $[\metantk_0]_{ij}$ converges in probability to a deterministic matrix $\lim_{l\rightarrow\infty} [\metantk_0]_{ij}$, as shown by (\ref{eq:lemma-proof:kernel-convergence:metantk_ij}). Thus, the whole block matrix $\metantk_0$ converges in probability to a deterministic matrix in the infinite width limit. Denote $\metaNTK =\lim_{l \rightarrow \infty} \metantk_0$, then we know $\metaNTK$ is a deterministic matrix with each block expressed as (\ref{eq:lemma-proof:kernel-convergence:metantk_ij}).

Since $\metantk_0 \equiv \metantk_0((\XXY),(\XXY)) = \frac 1 l \nabla_{\theta_0} F_0(\XXY) \nabla_{\theta_0} F_0(\XXY)^\top$, it is a symmetric square matrix. Hence all eigenvalues of $\metantk_0$ are greater or equal to $0$, which also holds true for $\metaNTK$. In addition, because of Assumption \ref{assum:full-rank}, $\metaNTK$ is positive definite, indicating $\lev(\metaNTK) > 0$. On the other hand, from \cite{CNTK}, we know diagonal entries and eigenvalues of $\NTK(\cdot,\star)$ are positive real numbers upper bounded by $\cO(L)$, as a direct result, it is easy to verify that the diagonal entries of the matrix $\metaNTK$ are also upper bounded, indicating $\Lev(\metaNTK) < \infty$. Hence, we have $0 < \lev(\metaNTK) < \Lev (\metaNTK) < \infty$.

\textbf{Extension.} It is easy to extend (\ref{eq:lemma-proof:kernel-convergence:metantk_ij}), the expression for $\metaNTK\equiv \lim_{l \rightarrow \infty}\metantk_0((\XXY),(\XXY)$, to more general cases. Specifically, we can express $\metaNTK(\cdot,\star)$ analytically for arbitrary inputs. To achieve this, let us define a kernel function,  $\SingleTaskmetaNTK: (\mathbb{R}^{n \times k} \times \mathbb{R}^{m\times k}) \times  (\mathbb{R}^{n \times k} \times \mathbb{R}^{m\times k}) \mapsto \mathbb{R}^{nk \times nk}$ such that
\begin{align}
    \SingleTaskmetaNTK((\cdot,\ast), (\bullet, \star)) &= \NTK(\cdot,\bullet) + \NTK(\cdot,\ast)\widetilde{\T}_{\NTK}^\lambda(\ast,\tau)\NTK(\ast,\star)\widetilde{\T}_{\NTK}^\lambda(\star,\tau)^\top \NTK(\star,\bullet) \nonumber\\
 &\quad -\NTK(\cdot,\ast)\widetilde{\T}_{\NTK}^\lambda(\ast,\tau) \NTK(\ast,\bullet) - \NTK(\cdot,\star) \widetilde{\T}_{\NTK}^\lambda(\star,\tau)^\top \NTK(\star,\bullet) .
\end{align}
Then, it is obvious that for $i,j\in[N]$, the $(i,j)$-th block of $\metaNTK$ can be expressed as $[\metaNTK]_{ij} = \SingleTaskmetaNTK((\xxi),(\xxj))$. 

For cases such as $\metaNTK((\xx),(\XX)) \in \mathbb{R}^{kn \times knN}$, it is also obvious that $\metaNTK((\xx),(\XX))$ is a block matrix that consists of $1 \times N$ blocks of size $k n \times k n$, with the $(1,j)$-th block as follows for $j \in [N]$, $$[\metaNTK((\xx),(\XX))]_{1,j}= \SingleTaskmetaNTK((X,X'),(X_j,X_j')).$$

\end{proof}
\label{supp:global-convergence:lemma-proof:kernel-convergence}
\subsection{Proof of Theorem \ref{thm:global-convergence:restated}}
\label{supp:global-convergence:theorem-proof}
\begin{proof}[Proof of Theorem \ref{thm:global-convergence:restated}]

Based on these lemmas presented above, we can prove Theorem \ref{thm:global-convergence:restated}.

Lemma \ref{lemma:bounded_init_loss} indicates that there exist $R_0$ and $l^*$ such that for any width $l \geq l^*$, the following holds true over random initialization with probability at least $(1-\delta_0/10)$,
\begin{align}\label{eq:thm:global-convergence:init-loss-bound}
    \|g(\theta_0)\|_2  \leq R_0 ~. 
\end{align}

Consider $C = \frac{3 K R_0}{\sigma}$ in Lemma \ref{lemma:local-Liphschitzness}. 

First, we start with proving (\ref{eq:convergence-parameters:supp}) and (\ref{eq:convergence-loss:supp}) by induction. Select $\wt l > l^*$ such that (\ref{eq:thm:global-convergence:init-loss-bound}) and (\ref{eq:jacobian-lip}) hold with probability at least
$1- \frac{\delta_0}{5}$ over random initialization for every $l \geq \wt l$. As $t=0$, by (\ref{eq:gd&jacobian}) and (\ref{eq:jacobian-lip}), we can easily verify that (\ref{eq:convergence-parameters:supp}) and (\ref{eq:convergence-loss:supp}) hold true
\begin{align*}
\begin{cases}
        \|\theta_1 - \theta_0 \|_2 &= \| -\eta J(\theta_0)^\top g(\theta_0) \|_2 \leq \eta \|J(\theta_0)\|_{op} \|g(\theta_0)\|_2 \leq  \frac{\eta_0}{l} \|J(\theta_0)\|_{F} \|g(\theta_0)\|_2 \leq  \frac{K\eta_0}{\sqrt l}R_0 ~. \\
         \|g(\theta_0)\|_2  &\leq R_0
\end{cases}
\end{align*}
Assume (\ref{eq:convergence-parameters:supp}) and (\ref{eq:convergence-loss:supp}) hold true for any number of training step $j$ such that $j< t$. Then, by (\ref{eq:jacobian-lip}) and (\ref{eq:convergence-loss:supp}), we have
 \begin{align*}
    \|\theta_{t+1} - \theta_t \|_2 \leq \eta \|J(\theta_t)\|_{op} \|g(\theta_t)\|_2 \leq  \frac{K\eta_0}{\sqrt l} \left(1 - \frac {\eta_0 \mins}{3}\right)^t R_0 ~.
\end{align*}

Beside, with the mean value theorem and (\ref{eq:gd&jacobian}), we have the following 
\begin{align*}
    \|g(\theta_{t+1})\|_2 &= \| g(\theta_{t+1} - g(\theta_t) + g(\theta_t))\|_2\\
    &=\|J( \theta_t^\mu) (\theta_{t+1}-\theta_t) + g(\theta_t)\|_2\\
    &= \|(I-\eta J( \theta_t^\mu)J(\theta_t)^\top ) g(\theta_t)\|_2\\
    &\leq  \|I-\eta J( \theta_t^\mu)J(\theta_t)^\top\|_{op} \|g(\theta_t)\|_2\\
    &\leq \|I-\eta J( \theta_t^\mu)J(\theta_t)^\top\|_{op} \left(1 - \frac {\eta_0 \mins}{3}\right)^t R_0
\end{align*}
where we define $ \theta_t^\mu$ as a linear interpolation between $\theta_t$ and $\theta_{t+1}$ such that $\theta_t^\mu\coloneqq \mu \theta_t + (1-\mu) \theta_{t+1}$ for some $0< \mu <1$. 

Now, we will show that with probability $1-\frac{\delta_0}{2}$,
\begin{align*}
    \|I-\eta J( \theta_t^\mu)J(\theta_t)^\top\|_{op} \leq 1 - \frac{\eta_0 \mins}{3} .
\end{align*}
Recall that $\metantk_0 \rightarrow \metaNTK$ in probability, proved by Lemma \ref{lemma:kernel-convergence}. Then, there exists $\hat l$ such that the following holds with probability at least $1-\frac{\delta_0}{5}$ for any width $l > \hat l$,
\begin{align*}
    \|\metaNTK - \metantk_0\|_F \leq \frac{\eta_0 \mins}{3}.
\end{align*}
Our assumption $\eta_0 < \frac{2}{\Lev + \lev}$ makes sure that
\begin{align*}
    \|I - \eta_0 \metaNTK\|_{op} \leq 1 - \eta_0 \mins ~.
\end{align*}
Therefore, as $l \geq (\frac{18 K^3 R_0}{\mins^2})^2$, with probability at least $1- \frac{\delta_0}{2}$ the following holds,
\begin{align*}
    &\qquad \|I - \eta J(\theta_t^\mu) J(\theta_t)^\top\|_{op} \\
    &= \|I - \eta_0 \metaNTK + \eta_0 \metaNTK - \metantk_0 +  \eta\left(J(\theta_0) J(\theta_0)^\top - J(\theta_t^\mu )J(\theta_t)^\top\right)\|_{op}\\
    &\leq \|I - \eta_0 \metaNTK\|_{op} + \eta_0 \|\metaNTK - \metantk_0\|_{op} + \eta\|J(\theta_0) J(\theta_0)^\top - J(\theta_t^\mu )J(\theta_t)^\top\|_{op}\\
    &\leq 1- \eta_0 \mins + \frac{\eta_0\mins}{3} + \eta_0 K^2 (\|\theta_t - \theta_0\|_2 + \|\theta_t^\mu - \theta_0\|_2)\\
    &\leq 1 - \eta_0 \mins + \frac{\eta_0 \mins}{3} +  \frac{6 \eta_0 K^3 R_0}{\mins\sqrt l}\\
    &\leq 1 - \frac{\eta_0 \mins}{3}
\end{align*}
where we used the equality $\frac 1 l J(\theta_0)J(\theta_0)^\top = \metantk_0$.

Hence, as we choose $\Lambda = \max\{l^*,\wt l, \hat l, \frac{18 K^3 R_0}{\mins^2})^2\}$, the following holds for any width $l > \Lambda$ with probability at least $1-\delta_0$ over random initialization
\begin{align}
    \|g(\theta_{t+1}\|_2  \leq \|I-\eta J( \theta_t^\mu)J(\theta_t)^\top\|_{op} \left(1 - \frac {\eta_0 \mins}{3}\right)^t R_0 \leq \left(1 - \frac {\eta_0 \mins}{3}\right)^{t+1} R_0,
\end{align}
which finishes the proof (\ref{eq:convergence-loss:supp}).

Finally, we prove (\ref{eq:convergence-metantk:supp}) by
\begin{align*}
    \|\metantk_0 - \metantk_t\|_F &= \frac{1}{l} \|J(\theta_0) J(\theta_0)^\top - J(\theta_t) J(\theta_t)^\top\|_F \\
    &\leq \frac{1}{l} \|J(\theta_0)\|_{op} \|J(\theta_0)^\top - J(\theta_t)^\top\|_F + \frac 1 l \|J(\theta_t)-J(\theta_0)\|_{op} \|J(\theta_t)^\top \|_F\\
    &\leq 2 K^2 \|\theta_0 - \theta_t\|_2\\
    &\leq \frac{6 K^3 R_0}{\mins \sqrt l},
\end{align*}
where we used (\ref{eq:convergence-parameters:supp}) and Lemma \ref{lemma:local-Liphschitzness}.
\end{proof}

\section{Proof of Corollary \ref{corr:GBML-output} (GBML Output)}\label{supp:GBML-output}
In this section, we will provide proof of Corollary \ref{corr:GBML-output}. Briefly speaking, with the help of Theorem \ref{thm:global-convergence:restated}, we first show the training dynamics of GBML with over-parameterized DNNs can be described by a differential equation, which is analytically solvable. By solving this differential equation, we obtain the expression for GBML output on any training or test task. 

Below, we first restate Corollary \ref{corr:GBML-output}, and then provide the proof.

\begin{corollary}[GBML Output (Corollary \ref{corr:GBML-output} Restated)]\label{corr:GBML-output:restated}
In the setting of
Theorem \ref{thm:global-convergence}, the training dynamics of the GBML can be described by a differential equation
$$\frac{d F_t(\XXY)}{d t}=- \eta  \, \metantk_0   (F_t(\XXY) - \Y)$$
where we denote $F_t \equiv F_{\theta_t}$ and $\metantk_0 \equiv \metantk_{\theta_0}((\XXY),(\XXY))$ for convenience.

Solving this differential equation, we obtain the meta-output of GBML on training tasks at any training time as 
\begin{align}
    F_t(\XXY)=(I - e^{- \eta\metantk_0 t})\Y + e^{-\eta \metantk_0 t}F_{0}(\XXY) \,. 
\end{align}

Similarly, on arbitrary test task $\task=(\xyxy)$, the meta-output of GBML is
\begin{align}
F_t(\xxy) 
=F_{0}(\xxy)  +  \metantk_0(\xxy) \T^{\eta}_{\metantk_0}(t)
\left(\Y-F_0(\XXY)\right)
\end{align}
where $\metantk_0(\cdot)\equiv \metantk_{\theta_0}(\cdot,(\XXY))$ and $\T^{\eta}_{\metantk_0}(t)=\metantk_0^{-1}\left(I- e^{-\eta\metantk_0 t}\right)$ are shorthand notations.

\begin{proof}
For the optimization of GBML, the gradient descent on $\theta_t$ with learning rate $\eta$ can be expressed as
\begin{align}
    \theta_{t+1} &= \theta_t - \eta \nabla_{\theta_t}  \loss(\theta_t) \nonumber \\
    &=\theta_t - \frac{1}{2}\eta \nabla_{\theta_t} \|F_{\theta_t}(\XXY) - \Y\|_2^2 \nonumber \\
    &=\theta_t - \eta \nabla_{\theta_t}F_{\theta_t}(\XXY)^\top \left(F_{\theta_t}(\XXY) - \Y\right) 
\end{align}

Since the learning rate $\eta$ is sufficiently small, the \textit{discrete-time} gradient descent above can be re-written in the form of \textit{continuous-time} gradient descent (i.e., gradient flow),
\begin{align}
    \frac{d \theta_t}{d t} = - \eta \nabla_{\theta_t}F_{\theta_t}(\XXY)^\top \left(F_{\theta_t}(\XXY) - \Y\right) \nonumber
\end{align}
Then, the training dynamics of the meta-output $F_{t}(\cdot)\equiv F_{\theta_t}(\cdot)$ can be described by the following differential equation,
\begin{align}
    \frac{d F_t(\XXY)}{d t} &= \nabla_{\theta_t} F_t(\XXY) \frac{d \theta_t}{d t}\nonumber\\
    &=- \eta \nabla_{\theta_t} F_{t}(\XXY)  \nabla_{\theta_t}F_{t}(\XXY)^\top \left(F_{t}(\XXY) - \Y\right) \nonumber\\
    &=-\eta \metantk_{t} \left ( F_{t}(\XXY) - \Y \right) \label{eq:corr:diff-F}
\end{align}
where $\metantk_t = \metantk_{t}((\XXY),(\XXY))=\nabla_{\theta_t} F_{t}(\XXY)  \nabla_{\theta_t}F_{t}(\XXY)^\top$.

On the other hand, Theorem \ref{thm:global-convergence:restated} gives the following bound in (\ref{eq:convergence-metantk:supp}),
\begin{align}
        \sup_{t} \|  \metantk_0 - \metantk_t\|_F  &\leq \frac {6K^3\lss}{\mins}  l^{-\frac 1 2},
\end{align}
indicating $\metantk_t$ stays almost constant during training for sufficiently over-parameterized neural networks (i.e., large enough width $l$). Therefore, we can replace $\metantk_t$ by $\metantk_0$ in (\ref{eq:corr:diff-F}), and get
\begin{align}
     \frac{d F_t(\XXY)}{d t} = -\eta \metantk_{0} \left ( F_{t}(\XXY) - \Y \right) ,
\end{align}
which is an ordinary differential equation (ODE) for the meta-output $F_t(\XXY)$ w.r.t. the training time $t$.

This ODE is analytically solvable with a unique solution. Solving it, we obtain the meta-output on training tasks at any training time $t$ as,
\begin{align}
    F_t(\XXY)=(I - e^{- \eta\metantk_0 t})\Y + e^{-\eta \metantk_0 t}F_{0}(\XXY).
\end{align}
The solution can be easily extended to any test task $\task = (\xyxy)$, and the meta-output on the test task at any training time is
\begin{align}
F_t(\xxy) =F_{0}(\xxy)  +  \metantk_0(\xxy) \T^{\eta}_{\metantk_0}(t)
\left(\Y-F_0(\XXY)\right) ,
\end{align}
where $\metantk_0(\cdot)\equiv \metantk_{\theta_0}(\cdot,(\XXY))$ and $\T^{\eta}_{\metantk_0}(t)=\metantk_0^{-1}\left(I- e^{-\eta\metantk_0 t}\right)$ are shorthand notations.
\end{proof}

\end{corollary}

\section{Introduction to Kernel Regression}
\label{supp:intro-kernel}

The purpose of this section is to familiarize readers with kernel regression, a well-studied method with theoretical guarantees for regression and classification in the setting of supervised learning.

Consider a supervised learning task of binary classification, $\task(\xyxy)\in \mathbb R^{d\times n} \times \mathbb R^{ n} \times \mathbb R^{d\times m} \times \mathbb R^{m}$, where $(X',Y')$ is the training dataset and $(X,Y)$ is the test dataset. Suppose we have kernel function $\Psi(\cdot,\star)$, then the prediction of a standard kernel regression on test samples is
\begin{align}\label{eq:supp:standard-kernel-regression}
    \hat Y = \Psi(X,X') \Psi(X',X')^{-1} Y'
\end{align}
where $\Psi(X,X') \in \mathbb{R}^{n \times m}$ and $\Psi(X',X') \in \mathbb{R}^{m \times m}$. 

Since it is binary classification, the set of training labels, $Y'$, is usually a vector with elements as $\{0,1\}$ (or $\{-1,1\}$), where $0$ and $1$ represent two classes, separately. In this case, for an element of $\hat Y$, if its value is greater or equal than $\frac{1}{2}$, then it is considered to predict the class of $1$; if its value is less than $\frac{1}{2}$, then it predicts the class of $0$.

In the case of multi-class classification, kernel regression methods usually use one-hot labels. For instance, if there are 5 classes and the training labels are $[3,2,5,3,\dots]$, then the one-hot version of $Y'$ is expressed as
\begin{align}\label{eq:supp:one-hot}
Y'=
    \begin{bmatrix}
0 & 0 & 1 & 0 & 0\\
0 & 1 & 0 & 0 & 0\\
0 & 0 & 0 & 0 & 1\\
0 & 0 & 1 & 0 & 0\\
 \vdots&  \vdots& \vdots & \vdots & \vdots\\
\end{bmatrix}
\end{align}
In this way, each column represents an individual dimension. Specifically, the kernel regression, (\ref{eq:supp:standard-kernel-regression}), is doing regression in each of these dimensions separately.

The derived kernel regression for few-shot learning in Theorem \ref{thm:MNK} is very different from this standard one, but the forms are similar.

\section{Gradient-Based Meta-Learning as Kernel Regression}
\label{supp:MetaNTK}
In this section, we first make an assumption on the scale of parameter initialization, then we restate Theorem \ref{thm:MNK}. After that, we provide the proof for Theorem \ref{thm:MNK}.

\cite{lee2019wide} shows the output of a neural network randomly initialized following (\ref{eq:recurrence}) is a zero-mean Gaussian with covariance determined by $\sigma_w$ and $\sigma_b$, the variances corresponding to the initialization of weights and biases. Hence, small values of $\sigma_w$ and $\sigma_b$ can make the outputs of randomly initialized neural networks approximately zero. We adopt the following assumption from \cite{CNTK} to simplify the expression of the kernel regression in Theorem \ref{thm:MNK}.

\begin{assumption}[Small Scale of Parameter Initialization]\label{assum:small-init}
The scale of parameter initialization is sufficiently small, i.e., $\sigma_w,\sigma_b$ in (\ref{eq:recurrence}) are small enough, so that $f_0(\cdot) \simeq 0$.
\end{assumption}

Note the goal of this assumption is to make the output of the randomly initialized neural network negligible. The assumption is quite common and mild, since, in general, the outputs of randomly initialized neural networks are of small scare compared with the outputs of trained networks \cite{lee2019wide}. 

\begin{theorem}[GBML as Kernel Regression (Theorem \ref{thm:MNK} Restated)]\label{thm:MNK:supp}
Suppose learning rates $\eta$ and $\lambda$ are infinitesimal. As the network width $l$ approaches infinity, with high probability over random initialization of the neural net, the GBML output, (\ref{eq:F_t:main_text}), converges to a special kernel regression,
\begin{align}\label{eq:F_t-MetaNTK:supp}
F_t(\xxy)= G_\NTK^{\tau}(\xxy) +\metaNTK((\xx),(\XX)) \T^{\eta}_{\metaNTK}(t) \left(\Y-G_{\NTK}^{\tau}(\XXY)\right)
\end{align}
where $G$ is a function defined below, $\NTK$ is the neural tangent kernel (NTK) function from \cite{ntk} that can be analytically calculated without constructing any neural net, and $\metaNTK$ is a new kernel, which name as Meta Neural Kernel (MNK). The expression for $G$ is
\begin{align}
    G_\NTK^\tau(\xxy) =  \NTK(X,X')\widetilde{T}^{\lambda}_\NTK (X',\tau)   Y'.
\end{align}
where $\widetilde{T}^{\lambda}_\NTK (\cdot,\tau) \coloneqq \NTK(\cdot,\cdot)^{-1}(I-e^{-\lambda \NTK(\cdot,\cdot) \tau}) $. Besides, $G_\NTK^{\tau}(\XXY) = (G_\NTK^{\tau}(\xxyi))_{i=1}^N$.

The MNK is defined as $\metaNTK \equiv \metaNTK((\XX),(\XX)) \in \mathbb{R}^{knN \times knN}$, which is a block matrix that consists of $N \times N$ blocks of size $kn\times kn$. For $i,j \in [N]$, the $(i,j)$-th block of $\metaNTK$ is
\begin{align} \label{eq:MetaNTK_ij=kernel:supp}
    [\metaNTK]_{ij}=\SingleTaskmetaNTK((\xxi),(\xxj)) \in \mathbb R^{kn \times kn} ,
\end{align}
where $\SingleTaskmetaNTK: (\mathbb{R}^{n \times k} \times \mathbb{R}^{m\times k}) \times  (\mathbb{R}^{n \times k} \times \mathbb{R}^{m\times k}) \rightarrow \mathbb{R}^{nk \times nk}$ is a kernel function defined as
\begin{align}\label{eq:MetaNTK_ij:supp}
    \SingleTaskmetaNTK((\cdot,\ast), (\bullet, \star)) &= \NTK(\cdot,\bullet) + \NTK(\cdot,\ast)\widetilde{\T}_{\NTK}^\lambda(\ast,\tau)\NTK(\ast,\star)\widetilde{\T}_{\NTK}^\lambda(\star,\tau)^\top \NTK(\star,\bullet) \nonumber\\
 &\quad -\NTK(\cdot,\ast)\widetilde{\T}_{\NTK}^\lambda(\ast,\tau) \NTK(\ast,\bullet) - \NTK(\cdot,\star) \widetilde{\T}_{\NTK}^\lambda(\star,\tau)^\top \NTK(\star,\bullet) .
\end{align}
The $\metaNTK((\xx),(\XX)) \in \mathbb{R}^{kn \times knN}$ in (\ref{eq:F_t-MetaNTK}) is also a block matrix, which consists of $1 \times N$ blocks of size $k n \times k n$, with the $(1,j)$-th block as follows for $j \in [N]$,
\begin{align}\label{eq:MetaNTK_1j:supp}
    [\metaNTK((\xx),(\XX))]_{1,j}= \SingleTaskmetaNTK((X,X'),(X_j,X_j')).
\end{align}
\end{theorem} 
\begin{proof}
First, (\ref{eq:F_t:main_text}) shows that the output of GBML on any test task $\task = (\xyxy)$ can be expressed as
\begin{align}\label{eq:thm:MNK:F_t:finite-width}
F_t(\xxy) =F_{0}(\xxy)  + \metantk_0(\xxy)\T^{\eta}_{\metantk_0}(t)\left(\Y-F_0(\XXY)\right)    
\end{align}
Note (\ref{eq:lemma-proof:kernel-convergence:F_0}) in Appendix \ref{supp:global-convergence:kernel-convergence} shows that 
\begin{align}\label{eq:thm:MNK:F_0:finite-width}
 F_0(\xxy) =f_0(\sX) + \ntk_0(\xx)\widetilde{\T}_{\ntk_0}^\lambda(\sX',\tau)\left(\sY' - f_0(\sX')\right),
\end{align}
With Assumption \ref{assum:small-init}, we can drop the terms $f_0(X)$ and $f_0(X')$ in (\ref{eq:thm:MNK:F_0:finite-width}). Besides, from \cite{ntk,CNTK,lee2019wide}, we know $\lim_{l \rightarrow \infty}\ntk_0(\cdot,\star) = \NTK(\cdot,\star)$, the Neural Tangent Kernel (NTK) function, a determinisitc kernel function. Therefore, $F_0(\xxy)$ can be described by the following function as the width appraoches infinity,
\begin{align}\label{eq:supp:MNK:G}
    \lim_{l \rightarrow \infty} F_0(\xxy) =  G_\NTK^\tau(\xxy) =  \NTK(X,X')\widetilde{T}^{\lambda}_\NTK (X',\tau)   Y'.
\end{align}
where $\widetilde{T}^{\lambda}_\NTK (\cdot,\tau) \coloneqq \NTK(\cdot,\cdot)^{-1}(I-e^{-\lambda \NTK(\cdot,\cdot) \tau}) $. Besides,  $G_\NTK^{\tau}(\XXY) = (G_\NTK^{\tau}(\xxyi))_{i=1}^N$.

In addition, from Lemma \ref{lemma:kernel-convergence}, we know $\lim_{l \rightarrow \infty}\metantk_0(\cdot,\star)=\metaNTK(\cdot,\star)$. Combined this with (\ref{eq:supp:MNK:G}), we can express (\ref{eq:thm:MNK:F_t:finite-width}) in the infinite width limit as
\begin{align}
F_t(\xxy)= G_\NTK^{\tau}(\xxy) +\metaNTK((\xx),(\XX)) \T^{\eta}_{\metaNTK}(t) \left(\Y-G_{\NTK}^{\tau}(\XXY)\right)
\end{align}
where $\metaNTK(\cdot,\star)$ is a kernel function that we name as Meta Neural Kernel function. The derivation of its expression shown in (\ref{eq:MetaNTK_ij=kernel:supp})-(\ref{eq:MetaNTK_1j:supp}) can be found in Appendix \ref{supp:global-convergence:lemma-proof:kernel-convergence}.
\end{proof}

\section{Generalization of Gradient-Based Meta-Learning}
\label{supp:generalization}

In this section, we prove a generalization bound on gradient-based meta-learning (GBML) with over-paramterized neural networks, corresponding to Theorem \ref{thm:gen-bound}, which is an informal version of Theorem \ref{thm:gen-bound:based-on-MetaNTK} shown below.

As stated in Sec. \ref{sec:generalization:gen-bound}, we consider \textit{stochastic} gradient descent (SGD) for GBML training, instead of (full-batch) gradient descent. In practice, the training of GBML usually uses SGD or its adaptive variants \cite{maml,reptile,alphaMAML}. Hence, our SGD setting is very natural for GBML. The specific SGD algorithm we are considering is shown in Algorithm \ref{algo:SGD-GBML}. Note that same as Algorithm \ref{alg:mamlsup}, we also present the case of 1-step inner-loop meta-adaptation (i.e., $\tau = 1$ in \eqref{eq:meta-adaption-descrete}) for simplicity.

\begin{algorithm}[t!]
\caption{MAML for Few-Shot Learning (SGD Version)}
\label{algo:SGD-GBML}
\begin{algorithmic}[1]
{\footnotesize
\REQUIRE $\{\task_i\}_{i=1}^N$: Training tasks with a \textit{random order}
\REQUIRE $\eta$, $\lambda$: Learning rate hyperparameters
\STATE Randomly initialize $\theta_0$
  \FORALL{$i = 1,2,...,N$} 
    \STATE Evaluate the loss of $f_{\theta_{i-1}}$ on support samples of $\task_i$: $ \ell(f_{\theta_{i-1}}(X_i'),Y_i')$
    \STATE Compute adapted parameters $\theta_{i-1}'$ with gradient descent: $\theta_{i-1}'=\theta_{i-1}-\lambda \nabla_{\theta_{i-1}} \ell(f_{\theta_{i-1}}(X_i'),Y_i')$
    \STATE Evaluate the loss of $f_{\theta_{i-1}'}$ on query samples of $\task_i$: $\ell(f_{\theta_{i-1}'}(X_i), Y_i)$
     \STATE Update parameters with gradient descent:\\
    $\theta_i \leftarrow \theta_{i-1} - \eta \nabla_{\theta_{i-1}} \ell ( f_{\theta_{i-1}'}(X_i),Y_i)$ 
    
 \ENDFOR
 \STATE \textbf{Output:} Randomly choose $\hat \theta$ uniformly from $\{ \theta_0 ,\ldots, \theta_{N-1} \}$.

}
\end{algorithmic}
\end{algorithm}

For simplicity of derivations, we consider a parameter initialization scheme slightly different from the one in Appendix \ref{supp:NTK-setup}. 
Specifically, at initialization, for each hidden layer $i=1,2,...,L$, the parameters of this layer, $\theta^{(i)}$, has i.i.d. entries drawn from a normal distribution $N(0,2/l_i)$, where $l_i$ is the width of layer $i$. At the last layer, i.e., $(L+1)$-th layer, the normal distribution changes to $N(0,1/l_L)$, since the last layer has no activation function. Besides, same as Theorem \ref{thm:global-convergence}, we assume the width of every hidden layer is the same for convenience, i.e., $l=l_1=l_2=...=l_L$.

We make several assumptions, including Assumptions \ref{assum:non-degeneracy}, \ref{assum:same-width}, \ref{assum:activation}, \ref{assum:full-rank} from Appendix \ref{supp:global-convergence}, and the following Assumption \ref{assum:input-normalization} on input data. Note that this Assumption \ref{assum:input-normalization} can be relaxed to $c_1 \leq \|x\|_2 \leq c_2$ for some $c_2 > c_1 > 0$, following \cite{cao2020generalization}.

\begin{assumption}[Input Normalization] \label{assum:input-normalization}For any single sample $x$ in any training or test tasks, $\|x\|_2 = 1$.
\end{assumption}

By the definition of task distribution $\mathscr{P}$ in Sec. \ref{sec:generalization:gen-bound}, we can define the population loss (i.e., generalization error) of GBML on the task distribution $\mathscr{P}$.
\begin{definition}[Population Loss, i.e., Generalization Error]\label{def:pop-loss}
Suppose $\mathscr{P}$ is the task distribution of interest, the population loss for GBML function $F$ with parameteres $\theta$ is defined as,
\begin{align}\label{eq:def-population-loss}
     \loss_{\mathscr{P}}(\theta)\triangleq \underset{\substack{\task \sim \mathscr{P}\\(\xyxy)\sim \mathbb{P}_\task }} {\E}\left[ ~\ell(F_\theta(\xxy),Y)~\right]~,
\end{align}
where $\ell(\cdot,\ast)=\frac 1 2\|\cdot - \ast\|^2_2$ is the $\ell_2$ loss function, and the expectation is  
taken respect to the random sampling of tasks and the random sampling
of input data. 
\end{definition}

Below, we define a function class of meta neural random features (MetaNRF). MetaNRF can be viewed as a reference function class
to estimate the ``realizability'' of the tasks of interest, i.e., how easy or hard the tasks can be resolved by GBML with over-parameterized DNNs. In addition, we also define an auxiliary MetaNRF function class.

\begin{definition}[Meta Neural Random Feature] 
Let $\theta_0$ be a randomly initialized network parameters. The function class of meta neural random feature (MetaNRF) is defined as 
\begin{align}
    \F(\theta_0,R)=\left\{h(\ast)=F_{\theta_0}(\ast)+\nabla_{\theta} F_{\theta_0}(\ast) \cdot \theta ~\middle|~ \theta \in \B(\mathbf{0}, R\cdot l^{-1/2})\right\}, 
\end{align}
where $R>0$ can be viewed as the radius of the MetaNRF function class, and $\B$ denotes the parameter neighborhood, defined as
\begin{align}
    \B(\theta,r)\coloneqq \left\{\theta' \middle| \|{\theta'}-\theta\| \leq r\right\}
\end{align}

Similarly, we define a new function class, the Auxiliary MetaNRF function class, $\wt \F$, such that
\begin{align}
    \wt \F(\theta_0,R)=\left\{h(\ast)=\nabla_{\theta} F_{\theta_0}(\ast) \cdot \theta ~\middle| ~\theta \in \B(\mathbf{0}, R\cdot l^{-1/2})\right\}.
\end{align}

\end{definition}

Below, we define some auxiliary notation for convenience.

\begin{definition}[Auxiliary Notation] For any task $\task = (\xyxy)$, we adopt the following notation from \eqref{eq:gen:def-terms}.

\begin{align}\label{eq:gen:def-terms'}
\left\{
\begin{aligned}
\wt Y &= Y - F_{\theta_0}(\xxy)\\
\wt \Y &= \Y - F_{\theta_0}(\XXY)\\
\widetilde F_\theta(\cdot)&=F_\theta(\cdot) - F_{\theta_0}(\cdot)
\end{aligned}
\right.
\end{align}
In addition, we define 
$$\wt X = (X,X',Y'),$$
so that or any function $h$ (e.g., $h$ could be $F_\theta$ or $\wt F_\theta$), we have
$$h(\wt X) = h(X,X',Y') ~.$$
\end{definition}

Below, we provide a generalization bound on GBML based upon the MetaNRF function class. 

\begin{theorem}[Generalization Bound Based on MetaNRF]\label{thm:gen-bound:based-on-radius}
For arbitrary $\delta \in (0,e^{-1}]$ and $R>0$, there exists 
\begin{align}
    l^*(\delta,R,L,N,n)= \wt \cO (\poly(R,L)) \poly(N,n) \log(1/\delta),
\end{align}
such that if the width $l$ satisfies $l\geq l^*(\delta,R,L,N,n)$, then with learning rate $\eta=\frac{\kappa R}{l\sqrt{Nn}}$ for some small constant $\kappa >0$, the output of Algorithm \ref{algo:SGD-GBML} satisfies the following inequality with probability at least $1-\delta$ over random initialization,
\begin{align}\label{eq:gen-bound:radius:0}
    \loss_{\mathscr{P}} \coloneqq \E[\loss_{\mathscr{P}}(\hat \theta)] \leq \inf_{h \in \wt F(\theta_0,R)}\left \{\frac{1}{N}\sum_{i=1}^N\ell\left(h(\wx_i), \wy_i\right)  \right\} + \cO\left(\frac{(L+1)R}{\sqrt{Nn}} + \sqrt{\frac{\log (1/\delta)}{N}}\right),
\end{align}
where the expectation is over a uniform draw of $\hat \theta$ from $\{\theta_i\}_{i\in[N]}$ demonstrated in Algorithm \ref{algo:SGD-GBML}, and also an uniform draw of the task-parameter $\gamma$ with the input datasets.
\end{theorem}

The population loss bound, (\ref{eq:gen-bound:radius:0}), consists of two terms. The first one relates the training loss by Algorithm \ref{algo:SGD-GBML} with the Auxiliary MetaNRF function class, while the second term represents the standard large-deviation error.

Note that the generalization bound in Theorem \ref{thm:gen-bound:based-on-radius} depends on the radius of the MetaNRFT function class (i.e., $R$). Below, we derive a generalization bound based on Meta Neural Kernel of Theorem \ref{thm:MNK}, shown in Theorem \ref{thm:gen-bound:based-on-MetaNTK}. The bound in Theorem \ref{thm:gen-bound:based-on-MetaNTK} can be viewed as the bound in Theorem \ref{thm:gen-bound:based-on-radius} with radius as $R=\widetilde {\mathcal O} (\sqrt{\widetilde{\Y}^\top \metaNTK^{-1} \widetilde{\Y}})$, while other terms are proved to be negligible with this choice of $R$.

\begin{theorem}[Generalization Bound Based on Meta Neural Kernel]\label{thm:gen-bound:based-on-MetaNTK}
Denote $\lev=\lev(\metaNTK)$ as the least eigenvalue of the Meta Neural Kernel $\metaNTK$. For arbitrary $\delta \in (0,\frac{1}{e}])$, there exists $l^*(\delta,L,N,n,\lev)$ such that if the width $l$ satisfies $l \geq l^*(\delta,L,N,n,\lev)$, then with learning rate $\eta = \cO(\frac{\kappa}{l} \sqrt{\frac{{\wt \Y}^\top \metaNTK^{-1} \wt \Y}{Nn}})$ for some small constant $\kappa > 0$, the output of Algorithm \ref{algo:SGD-GBML} satisfies the following generalization bound with probability at least $1-\delta$,
\begin{align}\label{eq:gen-bound:metaNTK:0}
    \loss_{\mathscr{P}} \coloneqq \E[\loss_{\mathscr{P}}(\hat \theta)] \leq
    \wt \cO\left((L+1)\sqrt{\frac{{\wt{\Y}_{\scriptscriptstyle{G}}}^\top \metaNTK^{-1} \wt {\Y}_{\scriptscriptstyle{G}}}{Nn}}\right)+
    \cO\left(\sqrt{\frac{\log (1/\delta)}{N}}\right),
\end{align}
where the expectation is over a uniform draw of $\hat \theta$ from $\{\theta_i\}_{i\in[N]}$ demonstrated in Algorithm \ref{algo:SGD-GBML}, and $\wt \Y$ is defined by \eqref{eq:def:Y_G}.
\end{theorem}

For the bound above, i.e., (\ref{eq:gen-bound:metaNTK:0}), the first term is the dominating term, while the second one is just a standard deviation error term that can be ignored. 

Now, let us discuss the dominating term of (\ref{eq:gen-bound:metaNTK:0}),
\begin{align}\label{eq:gen-bound:complexity-measure}
    (L+1)\sqrt{\frac{{\wt \Y}^\top \metaNTK^{-1} \wt \Y}{Nn}}.
\end{align}
\textbf{Remarks.} This term can be seen as a \textit{data-dependent} \textit{complexity measure} of tasks, which can be used to predict the test loss of GBML with DNNs. One of the advantages of the bound is that it is \textit{data-dependent}, i.e., the complexity measure in (\ref{eq:gen-bound:complexity-measure}) can be directly calculated from data of tasks, $\{\task_i={\xyxyi}\}_{i\in[N]}$, without the need to construct a neural networks or to assume a ground-truth data generating model.

\subsection{Proof of Theorem \ref{thm:gen-bound:based-on-radius}}

The proof of Theorem \ref{thm:gen-bound:based-on-radius} depends on several lemmas. Below, we present these lemmas first, then demonstrate the proof of Theorem \ref{thm:gen-bound:based-on-MetaNTK}. The proofs of these lemmas are in Appendix \ref{supp:generalization:helper-lemmas-proof}. 

\begin{lemma}[Meta-Output is Almost Linear in Weights around Initialization]\label{lemma:F-linear-in-weight-near-init} There exists a constant $\kappa>0$, $l^*(\kappa,L,N,n)$, and $\lambda^*(\kappa,L,N,n)$ such that, if the width $l$ satisfies $l\geq l^*(\kappa,L,N,n)$ and the learning rate $\lambda$ for meta-adaption satisfies $\lambda \leq \lambda^*(\kappa,L,N,n)$, then for any $ i \in [N]$ and any parameters $\theta, \Bar \theta \in \B(\theta_0, \omega)$ with $\omega \leq \kappa (L+1)^{-6} l^{-\frac 1 2}$, the following holds uniformly with probability at least $1-\cO(n (L+1)^2) \cdot \exp(-\Omega(\poly(l,\omega, L+1))$, 
\begin{align} \label{eq:linear-in-weights}
    \|F_\theta(\wt X)-F_{\Bar \theta}(\wt X) - \nabla_{\theta} F_\theta(\wt X)\cdot (\Bar \theta -\theta)\| \leq \cO(\omega^{1/3} (L+1)^2 \sqrt{l \log (l)}) \|\Bar \theta - \theta\| 
\end{align}
and similarly, since $\wF_{\theta}(\wx)$ is just $\F_{\theta}(\wx)$ shifted by a term independent of $\theta$, we have
\begin{align} \label{eq:linear-in-weights:1}
    \|\wF_\theta(\wt X)-\wF_{\Bar \theta}(\wt X) - \nabla_{\theta} \wF_\theta(\wt X)\cdot (\Bar \theta -\theta)\| \leq \cO(\omega^{1/3} (L+1)^2 \sqrt{l \log (l)}) \|\Bar \theta - \theta\|.
\end{align}
\end{lemma}

\begin{lemma}[Loss is Almost Convex Near Initialization]\label{lemma:loss-convex-near-init}  
Define $\loss_i(\theta)\coloneqq \ell(F_\theta(\xxyi), \sY)$ as the loss of a GBML model with parameters $\theta$ on task $\task_i=(\xyxyi)$.
\begin{align}
    \loss_i (\wt \theta) \geq \loss_i(\theta) + \nabla_{\theta} \loss_i(\wt X,\theta)\cdot (\wt \theta -\theta) - \epsilon
\end{align}
\end{lemma}

\begin{lemma}[Gradients Have Bounded Norm Near Initialization] \label{lemma:grad-norm-bounded-near-init}
For large enough width, i.e., $l\geq l^*$ for some $l^* \in \mathbb N_{+}$, there exists $\kappa>0$ such that for $\theta \in \B(\theta_0, \omega)$ with $\omega \leq \kappa (L+1)^{-6} l^{-\frac 1 2}$, the following holds uniformly with probability at least $1-\cO(NL^2)\cdot e^{-\Omega (l w^{2/3}(L+1))}$ over random initialization, 
\begin{align}
    \|\nabla_{\theta} F_\theta (\wx)\|_F, ~ \|\nabla_{\theta} F_\theta (\wx)\|_F \leq \cO(\sqrt l)
\end{align}
\end{lemma}

\begin{lemma}[Bound on the Cumulative Loss]\label{lemma:bound-on-cum-loss}
$\forall \delta,\epsilon, R > 0$, there exists a critical width 
\begin{align}
    l^*(\epsilon, \delta, R, L) = \wt \cO (\poly(R,L)) \frac{\log(1/\delta)}{\poly(\epsilon)}
\end{align}
such that: if the width is $l \geq l^*$, learning rate is $\eta = \frac{\nu \epsilon }{(L+1)l}$, and define $\nu=\frac{(L+1)^2 R^2}{2 \nu \epsilon^2}$, then for any parameters $\theta^* \in \B(\theta_0, l^{-1/2} R)$, the following bound on the cumulative loss holds true with probability at least $1-\delta$ over random initialization,
\begin{align}
    \frac 1 N \sum_{i=1}^N \loss_i (\theta_{i-1}) \leq \frac{1}{N}\sum_{i=1}^N \loss(\theta^*) + 3\epsilon
\end{align}
\end{lemma}

Below, we present the proof of Theorem \ref{thm:gen-bound:based-on-radius}.

\begin{proof}[Proof of Theorem~\ref{thm:gen-bound:based-on-radius}]
Consider $\epsilon = \frac{(L+1)R}{\sqrt{2 \nu Nn}}$ and $\eta=\frac{\sqrt{\nu} R}{\sqrt{2Nn}l}$ in Lemma \ref{lemma:bound-on-cum-loss}. Then, Lemma \ref{lemma:bound-on-cum-loss} indicates that with probability at least $1-\delta$,
\begin{align}\label{eq:gen-bound:radius:1}
        \frac 1 N \sum_{i=1}^N \loss (\theta_i) \leq \frac{1}{N}\sum_{t=0}^N \loss(\theta^*) + \frac{3(L+1)R}{\sqrt{2\nu Nn}}
\end{align}
Applying Proposition 1 of \cite{cesa2004generalization}, the following holds with probability at least $1-\delta$,
\begin{align}\label{eq:gen-bound:radius:2}
    \frac{1}{N}\sum_{i=1}^N \loss_{\mathscr{P}}(\theta_{i-1})  \leq \frac 1 N \sum_{i=1}^N \loss (\theta_i) + \sqrt{\frac{2 \log (1/\delta)}{Nn}}
\end{align}
By definition, $\E[\loss_{\mathscr{P}}(\hat \theta)]=\frac{1}{N}\sum_{i=1}^N \loss_{\mathscr{P}}(\theta_{i-1})$, where $\hat \theta$ is the selected parameters after training. Then, by (\ref{eq:gen-bound:radius:1}) and (\ref{eq:gen-bound:radius:2}), $\forall \theta^*\in \B(\theta_0,Rl^{-1/2})$, with probability at least $1-2 \delta$ we have 
\begin{align}\label{eq:gen-bound:radius:2.5}
    \E[\loss_{\mathscr{P}}(\hat \theta)] \leq \frac{1}{N}\sum_{t=0}^N \loss_i(\theta^*) +\frac{3(L+1)R}{\sqrt{2\nu nN}} + \sqrt{\frac{2 \log (1/\delta)}{N}}.
\end{align}

Now, consider a function in the auxiliary MetaNRF function class as  $h_{\theta_0,\theta^*}(\wx_i)\coloneqq \wt F_{\theta_0}(\wx_i) + \nabla_{\theta_0} \wF_{\theta_0} \cdot (\theta^*-\theta_0) \in \wt \F(\theta_0,R)$. 
Then, for $\Delta = \wt F_{\theta^*}(\wx_i) - h_{\theta_0,\theta^*}(\wx_i) $,
\begin{align}
    \loss_i(\theta^*)&= \ell(h_{\theta_0,\theta^*}(\wx_i) + \Delta, \wy_i) \\
    &= \|h_{\theta_0,\theta^*}(\wx_i) + \Delta - \wy_i\|^2\\
    &\leq \|h_{\theta_0,\theta^*}(\wx_i) - \wy_i\|^2 + 2\|h_{\theta_0,\theta^*}(\wx_i) - \wy_i\|\|\Delta\| + \|\Delta\|^2.
\end{align}
From Lemma \ref{lemma:F-linear-in-weight-near-init}, we know that 
\begin{align}\label{eq:gen-bound:radius:3}
\|\Delta\|&\leq \cO\left((R l^{-1/2})^{1/3} (L+1)^2 \sqrt{l \log (l)}\right) \|\Bar \theta - \theta\| \nonumber\\
&\leq \cO\left((L+1)^3 \sqrt{l \log(l)}\right) R^{4/3} \cdot l^{-2/3} 
\end{align}
where we used the condition $\theta^*\in \B(\theta_0,Rl^{-1/2})$ in the second inequality. 

Note that the dependence of (\ref{eq:gen-bound:radius:3}) on the width $l$ is $\cO(\sqrt{l \log(l)}l^{-2/3} )$. As result, $\|\Delta\|^2$ is the dominating term compared with the $\|\Delta\|$ term for large width. Since we consider very large width in this theorem, we can ignore the $\|\Delta\|$ term.

As long as $l \geq C R (L+1)^6[\log(l)]^3 N^{3/2}n^{-3/2}$ for some large enough constant $C>0$, (\ref{eq:gen-bound:radius:3}) can be bounded as
\begin{align} \label{eq:gen-bound:radius:4}
    \loss_i(\theta^*)&\leq  \|h_{\theta_0,\theta^*}(\wx_i) - \wy_i\|^2 + (L+1) R (N)^{-1/2}\nonumber\\
    &= \ell(h_{\theta_0,\theta^*}(\wx_i),\wy_i)+ (L+1) R (Nn)^{-1/2}
\end{align}

Plug (\ref{eq:gen-bound:radius:4}) into (\ref{eq:gen-bound:radius:2.5}), we have
\begin{align}
    \E[\loss_{\mathscr{P}}(\hat \theta)] \leq \frac{1}{N}\sum_{t=0}^N\ell(h_{\theta_0,\theta^*}(\wx_i), \wy_i) +\left(1+\frac{3}{\sqrt{2\nu}}\right)\frac{(L+1)R}{\sqrt{ Nn}} + \sqrt{\frac{2 \log (1/\delta)}{N}}
\end{align}
Rescaling $\delta$ and taking a infimum over $\theta^*\in \B(\theta_0,Rl^{-1/2})$ gives (\ref{eq:gen-bound:radius:0}). Proof is finished.
\end{proof}

\subsection{Proof of Theorem \ref{thm:gen-bound:based-on-MetaNTK}}

Now we present the proof of Theorem \ref{thm:gen-bound:based-on-MetaNTK}. The key idea is 
to bound the norm of $\theta$, which is the solution to the linear equation $\wY_i = \nabla_{\theta_0} \wF_{\theta_0}(\wX_i)\cdot \theta$ for $i =1,...,N$. Specifically, we look for $\theta^*$, the \textit{minimum distance solution} to $\theta_0$ that can let GBML fit the tasks $\{\task_i\}_{i\in[N]}$, and finally utilize it to construct a suitable function in the Auxiliary MetaNRF function class $\wt \F(\theta_0,\wt \cO({\wt \Y}^\top \metaNTK^{-1}\wt \Y))$.
\begin{proof}[Proof of Theorem \ref{thm:gen-bound:based-on-MetaNTK}]

Define $\lev = \lev(\metaNTK)$ and $\Lev=\Lev(\metaNTK)$ as the largest and least eigenvalues of $\metaNTK$, respectively. By Assumption \ref{assum:full-rank}, $\metaNTK$ is positive definite, hence $\lev >0$. Then, by Theorem \ref{thm:MNK} and standard matrix perturbation bound, there exists $l^*(\delta, L, N,n, \Lev,\lev)$ such that if the width $l$ satisfies $l \geq l^*(\delta, L, N,n, \Lev,\lev)$, then with probability at least $1-\delta$, $\metantk=\frac{1}{l}\nabla \wF_{\theta_0}(\wX) \nabla \wF_{\theta_0}(\wX)^\top $ is positive definite and the following holds true
\begin{align}\label{eq:gen-bound:metaNTK:2}
    \|\metantk^{-1}-\metaNTK^{-1}\| \leq \frac{ \wY^\top \metaNTK^{-1} \wY}{\|\wY\|^2}.
\end{align}
For convenience, define $G=\frac{1}{\sqrt{l}}\nabla \wF_{\theta_0}(\wX)^\top\in \mathbb R^{D \times Nnk}$, where $D$ is the dimension of network parameters. Consider the singular value decomposition of $G$: $G = U \Sigma V^\top$, where $U \in \mathbb R^{D \times knN}$,$V\in \mathbb R^{knN \times knN}$ are real orthonormal matrices, and $\Sigma \in \mathbb R^{knN \times knN}$ is a diagonal matrix. 

Define $\hat \theta^* = U \Sigma^{-1} V^\top \wY$, then
\begin{align}\label{eq:gen-bound:metaNTK:3}
    G^\top \hat \theta^* = (V\Sigma U^\top) (U \Sigma^{-1}V^\top \wY)=\wY.
\end{align}
Also, we have
\begin{align}
    \|\hat \theta^*\|^2 = \|U \Sigma^{-1}V^\top \wY\|^2 = \|\Sigma^{-1}V^\top \wY\|^2= \wY^\top V \Sigma^{-1}\Sigma^{-1}V^\top \wY= \wY^\top (G^\top G)^{-1} \wY= \wY^\top \metantk^{-1} \wY.
\end{align}
Furthermore, by (\ref{eq:gen-bound:metaNTK:2}) and (\ref{eq:gen-bound:metaNTK:3}), $\|\hat \theta^*\|^2$ can be bounded as
\begin{align}
    \|\hat \theta^*\|^2& = \wY^\top (\metantk^{-1}-\metaNTK^{-1})\wY  + \wY^\top \metaNTK^{-1} \wY \\
    &\leq \|\wY\|^2 \|(\metantk^{-1}-\metaNTK^{-1})\| +\wY^\top \metaNTK^{-1} \wY \\
    &\leq 2 \wY^\top \metaNTK^{-1} \wY 
\end{align}
Then, we scale $\hat \theta^*$ by a factor of $\frac{1}{\sqrt{l}}$ to get $\theta^*$, i.e., $\theta^* = \frac{1}{l}\hat \theta^*$. Obviously, $$\|\theta^*\| = \sqrt{\frac{2}{l}} \sqrt{\wY^\top \metaNTK^{-1} \wY },$$
hence, $\theta^* \in \B(\mathrm{0}, \frac{1}{\sqrt l} \sqrt{2\wY^\top \metaNTK^{-1} \wY})$. Note that by (\ref{eq:gen-bound:metaNTK:3}), we have $\wY = G^\top \hat \theta^* =\nabla_{\theta_0} \wF_{\theta_0}(\wX) \cdot \theta^*$, which implies
\begin{align}
    \wy_i = \nabla_{\theta_0} \wF_{\theta_0}(\wx_i) \cdot \theta^*.
\end{align}
For function $h(\cdot)=\nabla_{\theta_0} \wF_{\theta_0}(\cdot) ~ \theta^*$, we have
\begin{align}\label{eq:gen-bound:metaNTK:4}
    \ell(h(\wx_i), \wy_i) = 0.
\end{align}
Obviously, $h(\cdot) \in \wt \F (\theta_0, \sqrt{2\wY^\top \metaNTK^{-1} \wY })$. Therefore, applying Theorem \ref{thm:gen-bound:based-on-radius}, we can find the first term in (\ref{eq:gen-bound:radius:0}) vanishes because of (\ref{eq:gen-bound:metaNTK:4}), and the proof is basically finished. The only remaining step is to replace the $\wt \Y$ of \eqref{eq:gen-bound:metaNTK:0} with $\wt \Y_{\scriptscriptstyle{G}}$, which can be easily achieved by leveraging the convergence $\lim_{l\rightarrow \infty} \wt \Y = \wt \Y_{\scriptscriptstyle{G}}$.

\end{proof}
\subsection{Proof of Helper Lemmas}\label{supp:generalization:helper-lemmas-proof}
\subsubsection{Proof of Lemma \ref{lemma:F-linear-in-weight-near-init}}
\begin{proof}[Proof of Lemma \ref{lemma:F-linear-in-weight-near-init}]
Lemma 4.1 of \cite{cao2019generalization} indicates that for any sample set $X \in \mathbb R^{d\times n}$ (i.e., $n$ samples of dimension $d$),
\begin{align}
    \|f_\theta(X) - f_{\Bar \theta}(X) - \nabla f_\theta(X) \cdot (\Bar \theta - \theta )\|_2 \leq \cO(n\omega^{1/3} L^2 \sqrt{l \log (l)}) \|\Bar \theta - \theta\|_2. 
\end{align}
As for the meta-output $F$, given sufficiently small learning rate for meta adaption $\lambda$ and sufficiently large width $l$, for any parameters $\theta \in \B(\theta_0, \omega)$ and auxiliary samples $\wx = (\xxy)$, we have the following based on Sec. \ref{sec:global-convergence},
\begin{align}\label{eq:lemma:F-linear-in-weight-near-init:0}
    F_\theta(\wx) = F_\theta(\xxy) = f_\theta(X) + \ntk_{\theta}(X,X')\ntk_{\theta}^{-1} (I - e^{-\lambda \ntk_{\theta} \tau})(Y'-f_\theta(X')),
\end{align}
where $\ntk_\theta(X,X') = \frac{1}{l}\nabla_\theta F_\theta(X) \nabla_\theta F_\theta(X')^\top$, and $\ntk_{\theta}\equiv \ntk_\theta(X',X')$ for convenience.

Then, following the derivation of (\ref{eq:lemma:jacobian:F-grad}), we have
\begin{align}\label{eq:lemma:F-linear-in-weight-near-init:1}
    \nabla_{\theta} F_\theta(\wx) &=  \nabla_{\theta}f_\theta(\sX) - \ntk_{\theta}(X,X')\ntk_{\theta}^{-1} (I - e^{-\lambda \ntk_{\theta} \tau}) \nabla_{\theta}f_\theta(\sX')
\end{align}

Recall that the Lemma 1 of \cite{lee2019wide} proves the following holds true for arbitrary $X$ (note: $\cO(\cdot)$ is with respect to $\|\theta- \Bar \theta\|_2$ and $l$ in this whole proof)
\begin{align}
    \|\frac{1}{\sqrt l} \nabla f_\theta(X) - \frac{1}{\sqrt l}\nabla f_{\Bar \theta}(X)\|_F \leq \cO( \|\theta - \Bar \theta\|_2)\\
\|\frac{1}{\sqrt l} \nabla f_\theta(X) \|_F,~~ \|\frac{1}{\sqrt l}\nabla f_{\Bar \theta}(X) \|_F \leq \cO(1)
\end{align}
which indicates that $\|\ntk_\theta (X,X')\|_{op}\leq\|\ntk_\theta(X,X')\|_F=\|\frac{1}{l} \nabla f_\theta(X)  \nabla f_\theta(X')^\top\|_F \leq \cO(1)$ and $\|\ntk_{\Bar \theta}(X,X')\|_{op} \leq \|\ntk_{\Bar \theta}(X,X')\|_F=\|\frac{1}{l} \nabla f_{\Bar \theta}(X)  \nabla f_{\Bar \theta}(X')^\top\|_F \leq \cO(1)$. Also, we have
\begin{align}
    \|\ntk_\theta(X,X') - \ntk_{\Bar \theta}(X,X')\|_{op} &\leq \|\ntk_\theta(X,X') - \ntk_{\Bar \theta}(X,X')\|_F \nonumber\\
    &= \| \frac{1}{ l} \nabla f_\theta(X) \nabla f_\theta(X')^\top - \frac{1}{ l} \nabla f_{\Bar \theta}(X) \nabla f_{\Bar \theta}(X')^\top\|_F\nonumber\\
    &= \frac{1}{l} \|\frac 1 2( \nabla_\theta f_\theta(X) + \nabla_{\Bar \theta} f_{\Bar \theta}(X))(\nabla_\theta f_\theta(X')^\top -\nabla_{\Bar \theta} f_{\Bar \theta}(X')^\top) \nonumber\\
    &\quad + \frac {1} {l}( \nabla_\theta f_\theta(X) -  \nabla_{\Bar \theta} f_{\Bar \theta}(X))(\nabla_\theta f_\theta(X')^\top+\nabla_{\Bar \theta} f_{\Bar \theta}(X')^\top)\|_F\nonumber\\
    &\leq \frac{1}{2l} \|\nabla_\theta f_\theta(X) + \nabla_{\Bar \theta} f_{\Bar \theta}(X)\|_F\|\nabla_\theta f_\theta(X') -  \nabla_{\Bar \theta} f_{\Bar \theta}(X')\|_F\nonumber\\
    &\quad +\frac{1}{2l}  \|\nabla_\theta f_\theta(X) - \nabla_{\Bar \theta} f_{\Bar \theta}(X)\|_F\|\nabla_\theta f_\theta(X') +  \nabla_{\Bar \theta} f_{\Bar \theta}(X')\|_F\nonumber\\
    &\leq  \frac 1 2\left(\|\frac{1}{\sqrt l}\nabla_\theta f_\theta(X')\|_F +\|\frac{1}{\sqrt l} \nabla_{\Bar \theta} f_{\Bar \theta}(X')_F\|\right) \|\frac{1}{\sqrt l}\nabla_\theta f_\theta(X') - \frac{1}{\sqrt l} \nabla_{\Bar \theta} f_{\Bar \theta}(X')\|_F \nonumber\\
    &\quad + \frac 1 2\left(\|\frac{1}{\sqrt l}\nabla_\theta f_\theta(X')\|_F -\|\frac{1}{\sqrt l} \nabla_{\Bar \theta} f_{\Bar \theta}(X')_F\|\right) \|\frac{1}{\sqrt l}\nabla_\theta f_\theta(X')+ \frac{1}{\sqrt l} \nabla_{\Bar \theta} f_{\Bar \theta}(X')\|_F \nonumber\\
    &\leq 2 \cO(1)\cdot \cO(\|\theta-\Bar \theta\|_2)\nonumber\\
    &= \cO(\|\theta - \Bar \theta\|_2) \label{eq:lemma:F-linear-in-weight-near-init:2}
\end{align}
Now, let us move on to $\ntk_\theta^{-1}$ and $\ntk_{\Bar \theta}^{-1}$. Recall that \cite{ntk} proves for large enough width, $\ntk_\theta$ is positive definite, implying $\lev(\ntk_\theta)>0$. Similarly, $\lev(\ntk_{\Bar \theta}) > 0$. These indicate $\|\ntk_\theta^{-1}\|_{op},\|\ntk_{\Bar \theta}^{-1}\|_{op} \leq \cO(1)$.
Then, by an additive perturbation bound of the matrix inverse under Frobenius norm \cite{inverse-perturbation}, we have
\begin{align}
    \|\ntk_{ \theta}^{-1}-\ntk_{\Bar \theta}^{-1}\|_{op} \leq \frac{1}{\min\{\lev(\theta),\lev(\Bar \theta) \}} \|\ntk_\theta - \ntk_{\Bar \theta}\|_{op} \leq \cO(\|\theta - \Bar \theta\|_2).
\end{align}
Then, consider $(I-e^{-\eta \ntk_{\theta} \tau})$ and $(I-e^{-\eta \ntk_{\Bar \theta} \tau})$. Clearly, $\|(I-e^{-\eta \ntk_{\theta} \tau})\|_{op},\|(I-e^{-\eta \ntk_{\Bar \theta} \tau})\|_{op} \leq 1$. Also, by \cite{1977bounds}, we have
\begin{align}
\|(I-e^{-\eta \ntk_{\theta} \tau})-(I-e^{-\eta \ntk_{\Bar \theta} \tau})\|_{op}&=\|e^{-\eta \ntk_{\theta} \tau}-e^{-\eta \ntk_{\Bar \theta} \tau}\|_{op} \leq \eta \tau\|\ntk_\theta - \ntk_{\Bar\theta}\|_{op} \leq  \cO(\|\theta-\Bar \theta\|_2).
\end{align}

So far, we already proved $\|\ntk_\theta(X,X')\|_{op}, \|\ntk_{ \theta}^{-1}\|_{op}, \|I-e^{-\eta \ntk_{\theta} \tau}\|_{op} \leq \cO(1)$, with the same bounds apply to the case of $\Bar \theta$. As a result, we have $\|\ntk_{\theta}(X,X')\ntk_{\theta}^{-1} (I - e^{-\lambda \ntk_{\theta} \tau})\|_{op} \leq \cO(1)$. Besides, we also proved
$\|\ntk_\theta(X,X') - \ntk_{\Bar \theta}(X,X')\|_{op}, \|\ntk_{ \theta}^{-1}-\ntk_{\Bar \theta}^{-1}\|_{op}, \|(I-e^{-\eta \ntk_{\theta} \tau})-(I-e^{-\eta \ntk_{\Bar \theta} \tau})\|_{op} \leq \cO(\|\theta - \Bar \theta\|_2)$. With these results, applying the proof technique used in (\ref{eq:lemma:F-linear-in-weight-near-init:2}), we can easily obtain the following bound
\begin{align}
    \|\ntk_\theta(X,X')\ntk_{ \theta}^{-1} (I-e^{\eta \ntk_{\theta} t})- \ntk_{\Bar \theta}(X,X')\ntk_{\Bar \theta}^{-1} (I-e^{\eta \ntk_{\Bar \theta} t})\|_{op} \leq \cO(\|\theta -\Bar \theta\|_2)
\end{align}
Finally, combining results above, we obtain our desired bound
\begin{align}
    &\|F_\theta(\wt X)-F_{\Bar \theta}(\wt X) - \nabla_{\theta} F_\theta(\wt X)\cdot (\Bar \theta -\theta)\|_2  \nonumber\\
    &=~~\biggl\|f_\theta(X) + \ntk_{\theta}(X,X')\ntk_{\theta}^{-1} (I - e^{-\lambda \ntk_{\theta} \tau})(Y'-f_\theta(X'))\nonumber\\
    &~~- \left[f_{\Bar \theta}(X) + \ntk_{\Bar \theta}(X,X')\ntk_{\Bar \theta}^{-1} (I - e^{-\lambda \ntk_{\Bar \theta} \tau})(Y'-f_{\Bar \theta}(X'))\right]\nonumber\\
    &~~- \left[ \nabla_{\theta}f_\theta(\sX) - \ntk_{\theta}(X,X')\ntk_{\theta}^{-1} (I - e^{-\lambda \ntk_{\theta} \tau}) \nabla_{\theta}f_\theta(\sX')\right]\cdot (\Bar \theta -\theta) \biggr\|_2 \nonumber\\
    &=\biggl\|f_\theta(X) - f_{\Bar \theta}(X) - \nabla f_\theta(X) \cdot (\Bar \theta - \theta ) \nonumber\\
    &~~+\left(\ntk_\theta(X,X')\ntk_{ \theta}^{-1} (I-e^{\eta \ntk_{\theta} \tau})- \ntk_{\Bar \theta}(X,X')\ntk_{\Bar \theta}^{-1} (I-e^{\eta \ntk_{\Bar \theta} \tau})\right)\left(Y'-f_{\Bar \theta}(X')\right) \nonumber\\
    &~~- \ntk_{\theta}(X,X')\ntk_{\theta}^{-1} (I - e^{-\lambda \ntk_{\theta} \tau}) \left(f_\theta(X) - f_{\Bar \theta}(X) - \nabla f_\theta(X) \cdot (\Bar \theta - \theta )\right) \biggr\|_2\nonumber\\
    &\leq   \|f_\theta(X) - f_{\Bar \theta}(X) - \nabla f_\theta(X) \cdot (\Bar \theta - \theta )\|_2 \nonumber\\
    &~~+ \|\ntk_\theta(X,X')\ntk_{ \theta}^{-1} (I-e^{\eta \ntk_{\theta} \tau})- \ntk_{\Bar \theta}(X,X')\ntk_{\Bar \theta}^{-1} (I-e^{\eta \ntk_{\Bar \theta} \tau})\|_{op}\|Y'-f_{\Bar \theta}(X')\|_2\nonumber\\
    &~~+ \| \ntk_{\theta}(X,X')\ntk_{\theta}^{-1} (I - e^{-\lambda \ntk_{\theta} \tau})\|_{op}\|f_\theta(X) - f_{\Bar \theta}(X) - \nabla f_\theta(X) \cdot (\Bar \theta - \theta )\|_2\nonumber\\
    &\leq \cO(\omega^{1/3} (L+1)^2 \sqrt{l \log (l)}) \|\Bar \theta - \theta\|_2 + \cO(\|\theta - \Bar \theta\|_2) + \cO(1)\cdot \cO(\omega^{1/3} (L+1)^2 \sqrt{l \log (l)}) \|\Bar \theta - \theta\|_2\nonumber\\
    &=\cO(\omega^{1/3} (L+1)^2 \sqrt{l \log (l)}) \|\Bar \theta - \theta\|_2  ,\label{eq:lemma:F-linear-in-weight-near-init:-1} 
\end{align}
Note that in the last step of (\ref{eq:lemma:F-linear-in-weight-near-init:-1}), we used the result $\|f_\theta(X) - f_{\Bar \theta}(X) - \nabla f_\theta(X) \cdot (\Bar \theta - \theta )\|_2 \leq \cO(\omega^{1/3} (L+1)^2 \sqrt{l \log (l)}) \|\Bar \theta - \theta\|_2$ from \cite{cao2019generalization}, and the fact that $\|Y'-f_{\Bar \theta}(X')\|_2$ and $\| \ntk_{\theta}(X,X')\ntk_{\theta}^{-1} (I - e^{-\lambda \ntk_{\theta} \tau})\|_{op}$ are upper bounded by constants, which can be easily derived from Lemma \ref{lemma:bounded_init_loss}.
\end{proof}

\subsubsection{Proof of Remaining Lemmas}
\begin{proof}[Proof of Lemma \ref{lemma:loss-convex-near-init}]
This proof can be easily obtained by combining Lemma \ref{lemma:F-linear-in-weight-near-init} the proof of Lemma 4.2 of \cite{cao2019generalization}.
\end{proof}

\begin{proof}[Proof of Lemma \ref{lemma:grad-norm-bounded-near-init}]
This proof can be easily obtained by combining (\ref{eq:lemma:jacobian:F-grad:norm:0}) and Lemma B.3 of \cite{cao2019generalization}.
\end{proof}

\begin{proof}[Proof of Lemma \ref{lemma:bound-on-cum-loss}]
This proof can be easily obtained by combining Lemma \ref{lemma:loss-convex-near-init}, Lemma \ref{lemma:grad-norm-bounded-near-init}, and Lemma 4.3 of \cite{cao2019generalization}.
\end{proof}
\newpage

\section{Details of Experiments}
\label{supp:exp:gen-bound}
\subsection{Experiments on the Synthetic Dataset}
\textbf{Dataset Generation.} As described in Sec. \ref{sec:exp:gen-bound:synthetic}, we consider the problem setting of \textit{1-d few-shot regression with quadratic objective functions} described in Sec. \ref{sec:quadratic-example}, and generate $N=40$ training tasks along with $40$ test tasks following that problem setting of Sec. \ref{sec:quadratic-example}. Each task has $2$ support samples and $8$ query samples.

\textbf{Implementation of MAML.}
As for the implementation of MAML, we use the code from \cite{higher}, which obtain similar or better performance than the original MAML paper \cite{maml}. The details of neural network construction and learning rate can be found in \cite{higher} or its public codebase at \url{https://github.com/facebookresearch/higher/blob/master/examples/maml-omniglot.py}.

\textbf{Implementation of Meta Neural Kernels.} As discussed in Sec. \ref{sec:MetaNTK}, the Meta Neural Kernel $\metaNTK$ dervied in Theorem \ref{thm:MNK} is a composite kernel built upon a base kernel function, $\NTK$, which is precisely the Neural Tangent Kernel (NTK) derived in the setting of supervised learning \cite{ntk,lee2019wide,CNTK}. We adopt a Python implementation of NTK from \url{https://github.com/LeoYu/neural-tangent-kernel-UCI}, and then we implement the Meta Neural Kernel of Theorem \ref{thm:MNK} in Python.

\textbf{Computation of Generalization Bounds.} We compute the generalization bound by following \eqref{eq:MetaNTK:gen-bound} of Theorem \ref{thm:gen-bound}. Specifically, we compute the term $(L+1) \sqrt{\frac{\widetilde{\Y}_{\scriptscriptstyle{G}}^\top \metaNTK^{-1} \widetilde{\Y}_{\scriptscriptstyle G}}{Nn} }$, where $\metaNTK$ and $\wt \Y_{G}$ are computed on the training data (i.e., $\{\task_i\}_{i=1}^N$). 

\textbf{Neural Net Structures and Hyper-parameters}
For MAML, we train a 2-layer fully-connected neural net with ReLU activation by minimizing the MAML objective \eqref{eq:MAML-obj} on the training tasks with the Adam optimizer \cite{adam} as the outer-loop optimizer and SGD as the inner-loop optimizer. We take both inner- and outer-loop learning rates as 0.001, and the number of meta-adaptation steps is 1.

\subsection{Experiments on the Omniglot dataset}

\textbf{Details of Dataset and Few-Shot Classification Setup.} The Omniglot dataset contains 1623 handwritten characters from 50 alphabets. For each character, the dataset provides 20 image instances, in which each instance is handwritten by a unique person. A normal protocol of few-shot learning is to select 1200 characters for training and the remaining 423 characters for test \cite{maml}, then to perform $k$-way $n$-shot classification. The setup of $k$-way $n$-shot classification is: randomly take $k$ classes (i.e., characters in Omniglot), then provide $n$ different samples with labels of each class to the few-shot learning model (e.g., GBML model), and finally evaluate the model's performance on classifying the rest samples (i.e., $(20-n)$ samples in the case of Omniglot) in these $k$ classes. Note the $n$ and $k$ used here is consistent with our definition of them in Sec \ref{sec:few-shot-learning}.

\textbf{Implementation of MAML.}
As for the implementation of MAML, we use the code from \cite{higher}, which obtain similar or better performance than the original MAML paper \cite{maml}. The details of neural network construction and learning rate can be found in \cite{higher} or its public codebase at \url{https://github.com/facebookresearch/higher/blob/master/examples/maml-omniglot.py}. We also use the data loader for the Omniglot dataset in \cite{higher}. 

\textbf{Implementation of Meta Neural Kernels.} As discussed in Sec. \ref{sec:MetaNTK}, the Meta Neural Kernel $\metaNTK$ is a composite kernel built upon a base kernel function, $\NTK$, which is precisely the Neural Tangent Kernel (NTK) derived in the setting of supervised learning \cite{ntk,lee2019wide,CNTK}. Since we consider the few-shot image classification problem, we need a base kernel function that suits the image domain. Hence, we adopt the Convolutional Neural Tangent Kernel (CNTK) \cite{CNTK}, the NTK derived from over-parameterized Convolution Neural Networks (CNNs), as the base kernel function. Since the official CNTK code\footnote{\url{https://github.com/ruosongwang/CNTK}} is written in CUDA with a Python interface, it needs NVIDIA GPU for computation. Therefore, we use a workstation with 4 GPUs of RTX 2080 ti. Besides, we implement MNK in Python by following the formula in Theorem \ref{thm:MNK}.

\textbf{Computation of Generalization Bounds.} We compute the generalization bound by following \eqref{eq:MetaNTK:gen-bound} of Theorem \ref{thm:gen-bound}. Specifically, we compute the term $(L+1) \sqrt{\frac{\widetilde{\Y}_{\scriptscriptstyle{G}}^\top \metaNTK^{-1} \widetilde{\Y}_{\scriptscriptstyle G}}{Nn} }$, where $\metaNTK$ and $\wt \Y_{G}$ are computed on the training data (i.e., $\{\task_i\}_{i=1}^N$).

\textbf{Data Preprocessing.} Since we derive the MNK-based kernel method in the regression setting with $\ell_2$ loss, we have to perform label preprocessing in order to apply this kernel method to few-shot multi-class classification. The reason for that is demonstrated below.
The application of kernel regression on multi-class classification \cite{sklearn} usually uses one-hot encoding, a one-to-one mapping on digital labels. However, it fails in the case of kernel regression on \textit{few-shot} multi-class classification, since each classification task (e.g., training or test task) has its own classes of labels. For instance, in 5-way $n$-shot classification, a task assigns digital labels $\{1,2,3,4,5\}$ to its samples, but another task also has 5 classes of samples, so it assigns the same digital labels, $\{1,2,3,4,5\}$, to its samples. Then, different classes from multiple tasks share the same digital labels, which is ill-defined in the setting of kernel regression\footnote{See the next paragraph that explains why one-hot or digital labels are ill-defined for kernel regression on few-shot classification.}. To resolve this issue,
we design a label preprocessing technique that projects digital labels from different tasks into a single vector space. Specifically, we first choose a fixed feature extractor $\psi$ such that it can transform any sample $x$ into a feature vector, $\psi(x) \in \mathbb{R}^h$, in a $h$-d Euclidean space. Then, we
use this feature extractor to convert all samples in each training task into feature vectors. For test tasks, we do this for support samples only. After that, in each task, we compute the centroids of feature vectors corresponding to samples in each class  (i.e., obtain five centroids for the five classes in each task), and use the centroid (i.e., a $h$-d vector) of each class as its new label. In this way, classes from various tasks are marked by different vector labels, which are well-defined for kernel regression. 
For convenience, we fit a PCA on all training samples and use it as the feature extractor.


\textbf{Why Digital or One-Hot Labels cannot be Directly Used.} In the paragraph above, we briefly discussed why digital or one-hot labels could not be directly used for kernel regression on few-shot classification. Here, we explain the reasons more detailedly. In general, digital labels are not suitable for regression methods, since regression methods are usually based on $\ell_2$ loss that is not designed for categorical labels. Hence, in the application of kernel regression on multi-class classification, the one-hot encoding of digital labels is the most used label preprocessing technique\footnote{For instance, this is what scikit-learn \cite{sklearn}, one of the most popular code package for machine learning, uses for kernel methods: \url{https://scikit-learn.org/stable/modules/preprocessing.html\#preprocessing-categorical-features}.}. Even though one-hot label encoding works for kernel regression on multi-class classification, a problem of supervised learning, it does not fit the few-shot multi-class classification, which contains multiple supervised learning tasks. We can see the issue easily by a thought experiment: interchangeable labels. In the case of standard kernel regression for multi-class classification, a supervised learning task, the labels for different classes can be interchanged without causing any influence on the final prediction of the model, which can be well explained by the example of (\ref{eq:supp:one-hot}). In that example, there are five classes, and the digital labels for training samples are $[3,2,5,3,\dots]$, then the corresponding one-hot encoded labels, $Y'$, is expressed as
\begin{align}\label{eq:supp:one-hot:exp}
Y'=
    \begin{bmatrix}
0 & 0 & 1 & 0 & 0\\
0 & 1 & 0 & 0 & 0\\
0 & 0 & 0 & 0 & 1\\
0 & 0 & 1 & 0 & 0\\
 \vdots&  \vdots& \vdots & \vdots & \vdots\\
\end{bmatrix}
\end{align}
If we interchange two of these digital labels, e.g., $3 \leftrightarrow 5$, then the third and fifth columns of $Y'$ are interchanged. However, this operation has no impact on the kernel regression model, (\ref{eq:supp:standard-kernel-regression}), since the prediction of the model, $\hat Y$, also interchanges its third and fifth columns correspondingly. Finally, the prediction of the class of each sample remains the same.

However, in the setting of few-shot multi-class classification, this label interchangeability does not hold. For simplicity, consider a few-shot multi-class classification problem with only two training tasks, $\task_1 = (X_1,Y_1,X_1',Y_1')$ and $\task_2 = (X_2,Y_2,X_2',Y_2')$. Suppose each task contains five unique classes of samples. If we continue using one-hot labels, the classes for samples in $\task_1$ and $\task_2$ are labelled by the one-hot encoding of $\{1,2,3,4,5\}$. Assume $Y_1$ is the same as the $Y'$ in (\ref{eq:supp:one-hot:exp}), and $Y_2$ is the one-hot version of digital labels $\{1,5,2,3,\dots \}$. Then, the labels for query samples in training tasks, i.e., $\Y$ in (\ref{eq:F_t-MetaNTK}), can be seen as a concatenation of $Y_1$ and $Y_2$:
\begin{align}\label{eq:supp:one-hot:few-shot}
    \Y=
    \begin{bmatrix}
0 & 0 & \mathbf{1} & 0 & \mathbf 0\\
0 & 1 & \mathbf 0 & 0 & \mathbf 0\\
0 & 0 & \mathbf 0 & 0 & \mathbf 1\\
0 & 0 & \mathbf 1 & 0 & \mathbf 0\\
 \vdots&  \vdots& \vdots & \vdots & \vdots\\
 \it{1} & \it{0} & \it{0} & \it0 & \it0\\
\it 0 & \it 0 &\it{0} &\it 0 &\it 1\\
\it 0 &\it 1 &\it 0 &\it 0 &\it 0\\
\it 0 &\it 0 &\it 1 &\it 0 &\it 0\\
 \vdots&  \vdots& \vdots & \vdots & \vdots
\end{bmatrix}
\end{align}
where the upper rows represent $Y_1$ and the lower rows with the italian font are $Y_2$.

Since kernel regression treats each column as an individual dimension, different columns do not affect each other. However, elements in the same column have a correlation with each other, since they are in the same dimension. Therefore, labels for different classes cannot be interchanged in the example of (\ref{eq:supp:one-hot:few-shot}). For instance, if we interchange the two digital labels, $3 \leftrightarrow 5$, for $\task_1$, then the bold elements in $Y_1$ are interchanged between the third and fifth columns. However, the elements in the third column of $Y_1$ have a correlation with the elements of $Y_2$ in the third column. Thus the label change affects the prediction of the samples corresponding to the third column. Similarly, the fifth column is also impacted. Hence, interchanging digital labels of classes in a single task has an effect on the prediction of the kernel regression----the prediction of the kernel regression does not remain invariant w.r.t. interchanged labels. \textit{That is why the label inter-changeability is broken for few-shot multi-class classification, which should not happen, since the labels for classes in each task are assigned in an artificial and arbitrary order. As a result, the one-hot encoding of labels is ill-defined for kernel regression on few-shot multi-class classification problems. Therefore, we need to find a new label encoding method without the problem of broken label interchangeability.}

\textbf{Neural Net Structures and Hyper-parameters}
For MAML, we use the default neural net architecture in our adopted implementation\footnote{\url{https://github.com/facebookresearch/higher/blob/master/examples/maml-omniglot.py}}. Specifically, we build a Convolutional Neural Net (CNN) with 3 convolutional layers followed by 1 fully connected layer. The CNN uses ReLU activation, BatchNorm \cite{batchnorm}, and Max Pooling. We use the Adam optimizer \cite{adam} as the outer-loop optimizer with learning rate 0.01, and SGD as the inner-loop optimizer with learning rate 0.1. The number of meta-adaptation steps is 5. For the implementation of the MNK method shown in Theorem \ref{thm:MNK}, we adopt $\tau = \infty$ and $t = \infty$, which simply leads to vanishing exponential terms\footnote{This choice is also used in \cite{CNTK,harnessNTK,ECNTK}} in (\ref{eq:F_t-MetaNTK}).
Also, we adopt the commonly used ridge regularization for kernel regression\footnote{Same as the ridge regularization for kernel regression used in scikit-learn \cite{sklearn}:\\
\url{https://scikit-learn.org/stable/modules/kernel\_ridge.html}}, with coefficient as $10^{-5}$, to stabilize the kernel regression computation.


\end{document}